%% file: main.tex
\documentclass[letterpaper]{article} 
\usepackage[preprint]{aaai2027} 
\usepackage[hyphens]{url} 
\usepackage{graphicx} 
\usepackage{natbib} 
\usepackage{caption} 
\usepackage{booktabs}
\usepackage{multirow}
\usepackage{array}
\usepackage{amsmath,amssymb,amsthm,mathtools}
\usepackage{placeins}
\usepackage[table]{xcolor}
\newcommand{\RMS}{\operatorname{RMS}}
\newcommand{\sg}{\operatorname{sg}}
\newcommand{\ind}{\mathbb{I}}

\newtheorem{proposition}{Proposition}
\newtheorem{lemma}{Lemma}
\newtheorem{corollary}{Corollary}

\newcolumntype{L}[1]{>{\raggedright\arraybackslash}p{#1}}

\title{EMAN: Optimization-Driven Capacity Growth through Path Emergence in Multi-Task Learning}
\author{
Chenlei Fang\textsuperscript{1,2,3,4,5},
Jingchen Li\textsuperscript{1},
Hongzong LI\textsuperscript{1},
Qingyao Li\textsuperscript{2,3,4,5},\\
Yixuan Zhang\textsuperscript{2,3,4,5},
Huarui Wu\textsuperscript{2,3,4,5},
Haobin Shi\textsuperscript{1},
Chunjiang Zhao\textsuperscript{2,3,4,5}\corresponding
}
\affiliations{
\textsuperscript{\rm 1}School of Computer Science, Northwestern Polytechnical University, Xi'an, China\\
\textsuperscript{\rm 2}National Engineering Research Center for Information Technology in Agriculture, Beijing 100097, China\\
\textsuperscript{\rm 3}Information Technology Research Center, Beijing Academy of Agriculture and Forestry Sciences, Beijing 100097, China\\
\textsuperscript{\rm 4}Key Laboratory of Digital Rural Technology, Ministry of Agriculture and Rural Affairs, Beijing 100097, China\\
\textsuperscript{\rm 5}Beijing Research Institute of Agricultural Artificial Intelligence and Robotics, Beijing 100097, China\\

 \small { \texttt{chenleifang@mail.nwpu.edu.cn}},
 \small { \texttt{zhaocj@nercita.org.cn}} \\
}

\begin{document}
\maketitle

\begin{abstract}

Existing multi-task learning methods rely on hard sharing, multiple paths or experts, adaptive sharing, and dynamic expansion. However, their capacity changes are usually constrained by predefined structures or triggered by task boundaries and conflict signals. This raises a fundamental question: can a network start from exact single-path computation and grow a new independent path only when persistent optimization evidence appears? We propose the \textbf{\underline{E}}mergent
\textbf{\underline{M}}odular \textbf{\underline{A}}tomic \textbf{\underline{N}}etwork
(EMAN), an optimization-driven framework for exposing an antisymmetric growth direction through latent relative phases without instantiating a second path, and for monitoring multiple decision signals during the training process through a set of optimization evidence to transform local optimization evidence into a structural decision. EMAN materializes two equal-capacity independent paths only after certification. EMAN adaptively allocates shared and task-specific representation capacity to accommodate varying task requirements. Extensive experiments on controlled rank settings, PASCAL-Context, and NYUv2 validate its effectiveness, achieving improved performance at a competitive computational cost.

\end{abstract}

\section{Introduction}

Hard parameter sharing remains one of the simplest and most efficient designs
for multi-task learning (MTL) \citep{caruana1997multitask}. A dense shared
backbone is fully differentiable and inexpensive to optimize. Its fixed
representation capacity can also create negative transfer or a hard
bottleneck under heterogeneous task demands. Soft-sharing architectures,
task-interaction modules, and adaptive sharing policies provide greater
flexibility
\citep{misra2016crossstitch,ruder2019sluice,liu2019mtan,
vandenhende2020mtinet,xu2018padnet,sun2020adashare}.
Mixture-of-Experts (MoE) models take another route. They direct examples or
tasks through specialized experts
\citep{shazeer2017outrageously,ma2018mmoe,lepikhin2021gshard,
fedus2022switch}. Both families improve flexibility. Both also commit to
multiple modules, experts, routing rules, or sharing policies before
optimization reveals the need for extra capacity.

Recent MTL research approaches this tension from several directions. Current
studies examine spurious inter-task correlations, shared and task-specific
optimization, and task-adaptive parameter-efficient sharing
\citep{chai2025spurious,qin2025consistent,baek2025tadformer,
mantri2025ditask}. Dense-scene MTL also revisits cross-task structure on
NYUv2 and PASCAL-Context \citep{wang2026mtl3d}. Recent expert systems adapt
expert allocation or expand expert pools during training
\citep{park2026mass,zhang2026growondemand}. These advances improve
interaction, allocation, and expansion. The number of independent modules
still depends on fixed task-specific structures, explicit routing, or
declared task boundaries. Two paths can add capacity. The open question is
when their additional cost becomes justified.

\begin{figure}[t]
    \centering
    \includegraphics[width=\linewidth]{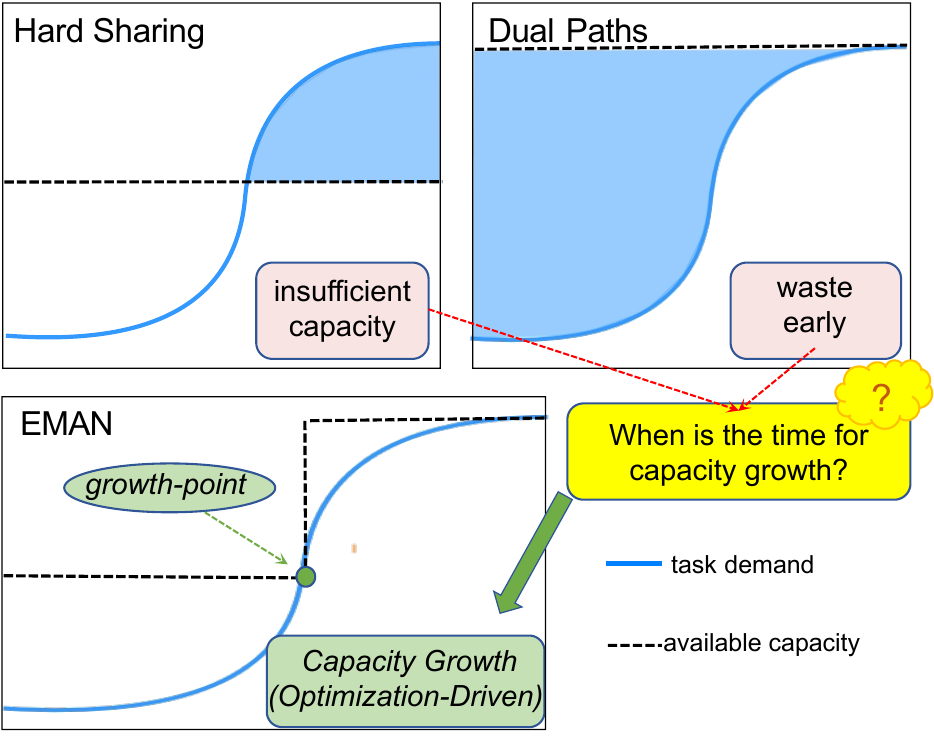}
    \caption{Capacity mismatch across training.
    Hard sharing is inexpensive early but can become insufficient later.
    Always-dual paths avoid the bottleneck but maintain excess capacity
    throughout training. EMAN starts from one exact shared path and
    materializes a second path only after persistent evidence.}
    \label{fig:capacity_mismatch}
\end{figure}

Figure~\ref{fig:capacity_mismatch} illustrates this mismatch. Hard sharing
fixes one-path capacity, so later demand can exceed the available
representation capacity. Always-dual paths avoid this bottleneck, but they
maintain two-path computation from the start. A more selective design should
retain exact sharing while one path is sufficient and add independent
capacity after persistent demand appears. The growth decision should emerge
from optimization itself.

We introduce the Emergent Modular Atomic Network (EMAN), an
optimization-driven dense-to-modular framework. EMAN starts with one physical
forward path and exact shared computation. A latent relative phase then probes
the local optimization geometry without changing physical capacity. The
fixed odd phase interaction exposes a candidate antisymmetric growth
direction. Optimization Evidence Certification (OEC) evaluates its strength,
cross-probe consistency, finite-step descent, dynamic safety, and temporal
persistence. Insufficient evidence retains the shared path. A successful
certificate materializes two isomorphic and independently trainable paths
through opposite perturbations. Training data drive the entire transition,
and the certification rule remains fixed throughout training.

The local analysis describes this transition through the antisymmetric path
coordinate $d$ and the relative phase $\phi$. Near the shared solution, a
phase-symmetric branch can lose transverse stability before the path
coordinates separate. After phase release, the cross term $\phi c^\top d$
induces the antisymmetric force $c\phi$. The second physical path is still
absent, yet the local geometry already exposes a candidate descent direction.
Rank-controlled experiments connect this mechanism to explicit capacity
demand. One physical path has rank-three capacity, and the joint task
structure requires ranks of three, four, and six. EMAN remains shared at
joint rank three. It materializes a second path under the rank-six bottleneck
and recovers the missing latent factors and predictive performance. The
rank-four condition characterizes OEC near the single-path capacity limit.
The released paths also span complementary latent subspaces.

Natural-image experiments examine the timing of this structural decision. In
the static PASCAL-Context setting, all four tasks appear from the beginning,
and EMAN releases early. Little opportunity remains for delayed-computation
savings. Under task arrival, EMAN preserves one-path sharing during the
initial two-task stage and releases after the saliency and surface-normal
tasks enter training. Across three seeds, it reaches final performance
comparable to the always-dual baseline, shortens the active dual-path stage to
about half, and uses about $75\%$ of its path-training compute. Earlier demand
moves the same frozen decision from epoch 21 to epoch 11. Gradual demand
provides a further test under progressive capacity requirements. On NYUv2,
EMAN achieves the best mean semantic-segmentation and depth results and the
highest aggregate score. Recon retains the strongest surface-normal result
and joint loss. PASCAL-Context shows less uniform task-level gains. The
overall pattern is consistent: the timing and strength of optimization demand
shape the value of path emergence.

Our contributions are summarized as follows:
\begin{enumerate}
    \item \textbf{Optimization-driven capacity growth from exact sharing.}
    We propose EMAN, which starts with one physical path and treats
    independent capacity as a structural decision during optimization. The
    model retains exact sharing while one path remains sufficient.

    \item \textbf{Phase-first evidence-gated path emergence.}
    EMAN exposes an antisymmetric growth direction through latent relative
    phase before adding an independent parameter block. A fixed evidence rule
    evaluates the direction and triggers a zero-sum materialization into two
    independently trainable paths.

    \item \textbf{Controlled and real-data validation.}
    Rank-controlled experiments establish capacity preservation, conditional
    rank recovery, and complementary released representations. NYUv2 and
    PASCAL-Context further demonstrate adaptive release, competitive
    performance, and reduced path-training computation.
\end{enumerate}

\section{Related Work}

\paragraph{Sharing in multi-task learning.}
Hard parameter sharing provides an efficient foundation for multi-task
learning, but one representation can entangle incompatible objectives
\citep{caruana1997multitask}. Cross-stitch networks, sluice networks,
attention modules, and multi-scale interaction mechanisms learn softer
sharing patterns
\citep{misra2016crossstitch,ruder2019sluice,liu2019mtan,
vandenhende2020mtinet,xu2018padnet}. AdaShare learns task-dependent sharing
policies \citep{sun2020adashare}. Gradient-based methods address loss
imbalance and task conflict
\citep{chen2018gradnorm,yu2020pcgrad,liu2021cagrad,navon2022nash,
javaloy2022rotograd}, and Recon converts highly conflicting shared layers
into task-specific ones \citep{shi2023recon}. Recent work further explores
the balance between shared and task-specific representations. ConsMTL
optimizes both parameter groups \citep{qin2025consistent}. TADFormer uses
input-conditioned task adaptation, and DiTASK applies parameter-efficient
task-specific transformations
\citep{baek2025tadformer,mantri2025ditask}. Other studies identify spurious
inter-task correlations or model geometric cross-task relations
\citep{chai2025spurious,wang2026mtl3d}. This line of work makes sharing more
flexible, yet the number of independent paths remains fixed before training.

\paragraph{Experts, modularity, and expansion.}
Sparse MoE and multi-gate expert models use routers to select specialized
computation
\citep{shazeer2017outrageously,ma2018mmoe,lepikhin2021gshard,
fedus2022switch}. Modularity creates room for specialization, but data
structure and optimization also shape the resulting functions
\citep{csordas2021modular,mittal2022modular,jarvis2023specialization,
zhang2023emergent,schug2024discovering}. Progressive networks, dynamically
expandable networks, and PathNet add or select capacity across tasks
\citep{rusu2016progressive,yoon2018den,fernando2017pathnet}. Net2Net
preserves the learned function during architectural growth
\citep{chen2016net2net}. More recent systems make expert allocation adaptive:
MASS detects semantic drift and expands a routed expert pool
\citep{park2026mass}, and GoD-MoE estimates expert demand and adds LoRA
experts during continual instruction tuning
\citep{zhang2026growondemand}. These methods begin with an explicit expert
structure or a declared task sequence. EMAN begins with dense single-path
computation. Phase-structured evidence determines whether a second
independent path should emerge. Capacity growth thus becomes an endogenous
event in optimization.

\paragraph{Symmetry and local curvature.}
Exchange-symmetric components admit symmetric and antisymmetric coordinates,
a standard basis for symmetry-breaking analysis
\citep{golubitsky1988singularities,kuznetsov2004elements}.
Hessian-vector products provide local curvature information without
constructing the full Hessian \citep{pearlmutter1994fast}. We draw on these
ideas near the exactly shared state. In EMAN, this information characterizes
the transverse geometry of the shared branch. Relative phase exposes a
coupled antisymmetric direction, and OEC evaluates its significance during
training. The analysis connects local geometry to the decision to create
independent capacity.

\section{EMAN: Optimization-Driven Path Emergence}
\label{sec:method}

\subsection{Overview}
\label{sec:method_overview}

EMAN treats multi-task capacity growth as a structural decision during
training. Training begins with exact single-path computation. A
relative-phase probe then exposes the local optimization geometry and a
candidate antisymmetric growth direction. Optimization Evidence
Certification (OEC) evaluates direction strength, cross-probe consistency,
finite-step descent, dynamic safety, and temporal persistence. Persistent
evidence triggers path materialization; otherwise, the model retains exact
sharing.

The complete procedure contains four stages:
\begin{center}
\small
exact sharing $\rightarrow$ phase probing $\rightarrow$ OEC\\[1.5mm]
$\rightarrow$ conditional path materialization
\end{center}

Figure~\ref{fig:eman_local_mechanism} focuses on the local transition. The
initial state contains one physical path, zero relative phase, and zero
antisymmetric displacement. Phase probing reveals a candidate direction
without adding physical capacity. OEC then examines the full evidence set.
A successful certificate produces a zero-sum split, followed by independent
optimization of the two materialized paths.

\begin{figure}[t]
    \centering
    \includegraphics[width=\linewidth]{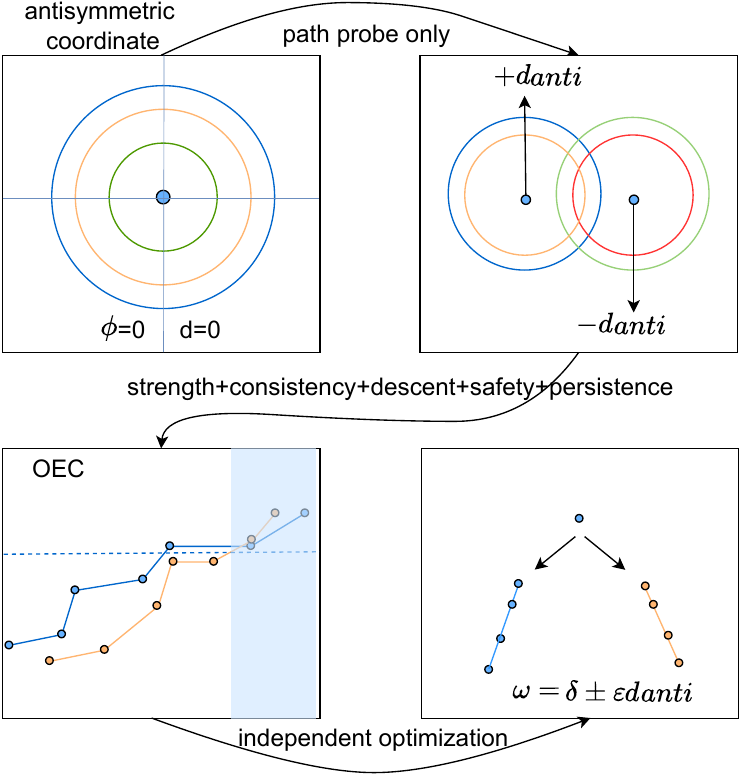}
    \caption{Phase-first certification and path materialization.
    EMAN starts from one tied physical path. Phase probing exposes a
    candidate antisymmetric direction without creating a new path. OEC
    evaluates strength, consistency, descent, safety, and persistence.
    Persistent evidence triggers a zero-sum split, and the released paths
    then optimize independently.}
    \label{fig:eman_local_mechanism}
\end{figure}

Figure~\ref{fig:method} gives the architecture-level view. The shared encoder
first performs exact single-path computation. The lower central block
contains phase-coupled probing and OEC. A successful certificate materializes
two equal-capacity paths through opposite perturbations. The upper central
block then performs two-path computation, feature fusion, and task-specific
prediction. The right block shows the phase-coupled pair-interaction matrix:
off-diagonal entries encode pair interactions, and diagonal entries encode
self terms.

\begin{figure*}[t]
    \centering
    \includegraphics[width=\textwidth]{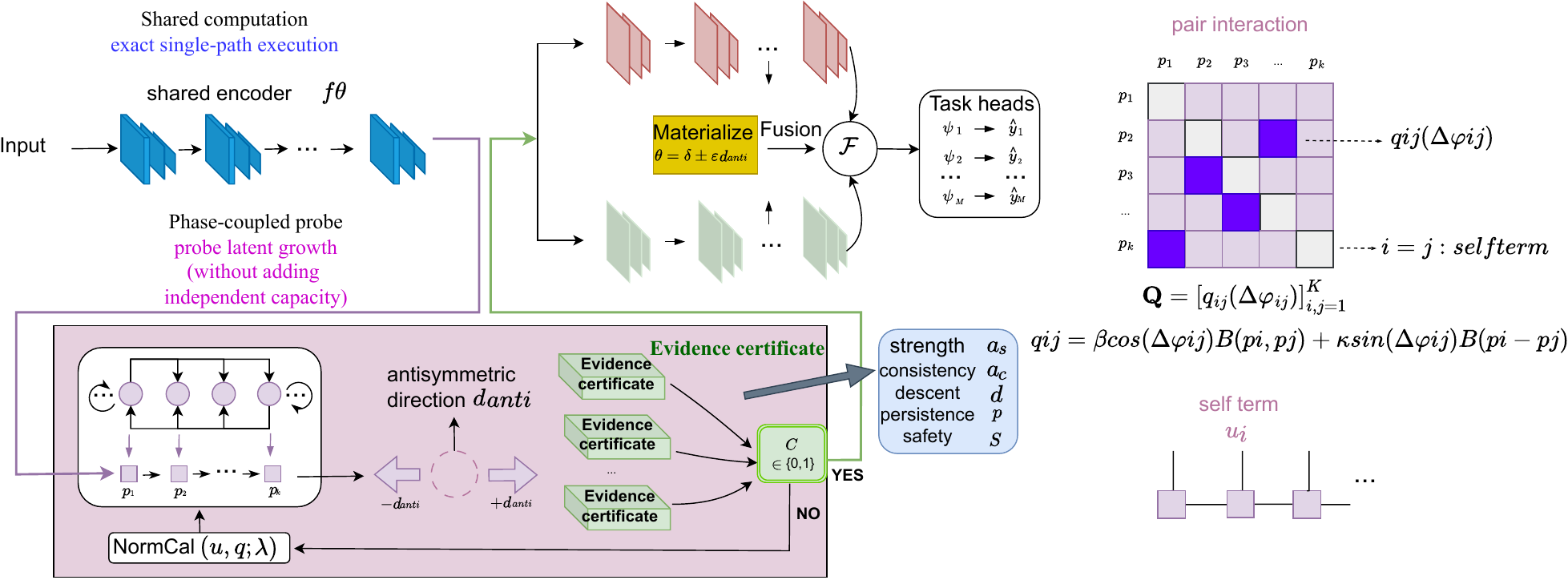}
    \caption{Overview of EMAN.
    The shared encoder performs exact single-path computation.
    Phase-coupled probing introduces latent relative phases and exposes an
    antisymmetric direction without adding independent path capacity. OEC
    evaluates direction strength, cross-probe consistency, finite-step
    descent, dynamic safety, and temporal persistence. A successful
    certificate materializes two equal-capacity paths through opposite
    perturbations. Their representations are fused and passed to the
    task-specific heads. The right panel shows the phase-coupled
    pair-interaction matrix. Off-diagonal entries represent pair
    interactions, and diagonal entries represent self terms.}
    \label{fig:method}
\end{figure*}

\subsection{Exact Shared Computation}
\label{sec:exact_sharing}

Let $x$ denote the input feature. The shared path produces
\begin{equation}
    h=f(x;\bar{\omega}),
    \label{eq:shared_path}
\end{equation}
with the shared parameter tensor $\bar{\omega}$. Before materialization,
EMAN stores only this tensor and performs one forward evaluation. Its output
is reused in two logical path coordinates:
\begin{equation}
    p_1=p_2=h.
    \label{eq:tied_paths}
\end{equation}
The logical duplication exposes symmetric and antisymmetric coordinates for
structural probing. It allows a zero-sum variation around $\bar{\omega}$
without storing or training a second parameter block. The coordinates remain
exchangeable and exactly tied, while the executed network still contains one
physical path. Parameter count, memory cost, and path computation therefore
remain equal to those of a shared model.

\subsection{Relative-Phase Interaction}
\label{sec:phase_interaction}

EMAN introduces two latent phase variables $\theta_1$ and $\theta_2$. Their
relative phase is
\begin{equation}
    \phi=\theta_1-\theta_2.
    \label{eq:relative_phase}
\end{equation}
The shared component and the phase-coupled pair component are
\begin{align}
    u &=
    \frac{1}{2}
    \left[A(p_1)+A(p_2)\right],
    \label{eq:shared_component}\\
    q_c &=
    \beta\cos(\phi)
    \left(p_1\odot p_2\right),
    \label{eq:cos_component}\\
    q_s &=
    \kappa\sin(\phi)
    \left(p_1-p_2\right),
    \label{eq:sin_component}\\
    q &=q_c+q_s.
    \label{eq:pair_component}
\end{align}

The cosine term captures symmetric pair interaction. The sine term captures
the antisymmetric response. At the tied point, $p_1-p_2=0$ and $q_s=0$, but
the derivative along the antisymmetric coordinate remains non-zero. EMAN
uses the fixed odd map
\begin{equation}
    B(d)=d.
\end{equation}
This map provides unit local gain at $d=0$. Relative-phase release can thus
expose an antisymmetric path force before path materialization. At this
stage, the phase variable changes the local response along the tied
antisymmetric coordinate. Physical capacity remains unchanged, and the
forward model still executes one path.

High-dimensional product interactions can dominate the shared component.
EMAN controls their magnitude through norm calibration:
\begin{equation}
    \widehat q
    =
    q\,
    \sg\!\left(
    \frac{\RMS(u)}
         {\RMS(q)+\varepsilon}
    \right),
    \qquad
    z=u+\rho\widehat q.
    \label{eq:normcal}
\end{equation}
The operator $\sg$ denotes stop-gradient. It blocks gradients through the
calibration ratio while controlling the forward magnitude of the pair term.
Norm calibration preserves the phase-dependent sign structure and
coefficient geometry. The same representation $z$ enters the task heads in
every structural state.

\subsection{Optimization Evidence Certification}
\label{sec:oec}

Capacity growth is irreversible, so one local descent observation is not
enough. Numerical noise, batch variation, and subset-specific behavior can
all produce an isolated signal. OEC combines multiple criteria and requires
persistent evidence before a structural transition.

Let $P_d$ project the logical path gradients onto the zero-sum
antisymmetric coordinate. The projected gradient at the tied state is
\begin{equation}
    g_d
    =
    P_d^\top
    \nabla_{\omega_1,\omega_2}
    \mathcal{L}
    \bigg|_{\omega_1=\omega_2=\bar{\omega}}.
    \label{eq:antisymmetric_gradient}
\end{equation}
The normalized candidate growth direction is
\begin{equation}
    d_{\mathrm{anti}}
    =
    -\frac{g_d}
    {\|g_d\|_2+\varepsilon}.
    \label{eq:danti}
\end{equation}

\noindent\textbf{Direction strength and consistency.}
The antisymmetric force must first exceed the frozen numerical-noise
threshold:
\begin{equation}
    S_t=\|g_d^{(t)}\|_2.
    \label{eq:oec_strength}
\end{equation}
EMAN also divides the training population into deterministic subsets and
constructs one candidate direction from each subset. Their agreement is
measured by
\begin{align}
    K_t^{\mathrm{half}}
    &=
    \left|
    \cos\!\left(
    d_{\mathrm{anti}}^{(t,a)},
    d_{\mathrm{anti}}^{(t,b)}
    \right)
    \right|,
    \label{eq:half_consistency}\\
    K_t^{\mathrm{shard}}
    &=
    \operatorname*{median}_{i<j}
    \left|
    \cos\!\left(
    d_{\mathrm{anti}}^{(t,i)},
    d_{\mathrm{anti}}^{(t,j)}
    \right)
    \right|.
    \label{eq:shard_consistency}
\end{align}
The absolute value removes the orientation ambiguity of the antisymmetric
axis. Both consistency scores must exceed their frozen thresholds.

\noindent\textbf{Descent compatibility and dynamic safety.}
OEC applies a finite zero-sum step along the candidate direction:
\begin{equation}
\begin{split}
    D_t
    =
    \mathcal{L}
    \left(
    \bar{\omega}+\eta d_{\mathrm{anti}},
    \bar{\omega}-\eta d_{\mathrm{anti}}
    \right)
    -
    \mathcal{L}
    \left(
    \bar{\omega},
    \bar{\omega}
    \right).
\end{split}
\label{eq:descent_test}
\end{equation}
The frozen test keeps the calibration state fixed. The dynamic test
recomputes it after the finite step. Both loss changes must be negative.
OEC also checks numerical stability, NormCal stability, path-exchange
symmetry, path-swap loss residuals, and output residuals. All values must
remain inside their frozen tolerances.

\noindent\textbf{Temporal persistence and certification.}
The instantaneous criteria must pass for a fixed number of consecutive probe
epochs. The evidence vector is
\begin{equation}
    e_t=
    \left(
    S_t,
    K_t^{\mathrm{half}},
    K_t^{\mathrm{shard}},
    D_t^{\mathrm{frozen}},
    D_t^{\mathrm{dynamic}},
    U_t
    \right),
    \label{eq:evidence_vector}
\end{equation}
where $U_t$ collects the dynamic safety indicators. The certification state
is
\begin{equation}
    C_t=
    \mathbb{I}
    \left[
    e_\tau\in\mathcal{E}_{\mathrm{safe}}
    \text{ for all }
    \tau\in\mathcal{W}_t
    \right].
    \label{eq:oec}
\end{equation}
The set $\mathcal{E}_{\mathrm{safe}}$ contains the frozen thresholds, and
$\mathcal{W}_t$ contains the required consecutive probe epochs.
The criteria address distinct failure modes. Strength separates the signal
from the numerical floor; cross-probe agreement rejects subset-specific
directions; descent tests local utility; safety rejects unstable probes; and
persistence removes transient events. Certification uses training data only.

\subsection{Conditional Path Materialization}
\label{sec:transition}

EMAN moves through three structural states. It starts in the exact shared
state:
\begin{equation}
    \omega_1=\omega_2=\bar{\omega},
    \qquad
    \theta_1=\theta_2=0.
    \label{eq:initial_state}
\end{equation}
Only the mother tensor $\bar{\omega}$ is physically stored, and each input
uses one path evaluation.

At scheduled probe epochs, EMAN evaluates the deterministic phase
certificate. Persistent and numerically safe evidence releases the latent
relative phase. The path parameters remain tied. A fixed adaptation window
then allows the phase to evolve, and OEC begins after this window.
Insufficient evidence retains exact shared computation.

A successful OEC decision gives $C_t=1$. EMAN copies the mother tensor into
two equal-capacity paths and applies the zero-sum perturbation
\begin{equation}
    \omega_1
    =
    \bar{\omega}
    +
    \epsilon_d d_{\mathrm{anti}},
    \qquad
    \omega_2
    =
    \bar{\omega}
    -
    \epsilon_d d_{\mathrm{anti}}.
    \label{eq:materialize}
\end{equation}
The arithmetic mean of the two initialized paths remains $\bar{\omega}$.
The transition opens an antisymmetric degree of freedom and preserves the
shared component at the release point.
The perturbation uses normalized parameter coordinates. The optimizer state
is copied symmetrically to both paths, so the transition preserves the
accumulated optimization history. The two paths then become independently
trainable. The task heads, task losses, optimizer, and representation
interface remain unchanged. The original multi-task objective drives all
later path differentiation.

\subsection{Local Phase-First Dynamics}
\label{sec:local_dynamics}

To describe phase-first path emergence, we consider a local expansion around
the shared state. Let $d$ denote the antisymmetric path coordinate, and let
$\phi$ denote the relative phase. The local risk near an exchange-symmetric
stationary branch is
\begin{equation}
\begin{split}
    \mathcal{R}(d,\phi)
    =
    \mathcal{R}_0
    &+
    \frac{a_\phi}{2}\phi^2
    +
    \frac{b_\phi}{4}\phi^4\\
    &+
    \frac{1}{2}d^\top A_d d
    +
    \phi c^\top d
    +
    \mathrm{h.o.t.}.
\end{split}
\label{eq:risk}
\end{equation}

The coefficient $b_\phi$ is positive, and the matrix $A_d$ is positive
definite. These conditions describe local stability along the tied path
coordinate. The cross term $\phi c^\top d$ couples relative phase and
antisymmetric path separation.

\begin{proposition}[Phase-First Antisymmetric Exposure]
\label{prop:phase_first}
Assume $a_\phi<0$. Restriction of Eq.~\eqref{eq:risk} to $d=0$ gives two
local stationary phase offsets:
\begin{equation}
    \phi_\pm
    =
    \pm\sqrt{-a_\phi/b_\phi}
    +
    o\!\left(\sqrt{|a_\phi|}\right).
\end{equation}
These offsets generate opposite antisymmetric path forces:
\begin{equation}
    \nabla_d\mathcal{R}(0,\phi_\pm)
    =
    c\phi_\pm
    +
    o(|\phi_\pm|).
\end{equation}
Path release gives the conditional local branch
\begin{equation}
    d(\phi)
    =
    -A_d^{-1}c\phi
    +
    o(|\phi|).
\end{equation}
\end{proposition}

The two phase offsets select opposite orientations of the same
antisymmetric branch. The conditional branch links the exposed force to the
initial direction of path separation. This ordering motivates phase-first
probing: relative phase can move before the physical paths separate.

\section{Experiments}
\label{sec:experiments}

The experiments first isolate capacity demand under controlled rank
constraints. We then examine certification dynamics, released
representations, and the performance--computation trade-off on
natural-image multi-task benchmarks.

\begin{figure}[t]
    \centering

    \includegraphics[width=\linewidth]
    {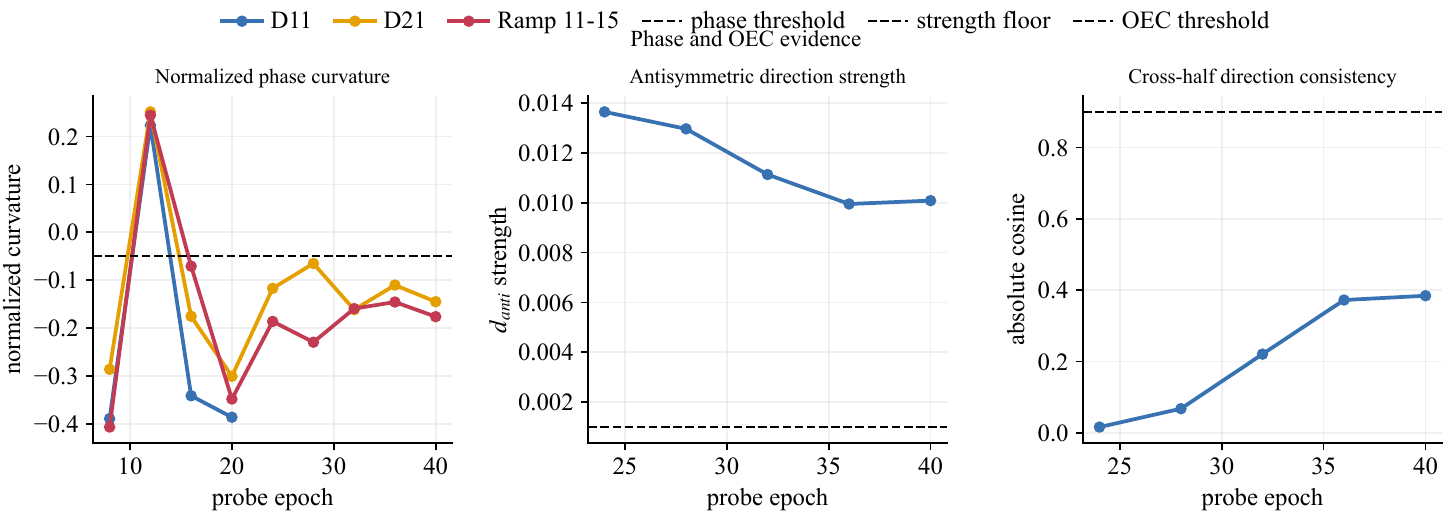}
    \par
    {\small (a) Internal phase and OEC dynamics.\par}

    \includegraphics[width=\linewidth]
    {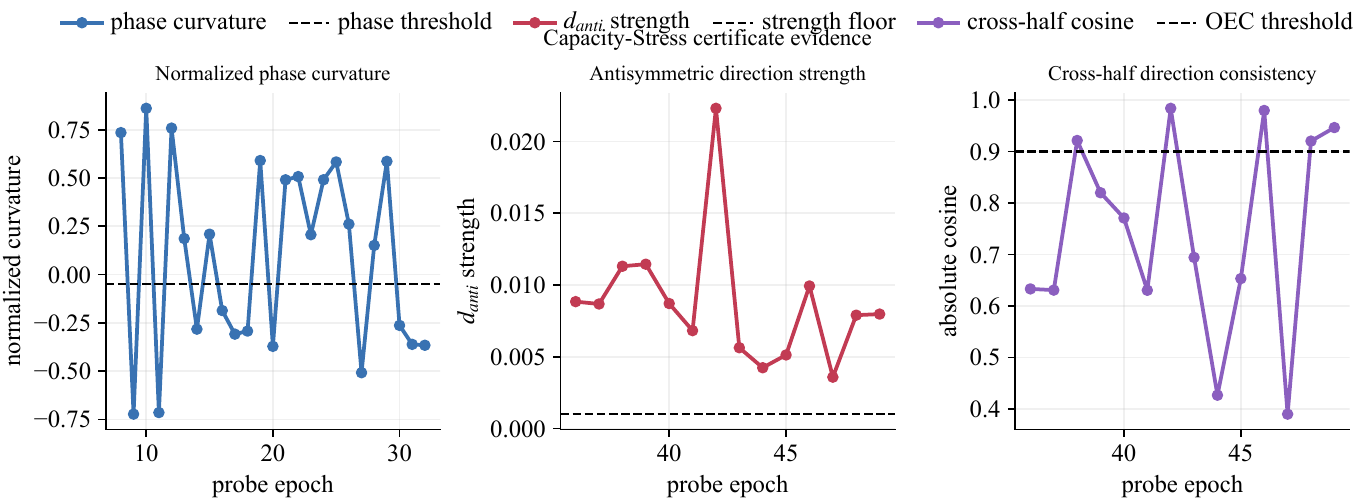}
    \par
    {\small (b) Population-level Capacity-Stress certification.\par}

    \includegraphics[width=\linewidth]
    {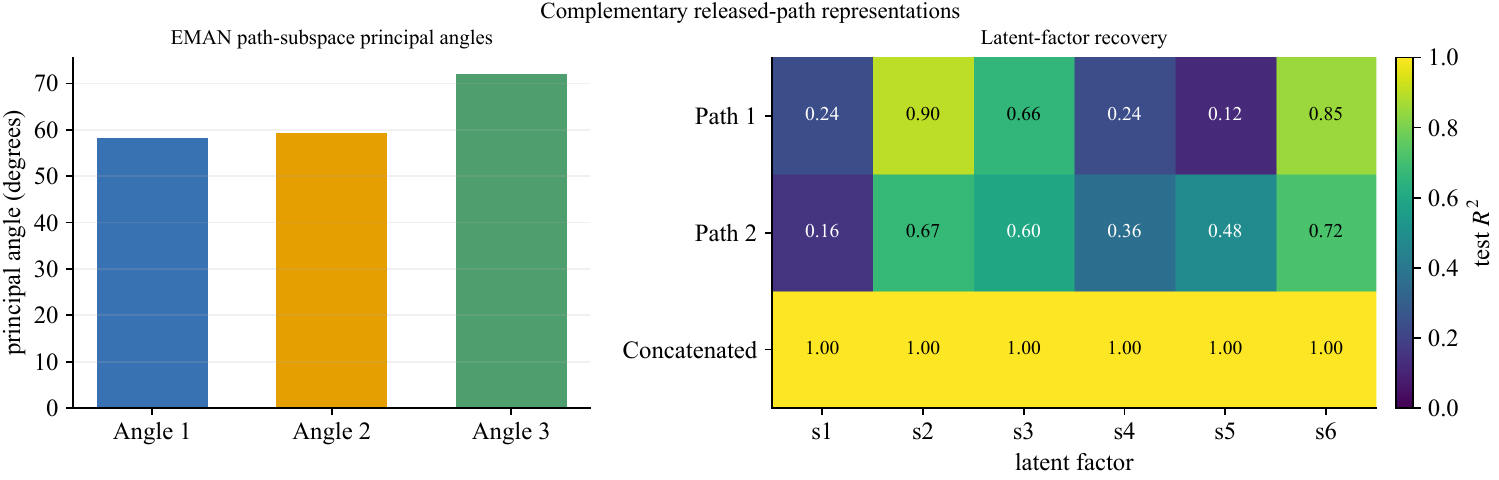}
    \par
    {\small (c) Released-path representation geometry.\par}

    \caption{Certification and controlled-capacity diagnostics.
    (a) Phase curvature, antisymmetric strength, and cross-half
    consistency; dashed lines mark frozen thresholds.
    (b) Population-level evidence satisfies the frozen certificate
    under Capacity-Stress.
    (c) Released paths span complementary subspaces and recover all
    latent factors.}
    \label{fig:oec_controlled_geometry}
\end{figure}

\subsection{Experimental Setup}
\label{sec:exp_setup}

\paragraph{Benchmarks.}
The controlled benchmark contains six latent factors and three regression
tasks. Joint target ranks are 3, 4, and 6 for \textsc{High},
\textsc{Medium}, and \textsc{Low}. Each physical path has rank 3.
NYUv2 contains semantic segmentation, depth estimation, and surface-normal
prediction \citep{silberman2012indoor}. Standard capacity uses 64 channels.
Capacity-Stress uses 8 channels, or $0.125\times$ path width.
PASCAL-Context provides a second dataset and a different task composition
\citep{mottaghi2014context}. Task-arrival protocols delay part of the task set and leave the transition
rule unchanged. Earlier-arrival and gradual-demand variants change the
demand onset, so they isolate timing sensitivity without threshold
retuning.

\begin{table}[t]
\centering
\small
\renewcommand{\arraystretch}{1.06}

\begin{tabular*}{\linewidth}{@{\extracolsep{\fill}}llcc@{}}
\toprule
Regime & Method & Macro $R^2$ & Phase/Path rel. \\
\midrule
\multicolumn{4}{l}{\textit{Sufficient capacity: joint rank 3}} \\
\textsc{High} & Hard-sharing
& $\pmb{0.9021\pm0.0850}$ & -- ; $0/3$ \\
& EMAN
& $\pmb{0.9021\pm0.0850}$ & $0/3$ ; $0/3$ \\
\midrule
\multicolumn{4}{l}{\textit{Insufficient capacity: joint rank 6}} \\
\textsc{Low} & Hard-sharing
& $0.4838\pm0.0207$ & -- ; $0/3$ \\
& PhaseOnly
& $0.4834\pm0.0208$ & $3/3$ ; $0/3$ \\
& EMAN
& $\pmb{0.8661\pm0.1252}$
& $\pmb{3/3\ ;\ 3/3}$ \\
\bottomrule
\end{tabular*}

\caption{Controlled capacity selection.
Mean $\pm$ sample standard deviation over three seeds.
Phase/Path rel. reports the numbers of seeds with phase release and path
materialization, respectively. Best results and key EMAN release outcomes are
shown in bold.}
\label{tab:controlled_capacity}
\end{table}

\paragraph{Metrics and baselines.}
We report native task metrics and the frozen directional aggregate
\begin{equation}
    D_t(M)=
    s_t\frac{M_t-M_t^{\mathrm{HS}}}
    {|M_t^{\mathrm{HS}}|+\varepsilon_t},
    \qquad
    G(M)=\frac{1}{T}\sum_{t=1}^{T}D_t(M),
    \label{eq:aggregate}
\end{equation}
where $s_t=1$ for higher-is-better metrics and $s_t=-1$ for
lower-is-better metrics. The constant $\varepsilon_t$ prevents division by
zero, and HS denotes Hard-sharing. Positive $G$ denotes an average directional improvement over Hard-sharing.
The native task metrics show which tasks contribute to that change. Controlled comparisons include
Hard-sharing, PhaseOnly, MeanPaths, and EMAN. Natural-image comparisons also
include AdaShare, Recon, Conflict Trigger, and Scheduled split. Every method
uses the same data split, augmentation, heads, losses, optimizer, scheduler,
and training horizon. Structural decisions use training data only. Full
settings and per-seed results appear in the supplementary material.

\subsection{Controlled Capacity and Certification}
\label{sec:controlled_capacity}

Figure~\ref{fig:oec_controlled_geometry}(a,b) connects internal optimization
evidence to the structural decision. In panel (a), phase-curvature crossings
are transient. Antisymmetric strength remains above the numerical floor, but
cross-half consistency does not persist long enough for certification.
Panel (b) shows a population-level certificate under Capacity-Stress.
Table~\ref{tab:controlled_capacity} evaluates the same decision principle
under explicit rank constraints. In \textsc{High}, the joint task rank equals
the rank-3 path capacity. EMAN remains shared in all seeds and matches
Hard-sharing at $0.9021\pm0.0850$ macro-$R^2$. In \textsc{Low}, the joint
rank rises to 6. Hard-sharing reaches $0.4838\pm0.0207$. PhaseOnly releases
phase in all seeds, retains one physical path, and reaches
$0.4834\pm0.0208$. EMAN releases both phase and path in all seeds and reaches
$0.8661\pm0.1252$. Phase release reveals the growth direction; path
materialization supplies the missing representation rank.

The rank-4 \textsc{Medium} regime lies near the one-path capacity limit.
EMAN remains shared and matches Hard-sharing at $0.7086$, while PhaseOnly
reaches $0.7081$. Its evidence does not satisfy the frozen persistence
criterion. Figure~\ref{fig:oec_controlled_geometry}(c) examines the
representation after release under rank-6 demand. Principal angles range
from $58.3^{\circ}$ to $72.0^{\circ}$. The concatenated representation
recovers all six latent factors with $R^2\geq0.999892$. The permutation
control reduces mean recovery to $0.0243$. The released paths therefore add
complementary representation capacity.

\subsection{Real-Data Performance and Computation}
\label{sec:real_performance}

Table~\ref{tab:nyuv2_results}(a) reports standard-capacity NYUv2 results.
EMAN has the best mean mIoU, depth error, and $G$; Recon has the best normal
error and joint loss. EMAN reaches $G=0.0078\pm0.0152$, with seed values
$0.0253$, $0.0005$, and $-0.0023$. Its clearest gains appear in semantic
segmentation and depth estimation. Table~\ref{tab:nyuv2_results}(b) reports structural selection and computation
under Capacity-Stress. EMAN achieves the highest $G$ at
$0.0306\pm0.0049$, compared with $0.0237\pm0.0086$ for MeanPaths and
$0.0115\pm0.0048$ for PhaseOnly. It materializes the second path in two of
three seeds and remains shared in the other. Its dual-active ratio is
$0.392$, and it uses $69.6\%$ of MeanPaths path-training FLOPs.

With task arrival at epoch 21, EMAN reaches
$G=0.0226\pm0.0081$, compared with $0.0202\pm0.0076$ for Scheduled split and
$0.0114\pm0.0121$ for MeanPaths. EMAN and Scheduled split use $75\%$ of
MeanPaths path-training FLOPs. Release occurs near the demand change, and active dual-path training covers
roughly the latter half of the schedule. Their wall-clock times remain close
to MeanPaths.

\begin{table*}[!t]
\centering
\small
\setlength{\tabcolsep}{3.2pt}
\renewcommand{\arraystretch}{1.06}

\begin{tabular*}{\textwidth}
{@{\extracolsep{\fill}}lcccccc@{}}
\toprule
\multicolumn{7}{c}{\textit{(a) Standard-capacity performance}} \\
\midrule
\multicolumn{1}{c}{Method}
& mIoU $\uparrow$
& Depth Abs. $\downarrow$
& Normal Mean $\downarrow$
& Joint Loss $\downarrow$
& $G\uparrow$
& Params (M) \\
\midrule

Hard-sharing
& $0.4461\pm0.0035$
& $0.4344\pm0.0016$
& $27.3616\pm0.1001$
& $0.9469\pm0.0221$
& $\phantom{-}0.0000\pm0.0000$
& $26.20$ \\

MeanPaths
& $0.4441\pm0.0039$
& $0.4346\pm0.0038$
& $27.3705\pm0.0217$
& $0.9445\pm0.0325$
& $-0.0018\pm0.0024$
& $28.82$ \\

AdaShare
& $0.1776\pm0.1091$
& $0.6497\pm0.1092$
& $36.9200\pm3.2958$
& $1.8019\pm0.7807$
& $-0.4818\pm0.2022$
& $26.20$ \\

Recon
& $0.4439\pm0.0018$
& $0.4372\pm0.0081$
& $\pmb{27.0515\pm0.2439}$
& $\pmb{0.9317\pm0.0183}$
& $\phantom{-}0.0000\pm0.0049$
& $27.02$ \\

EMAN
& $\pmb{0.4472\pm0.0087}$
& $\pmb{0.4279\pm0.0089}$
& $27.1913\pm0.3328$
& $0.9414\pm0.0367$
& $\phantom{-}\pmb{0.0078\pm0.0152}$
& $31.44$ \\

\midrule
\multicolumn{7}{c}{\textit{(b) Capacity-Stress: selection and computation}} \\
\midrule
\multicolumn{1}{c}{Method}
& $G\uparrow$
& Phase rel.
& Path rel.
& Dual-active
& Path FLOPs
& Total FLOPs \\
\midrule

Hard-sharing
& $0.0000\pm0.0000$
& --
& $0/3$
& $0.000$
& $0.500$
& $0.962$ \\

PhaseOnly
& $0.0115\pm0.0048$
& $2/3$
& $0/3$
& $0.000$
& $0.500$
& $0.962$ \\

MeanPaths
& $0.0237\pm0.0086$
& --
& $3/3$ at 1
& $1.000$
& $1.000$
& $1.000$ \\

Conflict Trigger
& $0.0000\pm0.0000$
& --
& $1/3$ at 79
& $0.008$
& $0.504$
& $0.963$ \\

EMAN
& $\pmb{0.0306\pm0.0049}$
& $\pmb{2/3}$
& $\pmb{2/3}$
& $0.392$
& $0.696$
& $0.977$ \\

\bottomrule
\end{tabular*}

\caption{NYUv2 performance and computation.
Panel (a) reports standard-capacity task performance and parameter counts.
Panel (b) reports structural selection and computation under Capacity-Stress.
Task metrics and $G$ are reported as mean $\pm$ sample standard deviation
over three seeds. Phase/Path rel. reports the number of seeds with phase
release and path materialization. Best mean performance values and EMAN
release counts are shown in bold. Capacity-Stress uses $0.125\times$ path
width; FLOPs are normalized by MeanPaths.}
\label{tab:nyuv2_results}
\end{table*}

\paragraph{Cross-dataset evaluation.}
PASCAL-Context tests the same structural rule under a different dataset and
task composition. In the epoch-21 task-arrival setting, EMAN releases at the
demand change in all three seeds and reaches $G=0.0043\pm0.0042$.
MeanPaths reaches $0.0027\pm0.0018$, and Hard-sharing reaches $0$. EMAN uses
$75\%$ of the path-training compute of always-dual MeanPaths. Earlier task
arrival moves the same frozen release from epoch 21 to epoch 11. EMAN reaches
$G=0.0037$, while the matched Scheduled split reaches $0.0020$. The release
time thus follows the observed demand onset without access to the task
schedule. Static demand tests immediate capacity pressure from the first epoch and
therefore favors an early structural response. Gradual demand asks whether
the evidence can accumulate under a smoother onset instead of one abrupt
task-arrival event. Detailed task-wise results, per-seed
traces, wall-clock profiles, and replay manifests appear in the supplementary
material.

\section{Discussion and Limitations}
\label{sec:discussion}

The central result concerns the timing of capacity, not only its final
amount. EMAN separates structural diagnosis from structural commitment. A
useful local direction does not yet establish a persistent capacity deficit.
Phase probing identifies a possible direction of growth; OEC decides whether
its strength, consistency, descent, safety, and persistence justify new
parameters.

The controlled experiments expose this distinction clearly. One path remains
sufficient at joint rank three. Joint rank six creates a clear deficit, and
materialization recovers the missing factors. The released paths also occupy
complementary latent subspaces rather than duplicate the shared
representation. Natural-image experiments show the same timing principle.
Delayed demand produces delayed release, earlier demand shifts the same frozen
decision from epoch 21 to epoch 11, and static demand produces an early
release. The value of path emergence therefore depends on when additional
capacity becomes useful and how long that demand persists. Delayed and
persistent demand provides the clearest setting for EMAN: a meaningful
shared-training stage precedes materialization.

The current certificate favors reliable evidence over aggressive expansion.
The \textsc{Medium} regime exposes this trade-off. A strict persistence rule
avoids premature growth, but it can miss a mild bottleneck. The reported
savings measure path-training evaluations before release. Both paths remain
active after materialization and during inference. Future work can address
these two limits separately: improve sensitivity near the capacity boundary
and reduce post-release cost through compression, path reuse, or conditional
execution.

\section{Conclusion}
\label{sec:conclusion}

Model capacity need not be fixed before optimization begins. EMAN starts
from exact single-path computation and uses phase-first evidence to decide
whether independent capacity should emerge. OEC turns a local
antisymmetric direction into a persistent structural decision, and the
zero-sum transition preserves the accumulated optimization state at release.
Controlled experiments show exact sharing under sufficient rank, capacity
growth under a clear deficit, and complementary representation recovery.
The PhaseOnly comparison also shows that direction discovery alone does not
supply missing representation capacity. NYUv2 and PASCAL-Context demonstrate
adaptive release under changing task demand, competitive final performance,
and reduced path-training computation.

These results position capacity growth as part of optimization itself. EMAN
offers a practical route from fixed hard sharing to evidence-driven
structural expansion.

\bibliography{references}

\input{supplement}
\end{document}

%% file: supplement.tex
\title{Supplementary Material for\\
EMAN: Optimization-Driven Capacity Growth through Path Emergence in Multi-Task Learning}
\author{Anonymous Submission}
\affiliations{}

\appendix

\section{Extended Theory and Method Details}
\label{sec:supp_theory}

\subsection{Physical Paths, Logical Coordinates, and Structural States}
\label{sec:supp_states}

EMAN separates a physical computation graph from a local coordinate system.
Before materialization, the model stores one trainable path with parameters
$\bar\omega$. It evaluates this path once:
\begin{equation}
 h=f(x;\bar\omega).
\end{equation}
The model then defines two logical coordinates with the same output:
\begin{equation}
 p_1=p_2=h.
\end{equation}
These coordinates support symmetry analysis and directional probes. They do
not create two encoders. They do not double path memory. They do not add an
independent representation parameterization.

The structural state belongs to
\begin{equation}
 \mathcal S_t\in
 \left\{\begin{array}{l}
 \textsc{Exact-Shared},\;\textsc{Phase-Released},\\
 \textsc{Path-Released}
 \end{array}\right\}.
\end{equation}
The state sequence is monotone:
\begin{equation}
 \begin{aligned}
 \textsc{Exact-Shared}&\longrightarrow\textsc{Phase-Released}\\
 &\longrightarrow\textsc{Path-Released}.
 \end{aligned}
 \label{eq:supp_state_sequence}
\end{equation}
A failed certificate keeps the current state. The implementation has no
fallback split. It has no validation-driven split. It has no task-ID trigger.

For two logical outputs, the representation map is
\begin{align}
 u &= \tfrac12\left[A(p_1)+A(p_2)\right],\\
 q_c &= \beta\cos(\phi)(p_1\odot p_2),\\
 q_s &= \kappa\sin(\phi)(p_1-p_2),\\
 q &= q_c+q_s,\\
 \widehat q &= q\,\sg\!\left(
 \frac{\RMS(u)}{\RMS(q)+\varepsilon}\right),\\
 z &= u+\rho\widehat q,
 \label{eq:supp_forward}
\end{align}
with relative phase $\phi=\theta_1-\theta_2$. The odd branch uses the fixed
map $B(d)=d$. The same task heads consume $z$ in every structural state.

\begin{table}[!t]
\centering

\small
\begin{tabular}{p{0.31\linewidth}p{0.58\linewidth}}
\toprule
State & Physical interpretation \\
\midrule
\textbf{\textsc{Exact-Shared}} & One stored path, one path evaluation, $\phi=0$, and no independent path coordinate. \\
\textbf{\textsc{Phase-Released}} & One stored path and one path evaluation. The relative phase can adapt. Path parameters remain tied. \\
\textbf{\textsc{Path-Released}} & Two stored paths. A certified zero-sum perturbation initializes them. They then optimize independently. \\
\bottomrule
\end{tabular}
\caption{Structural states. Logical duplication never implies physical duplication.}
\label{tab:supp_states}
\end{table}

Table~\ref{tab:supp_states} separates logical coordinates from trainable
parameter blocks and physical path evaluations. The exact-shared and
phase-released states both use one encoder evaluation. They differ only in the
availability of the latent relative phase. Independent representation capacity
appears only after path release. PhaseOnly can therefore change the probing
geometry, but it cannot instantiate a second encoder.

\subsubsection{State invariants and accounting}
Let $N_t^{\mathrm{phys}}$ denote the number of stored physical paths,
$N_t^{\mathrm{eval}}$ the number of path evaluations per input, and
$N_t^{\mathrm{ind}}$ the number of independently trainable path parameter
blocks. The three states satisfy
\begin{equation}
\begin{array}{c|ccc}
\mathcal S_t & N_t^{\mathrm{phys}} & N_t^{\mathrm{eval}} & N_t^{\mathrm{ind}}\\
\hline
\textsc{Exact-Shared} & 1 & 1 & 1\\
\textsc{Phase-Released} & 1 & 1 & 1\\
\textsc{Path-Released} & 2 & 2 & 2
\end{array}.
\label{eq:supp_state_accounting}
\end{equation}
The logical pair $(p_1,p_2)$ exists in all states, but it does not change these
physical counts. Consequently, the phase-probing stage can inspect a candidate
split direction without paying the path-training cost of an always-dual model.
After materialization, the two physical paths remain active for both training
and inference. The reported savings therefore concern the pre-release training
interval rather than sparse post-release execution.

\subsection{Exchange Symmetry and Coordinate Decomposition}
\label{sec:supp_symmetry}

Define the exchange operator
\begin{equation}
 \Pi(p_1,p_2,\theta_1,\theta_2)
 =(p_2,p_1,\theta_2,\theta_1).
\end{equation}
The symmetric and antisymmetric parameter coordinates are
\begin{equation}
 \bar\omega=\tfrac12(\omega_1+\omega_2),
 \qquad
 d=\tfrac12(\omega_1-\omega_2).
 \label{eq:supp_coordinates}
\end{equation}
The phase coordinates are
\begin{equation}
 \bar\theta=\tfrac12(\theta_1+\theta_2),
 \qquad
 \phi=\theta_1-\theta_2.
\end{equation}
Exchange preserves $\bar\omega$ and $\bar\theta$. It maps $d$ to $-d$ and
$\phi$ to $-\phi$.

\begin{proposition}[Exchange invariance]
The representation in Eq.~\eqref{eq:supp_forward} is invariant under $\Pi$.
\end{proposition}

\begin{proof}
The shared term is invariant because addition is commutative. The cosine term
is invariant because cosine is even and the product is symmetric. The sine
term is invariant because both factors change sign:
\begin{equation}
 \sin(-\phi)(p_2-p_1)=\sin(\phi)(p_1-p_2).
\end{equation}
The normalization depends only on root-mean-square magnitudes. It is also
invariant. The fused representation is therefore invariant.
\end{proof}

\begin{corollary}[No accidental path label]
The tied state has no preferred path index. A path swap leaves every task
output unchanged.
\end{corollary}

This symmetry matters for certification. The antisymmetric axis has an
orientation ambiguity. A candidate direction $d_{\mathrm{anti}}$ and its
negative represent the same separation axis. OEC therefore uses absolute
cosine agreement across probes.

At the exact tied state, exchange invariance also constrains local derivatives.
For a smooth invariant risk $\mathcal R(\bar\omega,d,\bar\theta,\phi)$,
\begin{equation}
 \mathcal R(\bar\omega,d,\bar\theta,\phi)
 =\mathcal R(\bar\omega,-d,\bar\theta,-\phi).
\end{equation}
Setting $\phi=0$ makes the restriction even in $d$. Therefore,
\begin{equation}
\nabla_d\mathcal R(\bar\omega,0,\bar\theta,0)=0.
\end{equation}
This zero gradient does not rule out a mixed phase--path instability. The product $\phi c^\top d$ is
exchange invariant because both factors change sign. Its mixed derivative can
therefore be nonzero at the tied state:
\begin{equation}
\begin{aligned}
 H_{d\phi}
 &=\left.\frac{\partial^2\mathcal R}
 {\partial d\,\partial\phi}\right|_{d=0,\phi=0}\\
 &=c.
\end{aligned}
\end{equation}
Phase release activates this mixed block and converts a latent coupling into a
measurable antisymmetric force. This is the local reason for probing phase
before allocating a second physical path.

\subsection{What Phase Release Can and Cannot Add}
\label{sec:supp_phase_capacity}

At the tied point, $p_1=p_2=h$. The odd branch vanishes:
\begin{equation}
 q_s=\kappa\sin(\phi)(h-h)=0.
\end{equation}
The even branch remains a function of the same path output:
\begin{equation}
 q_c=\beta\cos(\phi)(h\odot h).
\end{equation}
Phase release can therefore alter pair-interaction coefficients. It can expose
local derivatives along the antisymmetric coordinate. It cannot create a
second independently trainable encoder. The PhaseOnly control isolates this
point empirically. It releases phase but keeps one physical path.

\begin{proposition}[One-path parameter capacity before materialization]
Before path materialization, every trainable encoder parameter belongs to
$\bar\omega$. Phase release does not add an independent encoder parameter
block.
\end{proposition}

\begin{proof}
The implementation stores only $\bar\omega$. Both logical path outputs reuse
$f(x;\bar\omega)$. The trainable phase variables affect interaction
coefficients. They do not introduce a second map $f(x;\omega_2)$ with
independent parameters.
\end{proof}

The proposition concerns physical representation capacity. It does not claim
that phase has no functional effect. Phase changes the local interaction
geometry. That effect is the probe mechanism.

\subsection{Norm Calibration Preserves the Local Geometry}
\label{sec:supp_normcal}

Define
\begin{equation}
 \gamma(u,q)=\frac{\RMS(u)}{\RMS(q)+\varepsilon}.
\end{equation}
The forward scale is positive when $\RMS(u)>0$. The stop-gradient operator
freezes its backward derivative.

\begin{proposition}[Forward amplitude control]
For $q\neq0$ and $\RMS(u)>0$,
\begin{equation}
 \RMS(\widehat q)
 =\RMS(u)\frac{\RMS(q)}{\RMS(q)+\varepsilon}
 <\RMS(u).
\end{equation}
The calibration preserves the sign of every scalar coefficient in $q$. It also preserves exchange invariance. If $\RMS(u)=0$, the displayed relation is non-strict rather than strict.
\end{proposition}

\begin{proof}
The scale $\gamma$ is positive. Multiplication by a positive scalar preserves
all signs and all pairwise directions. Root-mean-square magnitude is
homogeneous. The stated identity follows directly. Exchange invariance follows
from the invariance of $\RMS(u)$ and $\RMS(q)$.
\end{proof}

Norm calibration controls amplitude. It does not create an antisymmetric
direction. It does not change the odd or even phase geometry. It also does not
serve as a growth certificate.

The stop-gradient is important for this separation of roles. During backward
propagation, the calibration factor is treated as a constant. For a local
perturbation $\delta q$,
\begin{equation}
 \delta\widehat q=\sg(\gamma)\,\delta q.
\end{equation}
Thus, the backward map applies a positive scalar to the pair-interaction
variation. It does not add a derivative through the ratio of RMS values. A
fully differentiable normalizer would introduce extra terms proportional to
$\nabla_q\RMS(q)$ and could rotate or attenuate the very direction that OEC is
attempting to measure. NormCal therefore stabilizes forward amplitude while
leaving the intended local coefficient geometry intact.

\subsection{Local Phase--Path Expansion}
\label{sec:supp_local_expansion}

Let $d$ denote the antisymmetric path coordinate. Let $\phi$ denote the
relative phase. Consider the local risk
\begin{equation}
\mathcal R(d,\phi)=\mathcal R_0+
\frac{a_\phi}{2}\phi^2+
\frac{b_\phi}{4}\phi^4+
\frac12d^\top A_dd+
\phi c^\top d+r(d,\phi).
\label{eq:supp_risk}
\end{equation}
Assume $b_\phi>0$ and $A_d\succ0$. Let $r$ be sufficiently smooth near the origin, with derivatives satisfying the stated higher-order remainder conditions. In particular, $\nabla_d r(d,\phi)=o(\|d\|+|\phi|)$. The cross term couples phase and path separation.

\begin{proposition}[Phase-first antisymmetric exposure]
As $a_\phi\to0^{-}$, assume $a_\phi<0$. The tied restriction $d=0$ has two local nonzero phase
stationary points:
\begin{equation}
 \phi_\pm=\pm\sqrt{-a_\phi/b_\phi}
 +o(\sqrt{|a_\phi|}).
 \label{eq:supp_phi_offsets}
\end{equation}
They produce opposite path forces:
\begin{equation}
 \nabla_d\mathcal R(0,\phi_\pm)
 =c\phi_\pm+o(|\phi_\pm|).
 \label{eq:supp_opposite_force}
\end{equation}
The conditional path branch satisfies
\begin{equation}
 d(\phi)=-A_d^{-1}c\phi+o(|\phi|).
 \label{eq:supp_conditional_branch}
\end{equation}
\end{proposition}

\begin{proof}
Set $d=0$. Phase stationarity gives
\begin{equation}
 a_\phi\phi+b_\phi\phi^3+o(|\phi|^3)=0.
\end{equation}
The two nonzero roots give Eq.~\eqref{eq:supp_phi_offsets}. Differentiate
Eq.~\eqref{eq:supp_risk} with respect to $d$. Evaluation at $d=0$ gives
Eq.~\eqref{eq:supp_opposite_force}. The matrix $A_d$ is nonsingular, and $\nabla_d r(d,\phi)=o(\|d\|+|\phi|)$. The implicit function theorem therefore gives a local branch with derivative
$d'(0)=-A_d^{-1}c$.
\end{proof}

\begin{corollary}[Sign reversal]
The leading path response is odd in phase:
\begin{equation}
 d(-\phi)=-d(\phi)+o(|\phi|).
\end{equation}
The two phase probes therefore expose opposite orientations of the same local
separation axis.
\end{corollary}

Eliminating the locally stable path coordinate gives an equivalent reduced
phase risk. Substituting $d(\phi)=-A_d^{-1}c\phi+o(|\phi|)$ into
Eq.~\eqref{eq:supp_risk} yields
\begin{equation}
 \mathcal R_{\mathrm{red}}(\phi)=\mathcal R_0+
 \frac12\left(a_\phi-c^\top A_d^{-1}c\right)\phi^2+
 \frac{b_\phi}{4}\phi^4+o(\phi^2).
 \label{eq:supp_reduced_risk}
\end{equation}
The coupling therefore lowers the effective phase curvature by
$c^\top A_d^{-1}c$. This connects the phase-first branch calculation to the
Schur-complement condition below. A stable restricted path block can coexist
with a descending coupled mode.

The result is local. It provides a sufficient mechanism. It does not make
phase necessary for all growth events. It does not prove global convergence.
It does not prove persistent specialization after release.

\subsection{Schur-Complement View of Coupled Curvature}
\label{sec:supp_schur}

The quadratic Hessian block at the tied state is
\begin{equation}
 H=
 \begin{bmatrix}
 A_d & c\\
 c^\top & a_\phi
 \end{bmatrix}.
\end{equation}

\begin{proposition}[Coupled negative mode]
Assume $A_d\succ0$. The matrix $H$ has a negative eigenvalue if and only if
\begin{equation}
 a_\phi-c^\top A_d^{-1}c<0.
 \label{eq:supp_schur_condition}
\end{equation}
\end{proposition}

\begin{proof}
Apply the congruence transform
\begin{equation}
 \begin{bmatrix}
 I & -A_d^{-1}c\\0&1
 \end{bmatrix}^{\!\top}
 H
 \begin{bmatrix}
 I & -A_d^{-1}c\\0&1
 \end{bmatrix}
 =
 \begin{bmatrix}
 A_d&0\\0&a_\phi-c^\top A_d^{-1}c
 \end{bmatrix}.
\end{equation}
Congruence preserves inertia. The result follows.
\end{proof}

This view separates block stability from coupled stability. Both restricted
blocks can be positive while the coupled system has a negative mode. OEC does
not use this algebraic condition as a direct trigger. It uses measured
training-time evidence.

\subsection{Finite-Step Descent and the OEC Descent Test}
\label{sec:supp_descent}

Let $g_d$ denote the projected gradient on the zero-sum coordinate. The
implemented direction is
\begin{equation}
 v_d=-\frac{g_d}{\|g_d\|_2+\varepsilon}.
\end{equation}
For the local bound, assume $g_d\neq0$ and use the associated unit direction
$\widetilde v_d=-g_d/\|g_d\|_2$.

\begin{lemma}[Small-step descent]
Assume the directional gradient is $L_d$-Lipschitz near the tied point. For
$\widetilde v_d=-g_d/\|g_d\|_2$,
\begin{equation}
 \mathcal L(\eta \widetilde v_d)-\mathcal L(0)
 \le -\eta\|g_d\|_2+\frac{L_d}{2}\eta^2.
 \label{eq:supp_descent_bound}
\end{equation}
The step is descending for $0<\eta<2\|g_d\|_2/L_d$.
\end{lemma}

\begin{proof}
Apply the descent lemma to the scalar restriction along
$\widetilde v_d$. The first-order term is $-\eta\|g_d\|_2$. The remainder is
bounded by $L_d\eta^2/2$.
\end{proof}

OEC does not estimate $L_d$. It directly evaluates a finite step. It performs
a frozen-calibration test and a dynamically recalibrated test. Both deltas
must be negative. This design guards against a direction with a favorable
surrogate but an unfavorable deployed forward map.

\subsection{Formal Optimization Evidence Certification}
\label{sec:supp_oec}

At probe epoch $t$, let
\begin{equation}
 g_d^{(t)}=P_d^\top\nabla_{\omega_1,\omega_2}\mathcal L_t.
\end{equation}
The normalized candidate direction is
\begin{equation}
 d_{\mathrm{anti}}^{(t)}=-
 \frac{g_d^{(t)}}{\|g_d^{(t)}\|_2+\varepsilon}.
\end{equation}
OEC evaluates five evidence groups.

\paragraph{Direction strength.}
\begin{equation}
 S_t=\|g_d^{(t)}\|_2.
\end{equation}
It must exceed a frozen numerical floor.

\paragraph{Cross-probe consistency.}
Split the training population into two deterministic halves. Each half yields
a candidate direction. The half score is
\begin{equation}
 K_t^{\mathrm{half}}=
 \left|\cos\!\left(
 d_{\mathrm{anti}}^{(t,a)},d_{\mathrm{anti}}^{(t,b)}
 \right)\right|.
\end{equation}
Fixed shards produce the median score
\begin{equation}
 K_t^{\mathrm{shard}}=
 \operatorname*{median}_{i<j}
 \left|\cos\!\left(
 d_{\mathrm{anti}}^{(t,i)},d_{\mathrm{anti}}^{(t,j)}
 \right)\right|.
\end{equation}

\paragraph{Finite-step descent.}
\begin{align}
 \Delta_t^{\mathrm{frozen}}&<0,\\
 \Delta_t^{\mathrm{dynamic}}&<0.
\end{align}
The second test recomputes the dynamic calibration state.

\paragraph{Dynamic safety.}
Path-swap loss residuals, output residuals, finite-value checks, and NormCal
checks must stay inside frozen tolerances.

\paragraph{Temporal persistence.}
Let $e_t$ collect all instantaneous evidence. Let
$\mathcal E_{\mathrm{safe}}$ denote the accepted set. Let $W_t$ denote the
frozen persistence window. The full decision is
\begin{equation}
 C_t=
 \ind\!\left[e_\tau\in\mathcal E_{\mathrm{safe}}
 \text{ for every }\tau\in W_t\right].
 \label{eq:supp_certificate}
\end{equation}

\begin{proposition}[Orientation-free consistency]
Absolute cosine consistency is invariant to an independent sign reversal of
any probe direction.
\end{proposition}

\begin{proof}
For signs $s_i,s_j\in\{-1,1\}$,
\begin{equation}
 \left|\cos(s_i d_i,s_j d_j)\right|
 =\left|s_is_j\cos(d_i,d_j)\right|
 =\left|\cos(d_i,d_j)\right|.
\end{equation}
\end{proof}

\begin{proposition}[Conservative decision rule]
Path materialization occurs only when every frozen component passes over the
full persistence window. A failed certificate implies no structural action.
It does not imply that one-path capacity is sufficient.
\end{proposition}

\begin{proof}
The first statement follows from Eq.~\eqref{eq:supp_certificate}. The second
statement follows because $C_t=0$ can result from any failed evidence term.
A false negative is therefore possible. The controlled rank-four case is an
observed example.
\end{proof}

All probe samples come from the training population. Validation and test data
never enter the structural decision.

\begin{table*}[!t]
\centering

\small
\begin{tabular}{p{0.17\textwidth}p{0.34\textwidth}p{0.38\textwidth}}
\toprule
Evidence group & Failure mode screened by the test & What a pass does not establish \\
\midrule
Direction strength & Numerical noise or an effectively zero antisymmetric force & Population agreement, descent, or a need for extra capacity \\
Cross-probe consistency & A direction that changes across halves or deterministic shards & A negative finite-step loss change or stable execution \\
Finite-step descent & A large and repeatable direction that is not locally useful under the deployed forward map & Long-horizon convergence after release \\
Dynamic safety & Non-finite values, unstable calibration, or broken exchange symmetry & Task specialization or compute savings \\
Temporal persistence & A one-epoch event caused by sampling variation or a transient optimization state & Universal detection of every weak bottleneck \\
\bottomrule
\end{tabular}
\caption{Interpretation of the five OEC evidence groups. Passing one group is
not sufficient; the certificate is the conjunction over all groups and the
persistence window.}
\label{tab:supp_oec_interpretation}
\end{table*}

Table~\ref{tab:supp_oec_interpretation} clarifies why OEC is intentionally
conservative. Direction strength answers whether a nontrivial local force is
present. Consistency asks whether the axis is population-level rather than
subset-specific. The two descent tests evaluate the actual finite perturbation
under frozen and recomputed calibration. Safety protects the structural
transition itself. Persistence turns the instantaneous conjunction into a
training-time decision. No single statistic is interpreted as a capacity
certificate in isolation.

\subsubsection{End-to-end decision sequence}
The implemented transition follows a fixed order.
\begin{enumerate}
 \item Train the exact-shared model with one physical path.
 \item At a scheduled phase-probe epoch, evaluate phase curvature and the
 frozen phase-safety checks on the training population.
 \item Release the latent relative phase only after the phase persistence rule
 passes. Keep the physical path parameters tied.
 \item Run the fixed phase-adaptation window. This allows the odd interaction
 to expose an antisymmetric force.
 \item Form deterministic population halves and shards.
 Construct one zero-sum direction from each partition.
 \item Evaluate strength, absolute-cosine agreement, frozen and dynamic
 finite-step descent, and all safety residuals.
 \item Materialize two paths only after the full OEC vector passes for the
 frozen number of consecutive probe epochs.
 \item Copy the mother parameters and optimizer state symmetrically. Continue
 ordinary first-order multi-task optimization with no diversity loss.
\end{enumerate}
This sequence distinguishes phase release from path release. A run can pass the
first decision and fail the second. It can also fail the phase decision and
never execute a path-force probe. Both outcomes are represented explicitly in
the per-seed evidence.

\subsection{Zero-Sum Materialization and Local Continuity}
\label{sec:supp_materialization}

A successful path certificate creates
\begin{equation}
 \omega_1=\bar\omega+\epsilon_d d_{\mathrm{anti}},
 \qquad
 \omega_2=\bar\omega-\epsilon_d d_{\mathrm{anti}}.
 \label{eq:supp_split}
\end{equation}

\begin{proposition}[Centroid preservation]
The zero-sum split preserves the mother parameter exactly:
\begin{equation}
 \tfrac12(\omega_1+\omega_2)=\bar\omega.
\end{equation}
\end{proposition}

\begin{proof}
Substitute Eq.~\eqref{eq:supp_split}. The perturbations cancel.
\end{proof}

\begin{proposition}[First-order cancellation for symmetric fusion]
Let $F(\omega_1,\omega_2)$ be twice continuously differentiable and exchange symmetric. Then
\begin{equation}
 F(\bar\omega+\epsilon d,\bar\omega-\epsilon d)
 =F(\bar\omega,\bar\omega)+O(\epsilon^2).
 \label{eq:supp_first_order_cancel}
\end{equation}
\end{proposition}

\begin{proof}
Exchange symmetry gives equal partial derivatives at the tied point:
$\nabla_1F=\nabla_2F$. The first-order Taylor term is
\begin{equation}
 \epsilon\nabla_1F^\top d-
 \epsilon\nabla_2F^\top d=0.
\end{equation}
The remaining change is second order.
\end{proof}

The proposition does not claim exact function preservation for an arbitrary
deep nonlinear network. It establishes local first-order continuity under the
symmetric interface. The optimizer state is copied symmetrically. This keeps
the accumulated history aligned at release time.

\subsection{Deep-Network Local Bridge}
\label{sec:supp_bridge}

The local theory is stated in representation coordinates. The implemented
model uses deep parameters. Let $h(\omega)$ denote a local feature tensor. Let
$J=\partial h/\partial\omega$ be its Jacobian at the probe checkpoint. Let
$g_h$ be a phase-exposed feature-space direction.

\begin{proposition}[Local parameter pullback]
The corresponding parameter gradient is
\begin{equation}
 g_\omega=J^\top g_h.
\end{equation}
If $J^\top g_h\neq0$, the feature-space candidate induces a nonzero local
parameter direction.
\end{proposition}

\begin{proof}
The result is the chain rule.
\end{proof}

A complementary statement starts from a desired feature perturbation. Let
$v_h$ lie in the image of $J$. Then there exists a parameter perturbation
$v_\omega$ with $Jv_\omega=v_h$; the minimum-norm choice is
$J^\dagger v_h$. For a differentiable task objective $\ell(h)$,
\begin{equation}
 \left.\frac{d}{d\eta}\ell\bigl(h(\bar\omega+\eta v_\omega)\bigr)
 \right|_{\eta=0}
 =\nabla_h\ell(h)^\top v_h.
\end{equation}
Thus, a descending feature-space direction that is locally realizable by the
network has a parameter-space representative with the same first-order loss
change. This statement is useful for interpreting the frozen-feature and
local-head diagnostics. It remains checkpoint-local and makes no global
optimization claim.

This bridge is local. It applies at the probed checkpoint. It does not prove
that every initialization reaches a separating branch. It does not prove that
the post-release trajectories remain complementary. These questions require
empirical evidence.

\subsection{Numerical Local Checks}
\label{sec:supp_numerical_theory}

A deterministic float64 implementation checks the local linear algebra. It is
separate from deep training. Table~\ref{tab:supp_numerical_checks} reports the
frozen diagnostics from the original theory package.

\begin{table}[!t]
\centering

\small
\setlength{\tabcolsep}{3pt}
\begin{tabular}{p{0.69\linewidth}r}
\toprule
Diagnostic & Value \\
\midrule
Minimum eigenvalue of full coupled system & \textbf{$-1.1231$} \\
Minimum eigenvalue of restricted path block & $2.0000$ \\
Minimum eigenvalue of restricted phase block & $2.0000$ \\
Path energy in minimum mode & $0.3787$ \\
Phase energy in minimum mode & $0.6213$ \\
Analytic--autograd Hessian discrepancy & \textbf{$8.88\times10^{-16}$} \\
Analytic--finite-difference discrepancy & $4.03\times10^{-7}$ \\
HVP and Lanczos residual & \textbf{$1.11\times10^{-16}$} \\
Safe relative split norms & $0.003$--$0.0001$ \\
\bottomrule
\end{tabular}
\caption{Numerical checks for the local coupled model. Positive restricted
blocks and a negative coupled mode illustrate the Schur-complement effect.}
\label{tab:supp_numerical_checks}
\end{table}

The first three rows of Table~\ref{tab:supp_numerical_checks} instantiate the
coupled-curvature argument. Each restricted block has minimum eigenvalue
$2.0$, yet the complete system has a negative eigenvalue of $-1.1231$. The
minimum mode contains both path and phase energy. It is therefore a genuinely
coupled mode rather than an instability confined to one coordinate. The
analytic, autograd, finite-difference, HVP, and Lanczos agreements provide
independent checks of the same local calculation. The split-norm range verifies
that the zero-sum perturbations used in the local continuity tests remain small
relative to the mother parameter norm.

\subsection{Theory Claim Boundary}
\label{sec:supp_theory_boundary}

The theory supports a local sufficiency statement. Relative phase can expose
an antisymmetric force before physical separation. A persistent, descending,
and safe force can justify a local zero-sum release. Three levels of statement
must remain distinct.

\begin{table*}[!t]
\centering

\small
\begin{tabular}{p{0.18\textwidth}p{0.71\textwidth}}
\toprule
Status & Statement \\
\midrule
\textbf{Local result} & Exchange invariance, phase-first antisymmetric exposure,
Schur-complement coupling, finite-step descent for a sufficiently small step,
and first-order cancellation under a symmetric zero-sum split. \\
\textbf{Empirical question} & Whether the training trajectory produces persistent OEC
evidence, whether released paths remain complementary, and whether delayed
growth improves the measured performance--computation trade-off. \\
\textbf{Outside the theorem} & Necessity of phase, global convergence from every
initialization, universal bottleneck detection, task-wise expert formation,
and superiority to every oracle-timed or randomly directed split. \\
\bottomrule
\end{tabular}
\caption{Scope of the theoretical statements. The table separates proved
local properties from empirically evaluated consequences and claims that are
outside the current result.}
\label{tab:supp_theory_scope}
\end{table*}

Table~\ref{tab:supp_theory_scope} makes the claim hierarchy explicit. The
local result is deliberately narrower than the empirical claim. It identifies a mechanism by which a tied dense model can expose a meaningful
separation direction. OEC then asks whether this mechanism is present and
persistent in a concrete training run. The experiments test the consequences
without converting them into universal theorems.

\section{Extended Experimental Results}
\label{sec:supp_experiments}

\subsection{Reporting Protocol and Evidence Organization}
\label{sec:supp_reporting}

\begin{table*}[!t]
\centering

\small
\begin{tabular}{p{0.16\textwidth}p{0.13\textwidth}p{0.10\textwidth}p{0.10\textwidth}p{0.35\textwidth}}
\toprule
Evaluation & Tasks or demand & Seeds & Horizon & Purpose \\
\midrule
Rank-controlled & joint ranks 3, 4, 6 & 3 & 50 epochs & Exact sharing, conservative boundary, and verified rank-six bottleneck. \\
NYUv2 standard & segmentation, depth, normals & 3 & fixed protocol & External multi-task performance and seed dependence. \\
NYUv2 Capacity-Stress & same three tasks, 8-channel path & 3 & 80 epochs & Population certificate, release events, and compute accounting. \\
NYUv2 task arrival & later demand at epoch 21 & 3 & 40 epochs & Delayed materialization and matched-schedule comparison. \\
PASCAL-Context standard & semantic, parts, saliency, normals & 3 & fixed protocol & Cross-dataset boundary. \\
PASCAL arrival & later tasks at epoch 21 & 3 & 40 epochs & Timing response on a second dataset. \\
PASCAL early arrival & later tasks at epoch 11 & 1 & 40 epochs & Timing sensitivity. \\
PASCAL ramp & demand ramp from epochs 11--15 & 3 & fixed protocol & Gradual-demand boundary and conflict-trigger comparison. \\
\bottomrule
\end{tabular}
\caption{Detailed evaluation map. Wall-clock values are descriptive and
hardware-specific.}
\label{tab:supp_protocol_map}
\end{table*}

For natural-image tasks, the directional aggregate is
\begin{equation}
 D_t(M)=s_t\frac{M_t-M_t^{\mathrm{HS}}}
 {|M_t^{\mathrm{HS}}|+\varepsilon_t},
 \qquad
 G(M)=\frac1T\sum_{t=1}^T D_t(M).
 \label{eq:supp_G}
\end{equation}
The sign $s_t$ is positive for higher-is-better metrics and negative for
lower-is-better metrics. All native task metrics remain visible.

Table~\ref{tab:supp_protocol_map} organizes the evidence by the question each
protocol can answer. The controlled benchmark isolates capacity because the
joint target rank is known. Standard-capacity natural-image experiments test
external task behavior but do not guarantee a capacity bottleneck. The
Capacity-Stress protocol deliberately narrows the path and records the full
population certificate. Task-arrival protocols separate the period before and
after additional demand. PASCAL-Context provides a dataset and task-composition
change. The early-arrival and ramp protocols test timing sensitivity and the
shape of the demand transition. This ordering prevents a result from one
protocol from being used to answer a question that only another protocol
identifies.

For $G$, Hard-sharing is the per-seed reference and therefore has value zero
by definition. A positive $G$ denotes a directional improvement averaged over
native metrics. A negative $G$ denotes a directional decrease. The aggregate
is not a substitute for the task-level columns. We report both because a
single average can hide compensating improvements and degradations.

\subsection{Implementation and Decision Settings}
\label{sec:supp_impl_settings}
All values below are transcribed from archived configuration or implementation files. Missing fields are not inferred.
Table~\ref{tab:supp_decision_settings} separates the standard OEC contract from the population settings used only by Capacity-Stress. Table~\ref{tab:supp_baseline_matching} records which baselines begin shared, remain shared, or expose two trainable paths.

\begin{table*}[!t]
\centering
\small
\textbf{(a) Controlled and NYUv2 data/model settings}\par\smallskip
\setlength{\tabcolsep}{3pt}
\begin{tabular}{L{0.12\textwidth}L{0.20\textwidth}L{0.28\textwidth}L{0.30\textwidth}}
\toprule
Setting & Controlled rank & NYUv2 external static & NYUv2 OEC/task arrival \\
\midrule
Data/split & seeded synthetic factors & 795 train, 654 held-out; SHA256-locked & same frozen train/held-out split \\
Input & six latent factors & RGB, $288\times384$ & RGB, $288\times384$ \\
Preprocess & none & MTAN tensors, ImageNet normalization & same; train flip $p=0.5$ \\
Backbone & linear rank-controlled path & ImageNet-pretrained dilated ResNet-18 & SmallNYUv2Encoder, width 32 \\
Path & rank capacity 3 & matched equal-capacity adapters & width 64; Conv3x3-BN-ReLU-Conv1x1 \\
Heads & three linear heads & homogeneous ASPP per task & three task-specific 1x1 decoders \\
Loss & mean squared error & masked CE, masked L1, normal loss; equal weights & segmentation, depth, normal; frozen arrival profile \\
\bottomrule
\end{tabular}

\vspace{5pt}
\textbf{(b) PASCAL data/model settings}\par\smallskip
\begin{tabular}{L{0.17\textwidth}L{0.36\textwidth}L{0.36\textwidth}}
\toprule
Setting & Native static/arrival & External static \\
\midrule
Data/split & PASCAL\_MT; first 512 frozen train/validation IDs & same frozen 512/512 subsets \\
Input/preprocess & RGB, resize $128\times128$, ImageNet normalization & same \\
Backbone & SmallEncoder width 32, three convolutional stages & ImageNet-pretrained dilated ResNet-18 \\
Path & width 64; Conv3x3-BN-ReLU-Conv1x1 & matched equal-capacity adapters \\
Heads & task-specific 1x1 heads & homogeneous ASPP per task \\
Loss & equal-weight semantic, parts, saliency, normal losses & same four equal-weight tasks \\
\bottomrule
\end{tabular}
\caption{Data and model settings. Panel (a) records controlled and NYUv2 protocols; panel (b) records the two PASCAL implementations.}
\label{tab:supp_impl_settings}
\end{table*}

\begin{table*}[!t]
\centering
\small
\renewcommand{\arraystretch}{0.92}
\textbf{(a) Optimization settings}\par\smallskip
\setlength{\tabcolsep}{3pt}
\begin{tabular}{L{0.14\textwidth}L{0.20\textwidth}L{0.27\textwidth}L{0.29\textwidth}}
\toprule
Setting & Controlled & NYUv2 external static & NYUv2 OEC / arrival / PASCAL \\
\midrule
Batch & archived runner setting & 8 & 8 \\
Optimizer & archived runner setting & Adam, lr $10^{-4}$, wd 0 & AdamW, lr $10^{-3}$, wd $10^{-4}$ \\
Scheduler & archived runner setting & StepLR(100, 0.5) &
Capacity-Stress: CosineAnnealingLR, $T_{\max}=80$;
arrival/PASCAL: $T_{\max}=40$ \\
Horizon & 50 epochs & 200 epochs &
Capacity-Stress: 80 epochs; arrival/PASCAL: 40 epochs \\
Evaluation & retained final summaries & held-out final evaluation & per-epoch history; final checkpoint \\
Seeds & 2027--2029 for High/Low & 2027--2029 & 2027--2029 where archived \\
\bottomrule
\end{tabular}

\vspace{5pt}
\textbf{(b) Hardware and reproducibility settings}\par\smallskip
\begin{tabular}{L{0.20\textwidth}L{0.69\textwidth}}
\toprule
Setting & Recorded value \\
\midrule
Hardware & one NVIDIA GeForce RTX 4070 SUPER GPU, 12 GiB, for natural-image runs \\
Software & PyTorch 2.6.0+cu124, CUDA 12.4, cuDNN 9.1 \\
Precision & controlled CPU runs without AMP; external adapters use AMP with pair branch in float32 \\
Determinism & seeded order, fixed probe batches, deterministic algorithms where registered \\
Loader workers & 2 for NYUv2; 0 for PASCAL \\
\bottomrule
\end{tabular}
\caption{Optimization and reproducibility settings. Panel (a) records training contracts; panel (b) records hardware and determinism.}
\label{tab:supp_impl_optimization}
\end{table*}

\begin{table*}[!t]
\centering
\small
\renewcommand{\arraystretch}{0.92}
\textbf{(a) General frozen decision settings}\par\smallskip
\begin{tabular}{L{0.31\textwidth}L{0.24\textwidth}L{0.27\textwidth}}
\toprule
Setting & Value & Status \\
\midrule
$K$ / aggregate epsilon & 2 / $10^{-12}$ & frozen \\
Split relative norm & 0.003 & frozen \\
Zero-sum materialization & true & frozen \\
Probe start / interval / batches & 8 / 4 / 8 & frozen \\
Phase curvature / persistence & $-0.05$ / 2 & frozen \\
Materialized optimizer state & copied to both paths & implementation contract \\
\bottomrule
\end{tabular}

\vspace{5pt}
\textbf{(b) NYUv2 OEC-specific settings}\par\smallskip
\begin{tabular}{L{0.31\textwidth}L{0.24\textwidth}L{0.27\textwidth}}
\toprule
Setting & Value & Status \\
\midrule
Initial phase / phase angle & symmetric zero / 0.10 rad & contract / frozen \\
$\beta,\kappa,\rho$ & $-1,+1,0.5$ & implementation contract \\
NormCal / parameter RMS floors & $10^{-8}$ / $10^{-3}$ & contract / frozen \\
Transverse gradient maximum & $10^{-3}$ & frozen \\
Finite-difference angle / ratio & $10^{-3}$ rad / $[0.5,1.5]$ & frozen \\
Adaptation window & 4 epochs & frozen \\
Cross-half / median-batch cosine & $\geq0.90$ / $\geq0.50$ & frozen \\
Path-strength floor & $\max(10^{-3},10e_{\rm repeat})$ & deterministic rule \\
Finite path step & 0.003 RMS-normalized & frozen \\
Frozen/dynamic descent & both loss deltas negative & implementation contract \\
Dynamic loss/swap/output floors & $10^{-4}/10^{-6}/10^{-5}$ & frozen \\
Dynamic hard cap / persistence & $5\times$ threshold / 2 & frozen \\
\bottomrule
\end{tabular}
\caption{Frozen EMAN decision settings. Panel (a) contains general contracts; panel (b) contains NYUv2 OEC thresholds. Capacity-Stress population settings are recorded separately in Table~\ref{tab:supp_probe_config}.}
\label{tab:supp_decision_settings}
\end{table*}

\begin{table*}[!t]
\centering
\small
\renewcommand{\arraystretch}{0.92}
\setlength{\tabcolsep}{3pt}
\begin{tabular}{L{0.13\textwidth}L{0.09\textwidth}L{0.15\textwidth}L{0.25\textwidth}L{0.14\textwidth}L{0.13\textwidth}}
\toprule
Method & Paths & Initial state & Trigger or policy & Release & Task schedule \\
\midrule
Hard-sharing & 1 & shared & none & -- & no \\
PhaseOnly & 1 & shared & phase certificate & phase event & no \\
MeanPaths & 2 & dual & none & epoch 1 & no \\
Scheduled split & 2 & shared & preregistered epoch & 11 or 21 & yes \\
Conflict Trigger & 2 & shared & persistent gradient conflict & event or -- & no \\
AdaShare & policy units & shared network & learned sharing policy & n/a & no \\
Recon & shared & shared & conflict-aware update & n/a & no \\
EMAN & 1 then 2 & exact shared & phase certificate then OEC & certified event & no \\
\bottomrule
\end{tabular}
\caption{Baseline structural states and release rules.}
\label{tab:supp_baseline_matching}
\end{table*}

\begin{table*}[!t]
\centering
\small
\renewcommand{\arraystretch}{0.88}
\textbf{(a) Parameter and protocol matching}\par\smallskip
\begin{tabular}{L{0.17\textwidth}L{0.27\textwidth}L{0.44\textwidth}}
\toprule
Method & Capacity match & Shared protocol \\
\midrule
Hard-sharing, PhaseOnly & one path & same backbone, heads, split, losses, optimizer, scheduler, and horizon \\
MeanPaths & released EMAN maximum & same protocol and per-path width \\
Scheduled, Conflict split & released EMAN maximum & same protocol, width, and zero-sum split norm \\
AdaShare, Recon & common backbone and heads & method-specific policy/update parameters reported separately \\
EMAN & one tied, then at most two & same protocol; only active path count changes after certification \\
\bottomrule
\end{tabular}

\vspace{5pt}
\textbf{(b) Compute-accounting contract}\par\smallskip
\begin{tabular}{L{0.17\textwidth}L{0.27\textwidth}L{0.44\textwidth}}
\toprule
Method & Path evaluations & Total accounting \\
\midrule
Hard-sharing, PhaseOnly & one per batch & backbone, active path, heads, fusion \\
MeanPaths & two per batch & backbone, two paths, heads, fusion \\
Scheduled, Conflict split & one before, two after split & same components plus registered decision overhead \\
AdaShare, Recon & archived active FLOPs & method-specific policy or conflict-update cost included \\
EMAN & one before, two after materialization & same components plus phase/OEC probe cost \\
\bottomrule
\end{tabular}
\caption{Matching and compute contracts. Panel (a) records capacity/protocol matching; panel (b) defines active-path and total-cost accounting.}
\label{tab:supp_baseline_compute_matching}
\end{table*}

\subsection{Rank-Controlled Results for Every Seed}
\label{sec:supp_controlled}

Each physical path has rank-three capacity. The \textsc{High} demand has
joint rank three. The \textsc{Low} demand has joint rank six. Table
\ref{tab:supp_controlled_all} reports every task and every seed. EMAN exactly
matches Hard-sharing in every \textsc{High} run. It performs no release. In
\textsc{Low}, PhaseOnly releases phase in every run but stays at the
Hard-sharing level. EMAN materializes in every run and recovers a much larger
macro-$R^2$.

\begin{table*}[!t]
\centering

\small
\renewcommand{\arraystretch}{0.90}
\setlength{\tabcolsep}{3.6pt}
\begin{tabular}{ll l cccc c}
\toprule
Seed & Regime & Method & Task 1 & Task 2 & Task 3 & Macro $R^2$ & Phase/Path \\
\midrule
Seed 1 & \textsc{High} & Hard-sharing & 0.9999 & 0.9999 & 0.9999 & \textbf{0.9999} & no/no \\
 & \textsc{High} & EMAN & 0.9999 & 0.9999 & 0.9999 & \textbf{0.9999} & no/no \\
 & \textsc{Low} & Hard-sharing & 0.7111 & 0.1705 & 0.6342 & 0.5053 & no/no \\
 & \textsc{Low} & PhaseOnly & 0.7037 & 0.1768 & 0.6346 & 0.5050 & yes/no \\
 & \textsc{Low} & EMAN & 0.9999 & 0.9999 & 0.9999 & \textbf{0.9999} & \textbf{yes/yes} \\
\midrule
Seed 2 & \textsc{High} & Hard-sharing & 0.8540 & 0.8262 & 0.8562 & \textbf{0.8455} & no/no \\
 & \textsc{High} & EMAN & 0.8540 & 0.8262 & 0.8562 & \textbf{0.8455} & no/no \\
 & \textsc{Low} & Hard-sharing & 0.4372 & 0.4481 & 0.5064 & 0.4639 & no/no \\
 & \textsc{Low} & PhaseOnly & 0.4396 & 0.4455 & 0.5055 & 0.4635 & yes/no \\
 & \textsc{Low} & EMAN & 0.8540 & 0.6999 & 0.7014 & \textbf{0.7518} & \textbf{yes/yes} \\
\midrule
Seed 3 & \textsc{High} & Hard-sharing & 0.8325 & 0.8414 & 0.9089 & \textbf{0.8609} & no/no \\
 & \textsc{High} & EMAN & 0.8325 & 0.8414 & 0.9089 & \textbf{0.8609} & no/no \\
 & \textsc{Low} & Hard-sharing & 0.4111 & 0.4624 & 0.5732 & 0.4823 & no/no \\
 & \textsc{Low} & PhaseOnly & 0.4108 & 0.4640 & 0.5700 & 0.4816 & yes/no \\
 & \textsc{Low} & EMAN & 0.8323 & 0.8259 & 0.8816 & \textbf{0.8466} & \textbf{yes/yes} \\
\bottomrule
\end{tabular}
\caption{Complete controlled results. Phase/Path reports event occurrence in
each run.}
\label{tab:supp_controlled_all}
\end{table*}

\begin{figure*}[!t]
\centering
\includegraphics[width=0.49\textwidth]{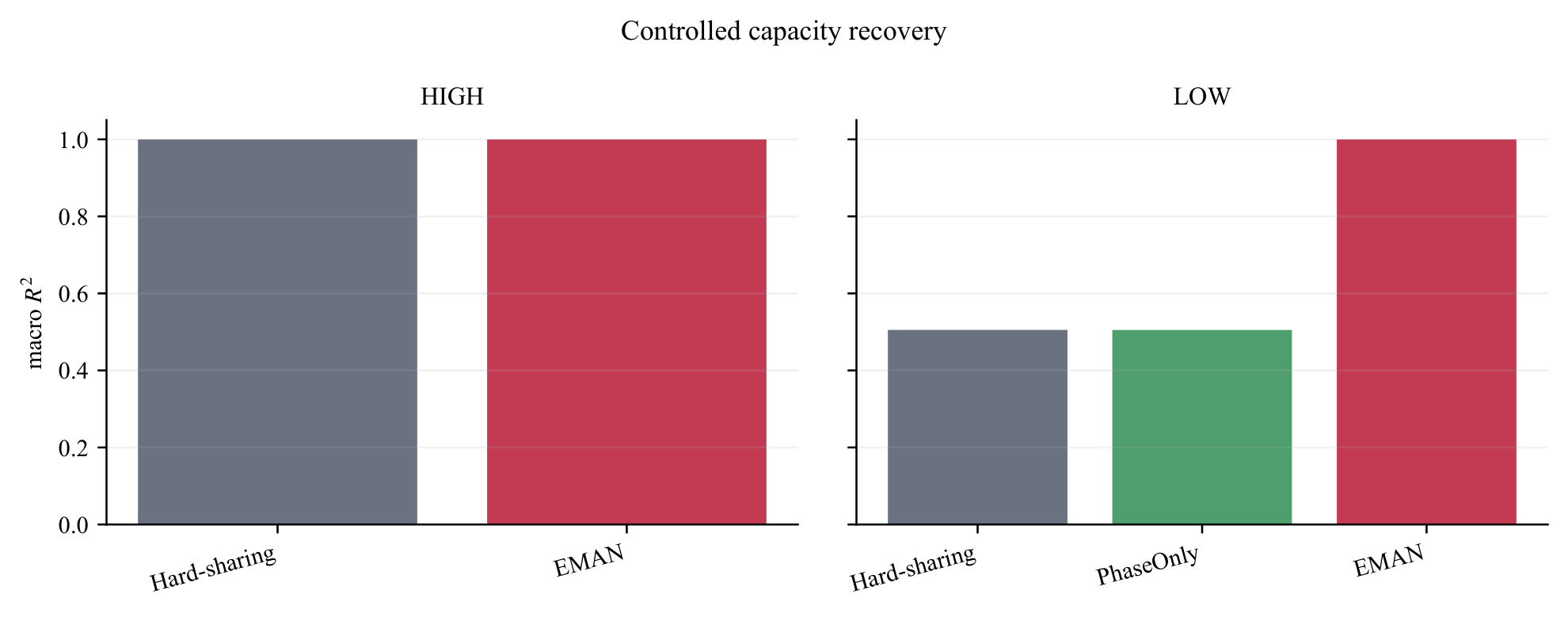}
\hfill
\includegraphics[width=0.49\textwidth]{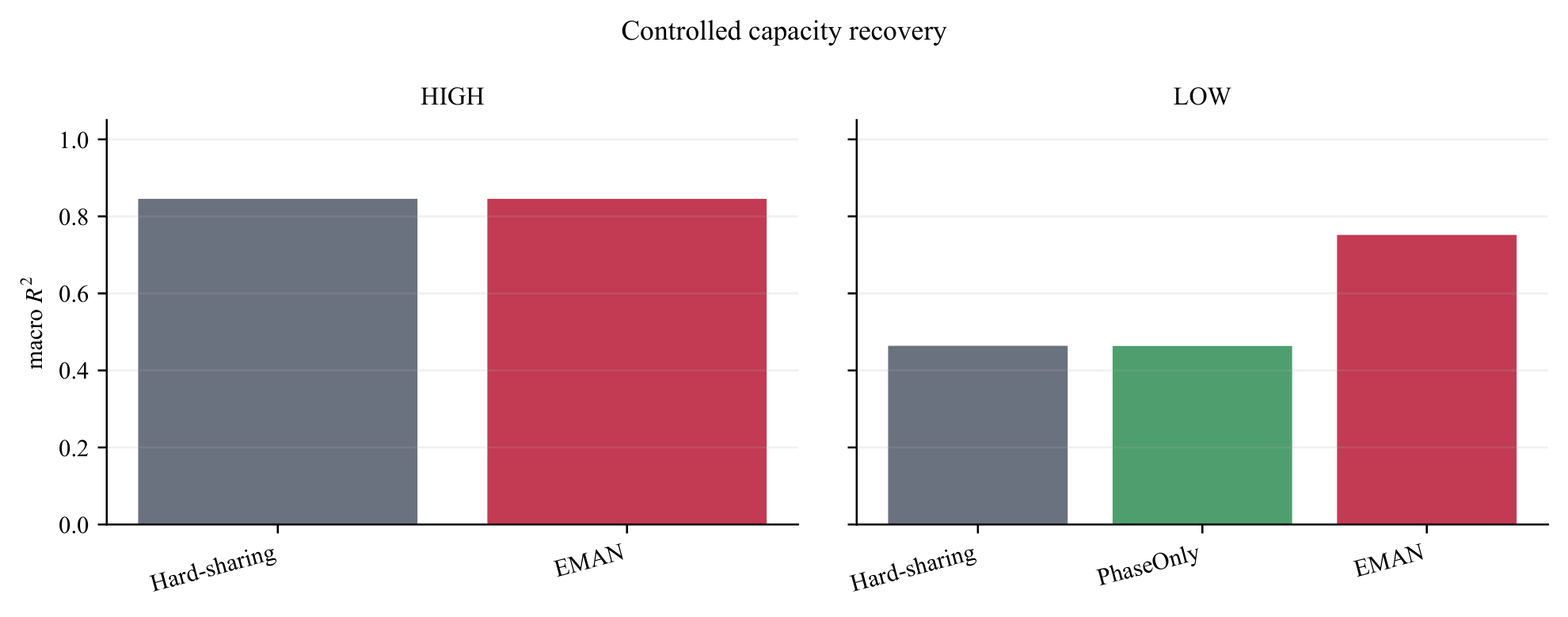}
\par
\includegraphics[width=0.49\textwidth]{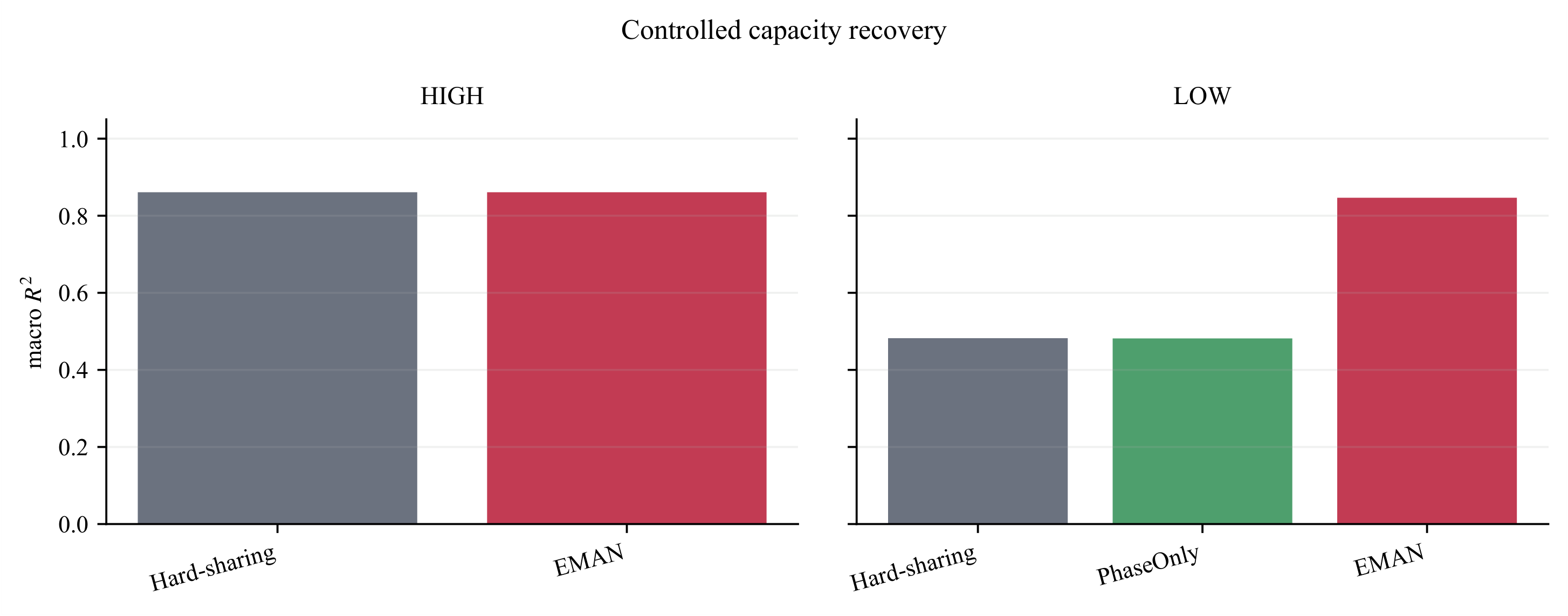}
\caption{Per-seed controlled results. The three plots expose the variance
behind the aggregate main-paper table. The structural decision remains stable:
all \textsc{High} runs stay shared, and all \textsc{Low} EMAN runs release.}
\label{fig:supp_controlled_seeds}
\end{figure*}

Table~\ref{tab:supp_controlled_all} and
Figure~\ref{fig:supp_controlled_seeds} show two types of stability. First, the
structural decision is stable across all three seeds: every \textsc{High} run
remains exact-shared, and every \textsc{Low} EMAN run releases phase and path.
Second, the amount of predictive recovery varies with the seed. Seed~1 is a
near-ceiling realization. Seeds~2 and~3 are harder, but EMAN still raises the
\textsc{Low} macro-$R^2$ from $0.4639$ to $0.7518$ and from $0.4823$ to
$0.8466$, respectively. This distinction matters. The certificate is evaluated
as a structural decision, while final predictive performance remains subject
to ordinary optimization variability after release.

The PhaseOnly rows isolate the role of latent phase. In each \textsc{Low} run,
phase is released but the path is not materialized. The macro-$R^2$ remains
within $0.0008$ of Hard-sharing. The result rules out the interpretation that
phase freedom by itself supplies the missing rank-six representation. In
contrast, EMAN changes the physical path count and recovers all three tasks
simultaneously. The gains are therefore associated with independent path
capacity rather than a task-specific loss trade-off.

The \textsc{Medium} demand has joint rank four. It exceeds one-path rank by
one. EMAN stays shared and obtains $0.7086$ macro-$R^2$. Hard-sharing obtains
$0.7086$, and PhaseOnly obtains $0.7081$. This setting lies close to the
single-path boundary. The frozen evidence rule does not certify a persistent
benefit from structural growth. We therefore retain it as a conservative false
negative rather than retuning the threshold after observing the result. The
case identifies the sensitivity limit of the current certificate. It does not
affect the stronger rank-six capacity result, where phase and path release
occur in every seed.

\subsection{Archived Structural Events and Diagnostic Boundaries}
\label{sec:supp_structural_events}

The Medium joint-rank-four result is retained only as the frozen aggregate summary: Hard-sharing $0.7086$, PhaseOnly $0.7081$, and EMAN $0.7086$ macro-$R^2$. The archive does not contain a matching three-seed table for these numbers. Per-seed values are therefore not reconstructed. A separate single-run directory named \texttt{rank\_controlled\_core/MEDIUM} contains a different final value and is not merged with the frozen summary.

\begin{table}[!t]
\centering

\small
\renewcommand{\arraystretch}{0.90}
\begin{tabular}{lcc}
\toprule
Method & Macro $R^2$ & Per-seed archive \\
\midrule
Hard-sharing & 0.7086 & not archived \\
PhaseOnly & 0.7081 & not archived \\
EMAN & 0.7086 & not archived \\
\bottomrule
\end{tabular}
\caption{Frozen Medium evidence boundary. ``Not archived'' is not replaced by a reconstructed value.}
\label{tab:supp_medium_boundary}
\end{table}

The standard-capacity NYUv2 external comparison uses the frozen static matched adapter. Its archived method states record phase release at epoch~1 and path release at epoch~2 for each seed. These runs do not contain OEC probe traces and therefore cannot support a claim about persistent phase evidence. They establish a matched static performance comparison only; the separate OEC and task-arrival protocols supply structural-decision evidence.

\begin{table}[!t]
\centering

\small
\renewcommand{\arraystretch}{0.90}
\begin{tabular}{lccc}
\toprule
Protocol & Seeds & Phase epoch & Path epoch \\
\midrule
NYUv2 static & 1--3 & 1 & 2 \\
PASCAL static & 1--3 & 1 & 2 \\
\bottomrule
\end{tabular}
\caption{Archived structural states for the static matched adapters.}
\label{tab:supp_static_events}
\end{table}

The static matched adapters reach the dual-path state at epochs 1 and 2.
These events define the static comparison. The task-arrival protocols
provide the corresponding evidence for release timing under changing demand. Static PASCAL uses a dual-path ratio of $39/40$ and a path-training ratio of $79/80$ in each archived run. Machine-readable per-seed records, including hardware-specific wall-clock values, are provided in \texttt{nyuv2\_standard\_structural\_events.csv} and \texttt{pascal\_static\_events.csv}.

\subsection{Coalition Diagnostics and Specialization Boundary}
\label{sec:supp_coalition}
The K=2 coalition audit reports path main effects, Shapley contributions, pair synergy, CSS, HPS, and task assignments from the archived controlled checkpoint. The retained LOW result has near-zero CSS and assigns all tasks to the same path. Pair synergy is reported separately and is not interpreted as path specialization. Consequently, the evidence supports complementary released subspaces, but not task-wise expert specialization. Full machine-readable values appear in \texttt{coalition\_diagnostics.csv}.

\subsection{Detailed Representation Geometry}
\label{sec:supp_representation}

The representation analysis uses the released \textsc{Low} model from
Seed~1. Linear held-out probes map Path~1, Path~2, and their concatenation to
the six latent factors. Table~\ref{tab:supp_factor_recovery} reports all
values. Neither individual path recovers every factor. The concatenation
recovers each factor with $R^2\ge0.999892$.

\begin{table*}[!t]
\centering

\small
\setlength{\tabcolsep}{4pt}
\begin{tabular}{lcccccc}
\toprule
Feature & $s_1$ & $s_2$ & $s_3$ & $s_4$ & $s_5$ & $s_6$ \\
\midrule
Path 1 & 0.240901 & 0.901442 & 0.662076 & 0.237050 & 0.123162 & 0.848102 \\
Path 2 & 0.155087 & 0.671282 & 0.600379 & 0.364944 & 0.484158 & 0.718674 \\
Concatenated & \textbf{0.999902} & \textbf{0.999901} & \textbf{0.999892} & \textbf{0.999896} & \textbf{0.999902} & \textbf{0.999902} \\
\bottomrule
\end{tabular}
\caption{Factor-wise held-out recovery for the released Seed~1 model.}
\label{tab:supp_factor_recovery}
\end{table*}

\begin{table}[!ht]
\centering

\small
\begin{tabular}{cc}
\toprule
Index & Canonical correlation \\
\midrule
1 & \textbf{1.000000} \\
2 & \textbf{1.000000} \\
3 & \textbf{1.000000} \\
4 & \textbf{1.000000} \\
5 & 0.954043 \\
6 & 0.940091 \\
\bottomrule
\end{tabular}
\caption{Canonical-correlation spectrum between the concatenated path
representation and the latent state.}
\label{tab:supp_cca}
\end{table}

\begin{table}[!ht]
\centering

\small
\begin{tabular}{lr}
\toprule
Metric & Value \\
\midrule
canonical rank & \textbf{6.000000} \\
effective rank concat & 3.891877 \\
effective rank path1 & 2.945753 \\
effective rank path2 & 2.958267 \\
mean canonical correlation & 0.982356 \\
mean concat factor R2 & \textbf{0.999899} \\
min concat factor R2 & \textbf{0.999892} \\
minimum canonical correlation & 0.940091 \\
normalized projector overlap & 0.210805 \\
permuted max factor R2 & 0.057603 \\
permuted mean R2 & 0.024310 \\
unique gain path1 & \textbf{0.500812} \\
unique gain path2 & \textbf{0.497777} \\
\bottomrule
\end{tabular}
\caption{Additional latent-geometry diagnostics for Seed~1. Effective rank
measures spectral uniformity. It is not algebraic rank.}
\label{tab:supp_geometry}
\end{table}

The three principal angles are $58.2795^\circ$, $59.3178^\circ$, and
$71.9900^\circ$. The normalized projector overlap is $0.210805$. The unique
conditional gains are $0.500812$ and $0.497777$. 

A permutation control gives
mean factor $R^2=0.024310$ and maximum factor $R^2=0.057603$. These results
support complementary capacity. They do not support one expert per task.

Table~\ref{tab:supp_factor_recovery} provides the most direct
factor-by-factor view. Path~1 is strongest on $s_2$, $s_3$, and $s_6$. Path~2 has relatively larger
recovery on $s_4$ and $s_5$. Neither path alone reaches full recovery. The
concatenated representation recovers every factor. Its minimum factor score is
$0.999892$. This pattern is incompatible with simple duplication. Identical
path subspaces would not recover the missing factors after concatenation.

Table~\ref{tab:supp_cca} gives a basis-invariant check. Four canonical
correlations are numerically one, and the remaining two are $0.954043$ and
$0.940091$. The complete representation therefore preserves all six latent
directions with high linear alignment. Table~\ref{tab:supp_geometry} adds
three controls. The principal angles and low projector overlap show that the
individual path subspaces are not identical. The two unique conditional gains
are nearly balanced. The permutation control collapses mean recovery from
$0.999899$ to $0.024310$, which rules out a probe or evaluation artifact.

\begin{table}[!t]
\centering

\small
\renewcommand{\arraystretch}{0.90}
\begin{tabular}{lcccc}
\toprule
Diagnostic & Task 1 & Task 2 & Task 3 & Aggregate \\
\midrule
Assigned path & 1 & 1 & 1 & -- \\
Pair synergy & 2.7079 & 3.6472 & 3.8095 & -- \\
CSS & -- & -- & -- & $-0.0000$ \\
HPS & -- & -- & -- & 1.2757 \\
Selectivity & -- & -- & -- & 0.1634 \\
\bottomrule
\end{tabular}
\caption{LOW controlled coalition diagnostics for EMAN at the archived seed-2027 final checkpoint. Pair synergy is task-wise interaction, not specialization.}
\label{tab:supp_coalition}
\end{table}

A separate coalition diagnostic gives near-zero task-wise specialization for
EMAN. Its HPS is $1.2757$, while task-wise pair synergies are $2.7079$,
$3.6472$, and $3.8095$. All three tasks receive the same path assignment
under the main-effect rule. The result rejects a task-expert interpretation.
It does not reject complementary latent information.

The representation conclusion is consequently narrower and more useful than
an expert-specialization claim. The two paths become complementary carriers of
latent information, but the task heads continue to use their fused
representation. No path is assigned to a predefined task, and the evidence
does not require such an assignment.

\subsection{NYUv2 Capacity-Stress: Complete Per-Seed Performance}
\label{sec:supp_capacity_stress}

Capacity-Stress combines an eight-channel path with the frozen
\texttt{NYUV2\_D11} demand profile: segmentation and depth are active during
epochs 1--10, and surface-normal demand is activated from epoch 11.
The training horizon is 80 epochs and the cosine scheduler uses
$T_{\max}=80$. Each
population probe uses all 795 training samples and their horizontal flips.
This yields 1,590 states. The states are divided into eight deterministic
shards. Seed~1 and Seed~3 release. Seed~2 remains shared. Table
\ref{tab:supp_capacity_perf} reports every method and seed.

\begin{table}[!ht]
\centering

\small
\begin{tabular}{lc}
\toprule
Recorded probe setting & Value \\
\midrule
Training samples & 795 \\
Augmented states & 1,590 \\
Deterministic shards & 8 \\
Shard sizes & $3\times200+5\times198$ \\
Path-strength floor & \textbf{$10^{-3}$} \\
Phase persistence & \textbf{2 consecutive passes} \\
Path persistence & \textbf{2 consecutive passes} \\
Validation/test use & \textbf{none} \\
\bottomrule
\end{tabular}
\caption{Recorded population-certificate configuration. The table reports
archived settings rather than tuned post hoc values.}
\label{tab:supp_probe_config}
\end{table}

Table~\ref{tab:supp_probe_config} records the decision protocol rather than a
performance result. Every probe uses the full training population and its
registered flip augmentation. Deterministic shards keep the population
partition fixed across probe epochs. The $10^{-3}$ path-strength floor screens
numerically negligible forces. Separate two-pass persistence rules govern
phase and path release. Validation and test samples are absent from all
decisions. These settings were frozen before the reported per-seed outcomes.

\begin{table*}[!t]
\centering

\small
\setlength{\tabcolsep}{4pt}
\begin{tabular}{llccccc}
\toprule
Seed & Method & $G$ & mIoU & Depth Abs. & Normal Mean & Joint Loss \\
\midrule
Seed 1 & Hard-sharing & 0.0000 & 0.1216 & 0.8584 & 43.4562 & 2.9159 \\
 & PhaseOnly & 0.0163 & 0.1251 & 0.8487 & 43.0796 & 2.8871 \\
 & MeanPaths & 0.0262 & 0.1285 & 0.8576 & 42.5451 & 2.8768 \\
 & Conflict Trigger & 0.0000 & 0.1216 & 0.8584 & 43.4562 & 2.9159 \\
 & EMAN & \textbf{0.0361} & 0.1261 & 0.8230 & 42.1715 & 2.8877 \\
\midrule
Seed 2 & Hard-sharing & 0.0000 & 0.1268 & 0.8548 & 42.9600 & 2.8817 \\
 & PhaseOnly & 0.0066 & 0.1285 & 0.8495 & 42.9514 & 2.8873 \\
 & MeanPaths & \textbf{0.0307} & 0.1361 & 0.8522 & 42.2700 & 2.8433 \\
 & Conflict Trigger & 0.0000 & 0.1268 & 0.8548 & 42.9600 & 2.8817 \\
 & EMAN & 0.0267 & 0.1298 & 0.8349 & 41.5182 & 2.8483 \\
\midrule
Seed 3 & Hard-sharing & 0.0000 & 0.1248 & 0.8548 & 42.8611 & 2.8862 \\
 & PhaseOnly & 0.0116 & 0.1291 & 0.8505 & 43.0620 & 2.8765 \\
 & MeanPaths & 0.0140 & 0.1293 & 0.8536 & 42.6534 & 2.8650 \\
 & Conflict Trigger & -0.0000 & 0.1248 & 0.8548 & 42.8611 & 2.8862 \\
 & EMAN & \textbf{0.0291} & 0.1291 & 0.8294 & 41.8817 & 2.8692 \\
\bottomrule
\end{tabular}
\caption{Capacity-Stress performance for all seeds. Higher is better for
$G$ and mIoU. Lower is better for the remaining columns.}
\label{tab:supp_capacity_perf}
\end{table*}

\begin{figure*}[!t]
\centering
\includegraphics[width=0.47\textwidth]{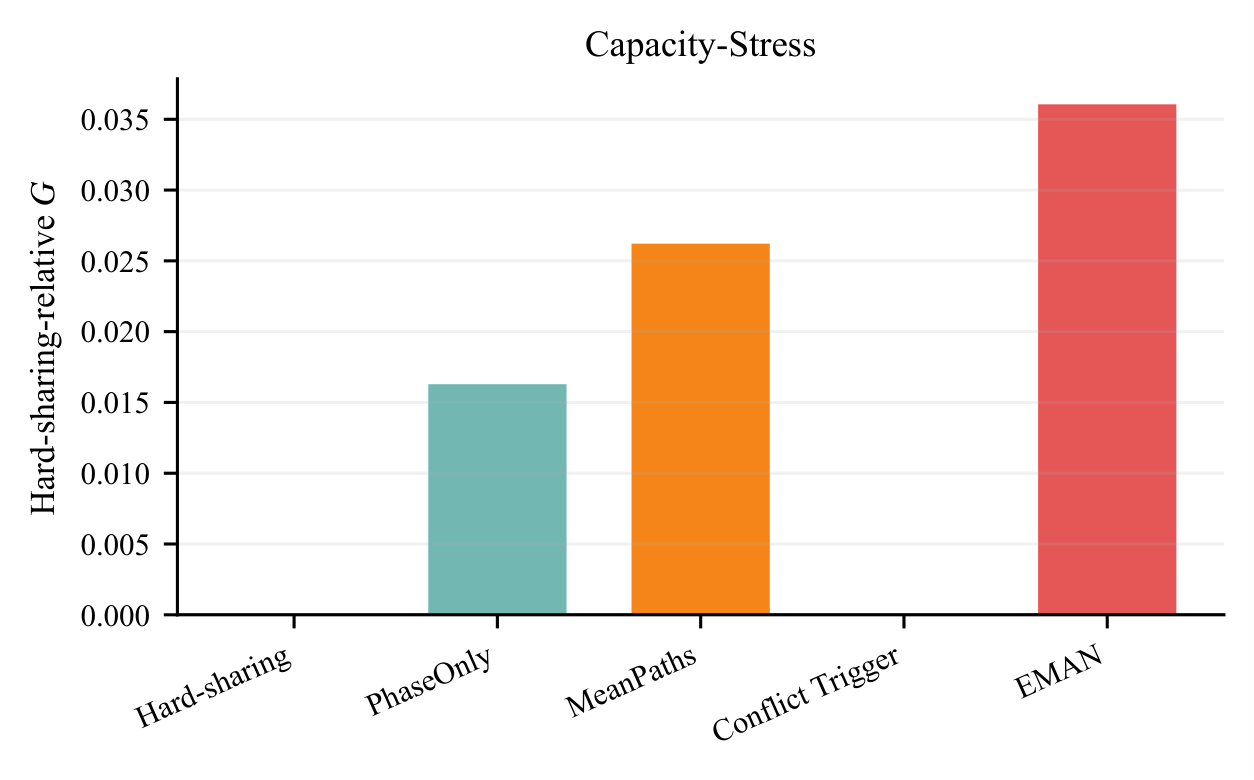}
\hfill
\includegraphics[width=0.47\textwidth]{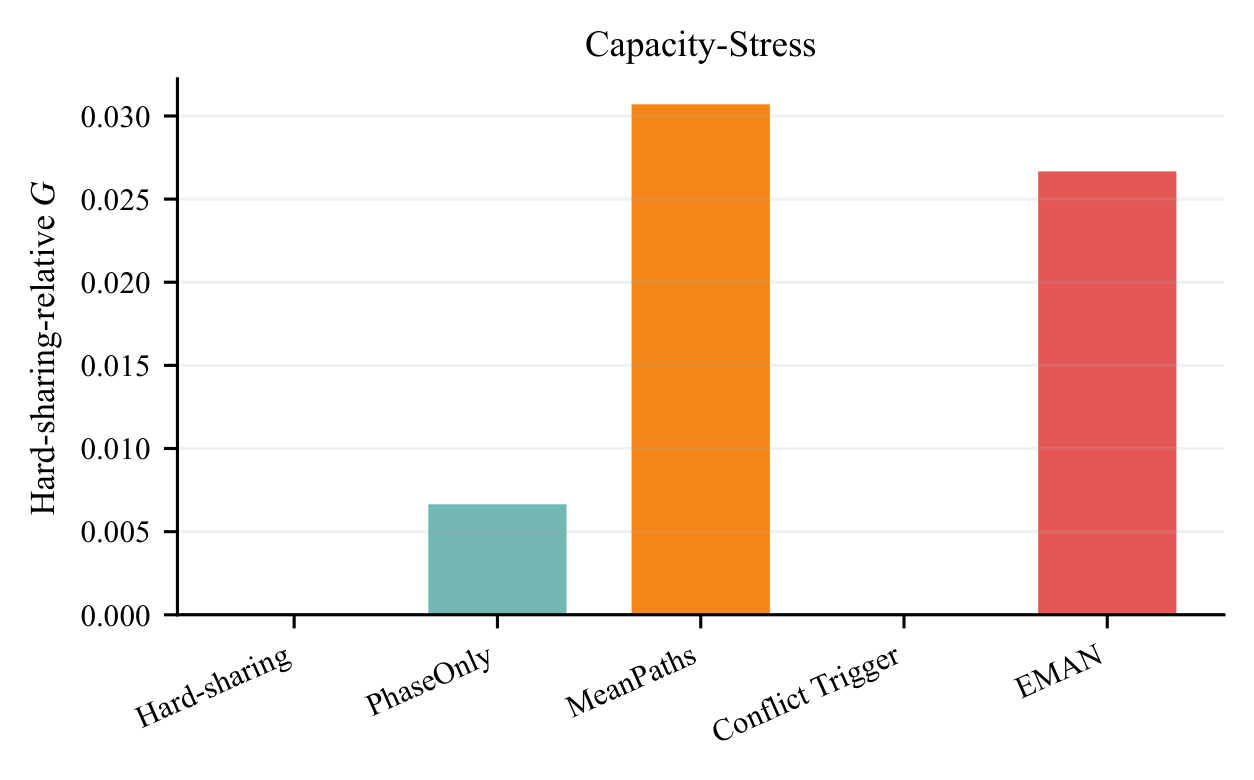}
\hfill
\includegraphics[width=0.47\textwidth]{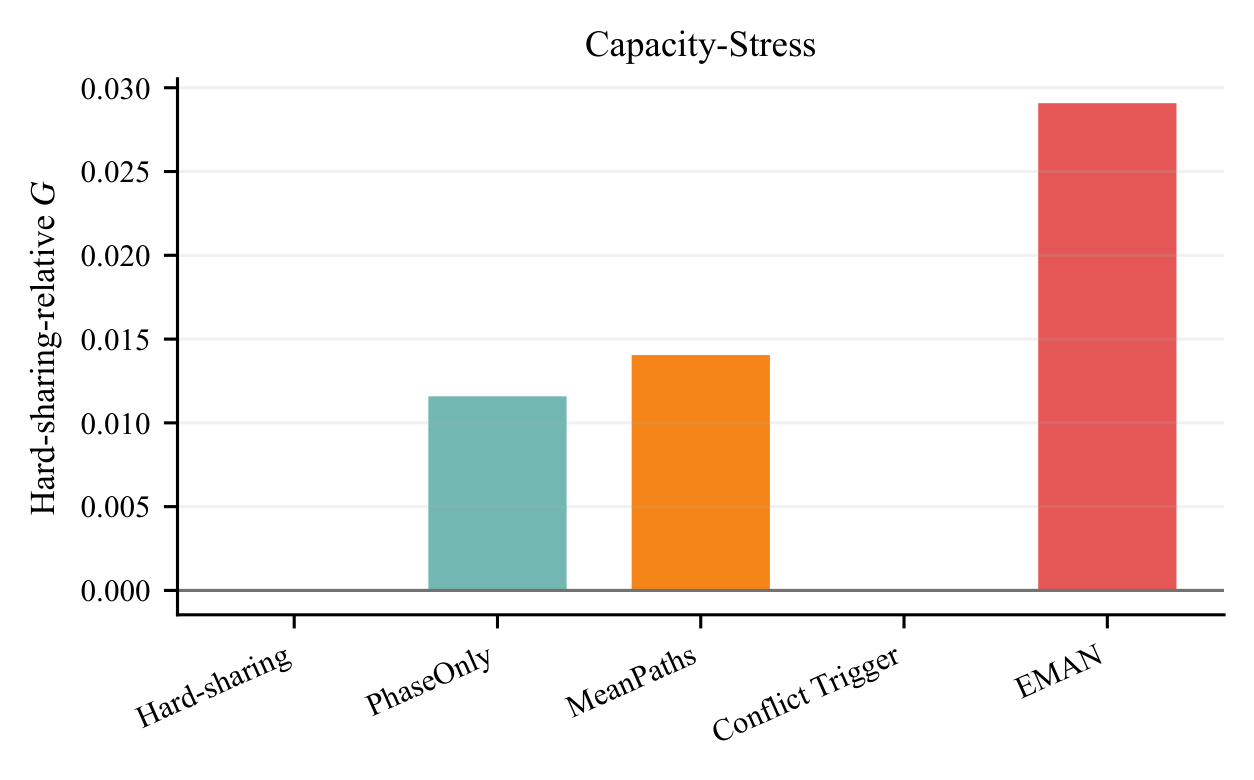}
\par\smallskip
\includegraphics[width=0.72\textwidth]{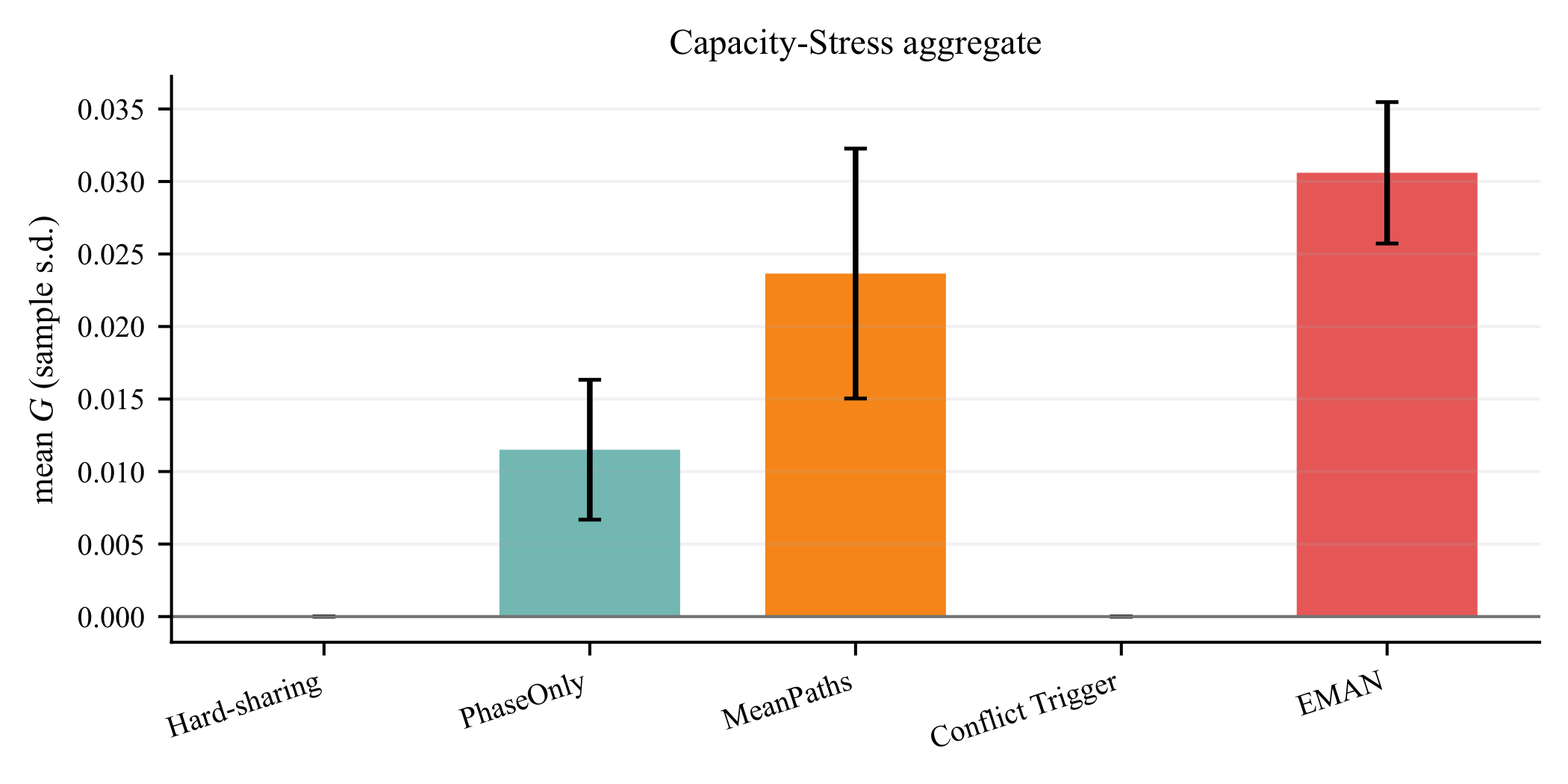}
\caption{Capacity-Stress per-seed and aggregate $G$. The first row retains every registered seed; the lower panel reports the three-seed mean with sample-standard-deviation error bars.}
\label{fig:supp_capacity_performance}
\end{figure*}

Table~\ref{tab:supp_capacity_perf} and
Figure~\ref{fig:supp_capacity_performance} expose the per-seed task trade-offs.
EMAN has the highest three-seed mean $G$, while the per-seed ranking remains mixed. MeanPaths allocates both paths from epoch~1. EMAN has a positive directional aggregate in all three seeds, while structural materialization occurs in Seeds~1 and~3 only. Seed~1 releases late and has its largest directional gain on depth. Seed~2 does not release and remains in exact single-path computation. Its final directional aggregate remains positive under the registered evaluation, with the largest directional gain on surface normals. Seed~3 releases early and shows positive directional gains on all three registered task metrics, with its largest gain on semantic segmentation. The native task metrics remain mixed, so the aggregate claim is limited to the registered mean rather than uniform task-wise superiority.

The non-release Seed~2 is informative rather than missing data. Its phase
curvature does not satisfy the frozen confidence rule. The model therefore
keeps exact sharing. This outcome shows that Capacity-Stress increases the
chance of materialization but does not hard-code a split. It also explains why
the three-seed mean trades some performance against fewer active path
evaluations.

\begin{table}[!t]
\centering
\small
\begin{tabular}{lcc}
\toprule
Method & Mean $G$ & Sample s.d. \\
\midrule
Hard-sharing & 0.0000 & 0.0000 \\
PhaseOnly & 0.0115 & 0.0048 \\
MeanPaths & 0.0237 & 0.0086 \\
Conflict Trigger & -0.0000 & 0.0000 \\
EMAN & \textbf{0.0306} & 0.0049 \\
\bottomrule
\end{tabular}
\caption{Capacity-Stress three-seed aggregate. EMAN has the highest registered mean $G$; the error column is the sample standard deviation.}
\label{tab:supp_capacity_aggregate}
\end{table}

Table~\ref{tab:supp_capacity_events} reports event timing and systems
measurements. Seed~1 releases phase at epoch 32 and paths at epoch 49. Seed~3
releases at epochs 11 and 19. Seed~2 has no release. The final path distances
are $0.0917$, $0$, and $0.5303$. The dual-active ratios are $0.400$, $0$, and
$0.775$.

\begin{table*}[!t]
\centering

\small
\setlength{\tabcolsep}{3.2pt}
\begin{tabular}{llcccccc}
\toprule
Seed & Method & Phase ep. & Path ep. & Dual ratio & Total FLOPs/MP & Path dist. & Wall min. \\
\midrule
Seed 1 & Hard-sharing & -- & -- & 0.000 & 0.962 & 0.0000 & 47.2 \\
 & PhaseOnly & \textbf{32} & -- & 0.000 & 0.962 & 0.0000 & 115.8 \\
 & MeanPaths & -- & 1 & 1.000 & 1.000 & 0.9413 & 59.0 \\
 & Conflict Trigger & -- & -- & 0.000 & 0.962 & 0.0000 & 134.0 \\
 & EMAN & \textbf{32} & \textbf{49} & 0.400 & 0.977 & 0.0917 & 153.7 \\
\midrule
Seed 2 & Hard-sharing & -- & -- & 0.000 & 0.962 & 0.0000 & 71.5 \\
 & PhaseOnly & -- & -- & 0.000 & 0.962 & 0.0000 & 142.0 \\
 & MeanPaths & -- & 1 & 1.000 & 1.000 & 0.9147 & 69.2 \\
 & Conflict Trigger & -- & -- & 0.000 & 0.962 & 0.0000 & 149.7 \\
 & EMAN & -- & -- & 0.000 & 0.962 & 0.0000 & 204.6 \\
\midrule
Seed 3 & Hard-sharing & -- & -- & 0.000 & 0.962 & 0.0000 & 72.8 \\
 & PhaseOnly & \textbf{11} & -- & 0.000 & 0.962 & 0.0000 & 79.6 \\
 & MeanPaths & -- & 1 & 1.000 & 1.000 & 0.8915 & 73.0 \\
 & Conflict Trigger & -- & \textbf{79} & 0.025 & 0.963 & 0.0001 & 150.6 \\
 & EMAN & \textbf{11} & \textbf{19} & 0.775 & 0.992 & 0.5303 & 92.3 \\
\bottomrule
\end{tabular}
\caption{Capacity-Stress event and system details. Total FLOPs are normalized
within each seed by MeanPaths. Wall-clock is hardware-specific.}
\label{tab:supp_capacity_events}
\end{table*}

\begin{table*}[!t]
\centering

\small
\setlength{\tabcolsep}{4pt}
\begin{tabular}{lccccccc}
\toprule
Seed & Probe epoch & Norm. curvature & UCB95 & FD pass & Safety pass & Persistence & Phase release \\
\midrule
Seed 1 & \textbf{32} & $-0.3663$ & \textbf{$-0.00625$} & yes & yes & 2 & \textbf{yes} \\
Seed 2 & 80 & $-0.1151$ & $+0.02387$ & yes & yes & 0 & no \\
Seed 3 & \textbf{11} & $-0.4321$ & \textbf{$-0.00604$} & yes & yes & 2 & \textbf{yes} \\
\bottomrule
\end{tabular}
\caption{Phase-certificate snapshots. Seed~2 reports its terminal probe. Its
upper confidence bound remains positive, so the frozen phase rule does not
accept the event.}
\label{tab:supp_phase_snapshots}
\end{table*}

\begin{table*}[!t]
\centering

\small
\setlength{\tabcolsep}{3.5pt}
\begin{tabular}{lccccccc}
\toprule
Seed & Path epoch & Strength & Half cosine & Median shard cosine & Frozen $\Delta L$ & Dynamic $\Delta L$ & Persistence \\
\midrule
Seed 1 & \textbf{49} & \textbf{0.007971} & 0.9465 & 0.9087 & $-2.4\times10^{-5}$ & $-2.4\times10^{-5}$ & \textbf{2} \\
Seed 2 & -- & -- & -- & -- & -- & -- & -- \\
Seed 3 & \textbf{19} & \textbf{0.004558} & 0.9424 & 0.8715 & $-1.4\times10^{-5}$ & $-1.4\times10^{-5}$ & \textbf{2} \\
\bottomrule
\end{tabular}
\caption{Accepted path-certificate snapshots. Both released seeds exceed the
recorded $10^{-3}$ strength floor, pass consistency, produce negative finite-step
loss changes, and satisfy the two-probe persistence rule. Seed~2 never enters
path certification because phase release is not accepted.}
\label{tab:supp_path_snapshots}
\end{table*}

Table~\ref{tab:supp_capacity_events} separates release timing from final
path distance. Seed~1 materializes at epoch~49. Its dual-active ratio is
$40\%$. Seed~3 materializes at epoch~19 and has a $77.5\%$ ratio. The later
release in Seed~1 leaves less time for divergence. Seed~2 has zero distance
because no second physical path exists. MeanPaths is dual from epoch~1 and
provides the path-compute reference. The wall-clock columns also expose
population-probe overhead. Certification can take more time even when it
reduces path evaluations.

Table~\ref{tab:supp_phase_snapshots} explains the first gate. Seeds~1 and~3
have negative normalized curvature with negative upper confidence bounds,
pass finite-difference and safety checks, and satisfy the two-probe persistence
rule. Seed~2 has negative point curvature but a positive upper confidence
bound. OEC therefore treats the evidence as insufficiently stable and does not
release phase. Table~\ref{tab:supp_path_snapshots} explains the second gate.
The two accepted seeds exceed the strength floor, show high half- and
shard-level cosine agreement, and reduce loss under both finite-step tests.
Seed~2 has no path-certificate row because the first gate never opens.

\begin{figure*}[!t]
\centering
\setlength{\abovecaptionskip}{3pt}
\includegraphics[width=0.455\textwidth]{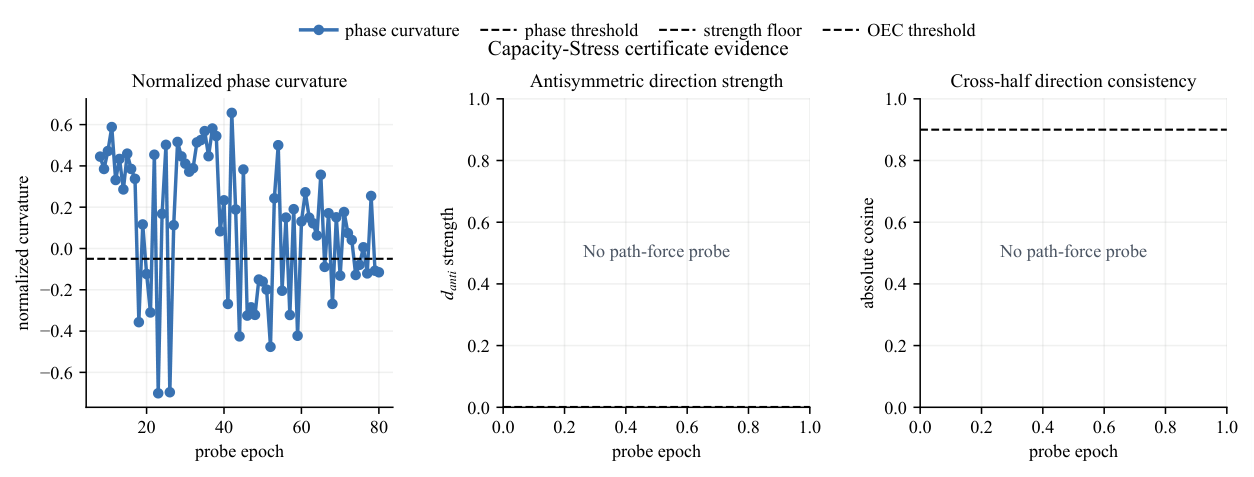}
\hfill
\includegraphics[width=0.455\textwidth]{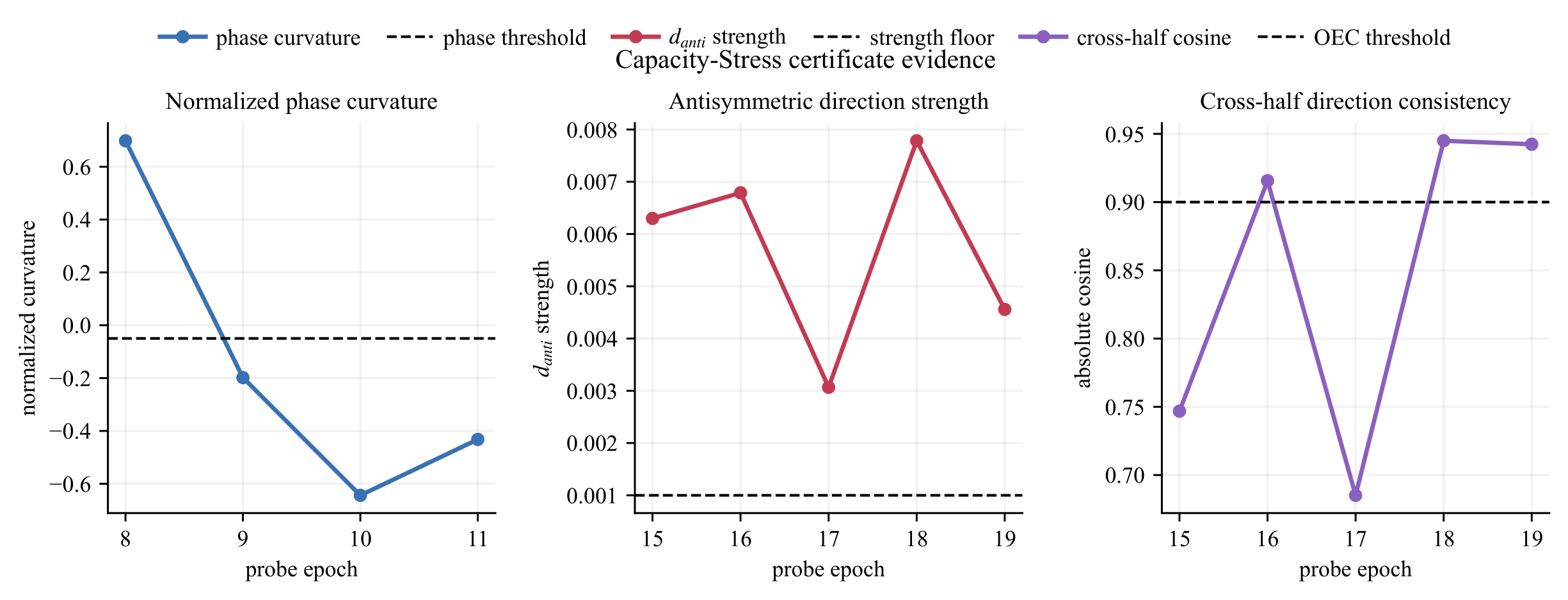}
\caption{Expanded population certificates under Capacity-Stress. Seed~2 never produces a persistent phase certificate and remains in exact single-path computation. Seed~3 passes the path gate over the frozen persistence window and releases at epoch~19.}
\label{fig:supp_capacity_certificates}
\end{figure*}

Figure~\ref{fig:supp_capacity_certificates} contrasts the two decisive
certificate patterns. In Seed~2, only the phase-curvature history is present;
path-force strength and consistency are absent because the run never enters
the OEC path-probing state. In Seed~3, the phase event is followed by a stable
adaptation interval and then by two consecutive path-certificate passes. The
comparison demonstrates that a blank path-force panel denotes an unexecuted
probe stage, not a zero-valued measurement or a plotting failure.

\begin{figure*}[!t]
\centering
\includegraphics[width=0.78\textwidth]{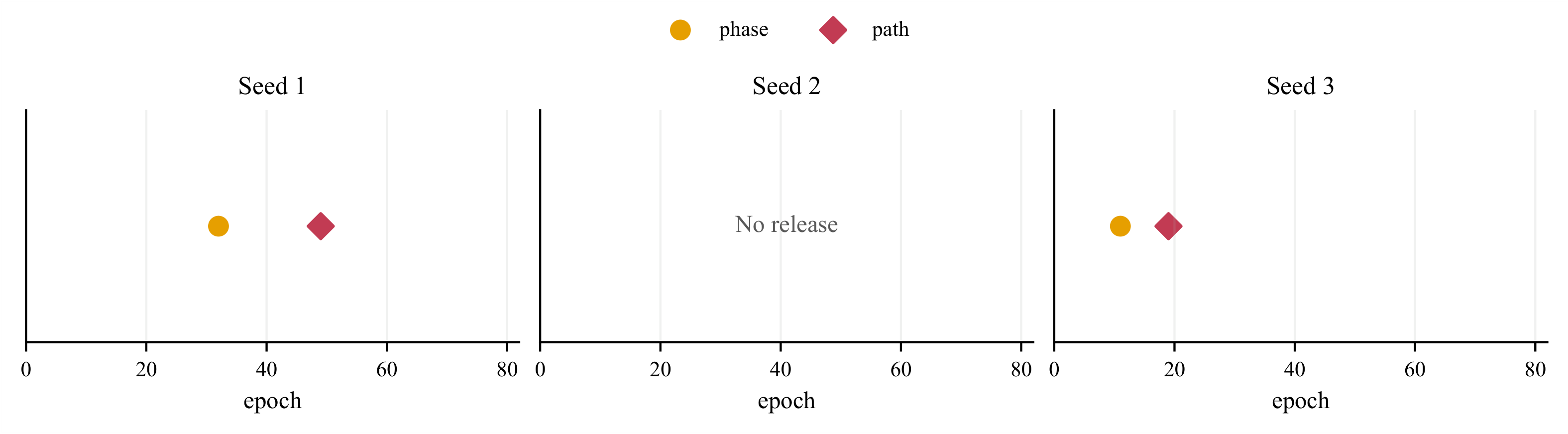}
\caption{Capacity-Stress release events. Seeds~1 and~3 pass the registered phase and path decisions; Seed~2 remains shared. Final path distances and active-path ratios are reported in Table~\ref{tab:supp_capacity_events}.}
\label{fig:supp_capacity_events}
\end{figure*}

Figure~\ref{fig:supp_capacity_events} provides the event trace and the final
structural consequence in one view. Release markers identify when each state
transition occurs. Path distance remains zero until materialization and grows
only after the paths become independently trainable. Seed~1 and Seed~3 follow
this expected ordering. Seed~2 contains no marker and remains at zero distance.
The figure therefore links the discrete certificate decision to the continuous
post-release separation trajectory.

EMAN has the highest mean $G$ under the registered Capacity-Stress protocol. MeanPaths trains two paths throughout. EMAN uses
$69.6\%$ of MeanPaths path-training evaluations on average. Its estimated
total FLOPs are $97.7\%$ because the shared heads, fusion, and other non-path
components remain active. Population certification raises mean wall-clock to
$2.24\times$ MeanPaths in this research implementation. This gap separates the
algorithmic path-evaluation saving from current systems overhead. It motivates
batched or lower-frequency certificate implementations; it does not change the
recorded path-compute accounting.

\subsection{Standard-Capacity NYUv2: Seed Dependence and Systems Cost}
\label{sec:supp_nyuv2_external}

The main paper reports the three-seed aggregate. Table
\ref{tab:supp_nyuv2_perseed} gives every seed. EMAN has positive $G$ in
Seed~1, near-zero $G$ in Seed~2, and negative $G$ in Seed~3. The aggregate
claim is therefore competitive performance. It is not uniform superiority.

\begin{table*}[!t]
\centering

\small
\begin{tabular*}{0.72\textwidth}{@{\extracolsep{\fill}}lccc@{}}
\toprule
Method & Params (M) & Mean peak GiB & Wall-clock (h) \\
\midrule
AdaShare & 26.20 & 2.76 & $2.21\pm0.07$ \\
EMAN & 31.44 & 2.49 & $1.12\pm0.36$ \\
Hard-sharing & 26.20 & 1.96 & $1.12\pm0.01$ \\
MeanPaths & 28.82 & \textbf{1.93} & \textbf{$1.11\pm0.05$} \\
Recon & 27.02 & 3.90 & $2.00\pm0.04$ \\
\bottomrule
\end{tabular*}
\caption{Mean NYUv2 system measurements over three seeds. Wall-clock reports
mean $\pm$ sample standard deviation.}
\label{tab:supp_nyuv2_systems}
\end{table*}

\begin{figure*}[!t]
\centering
\includegraphics[width=0.90\textwidth]{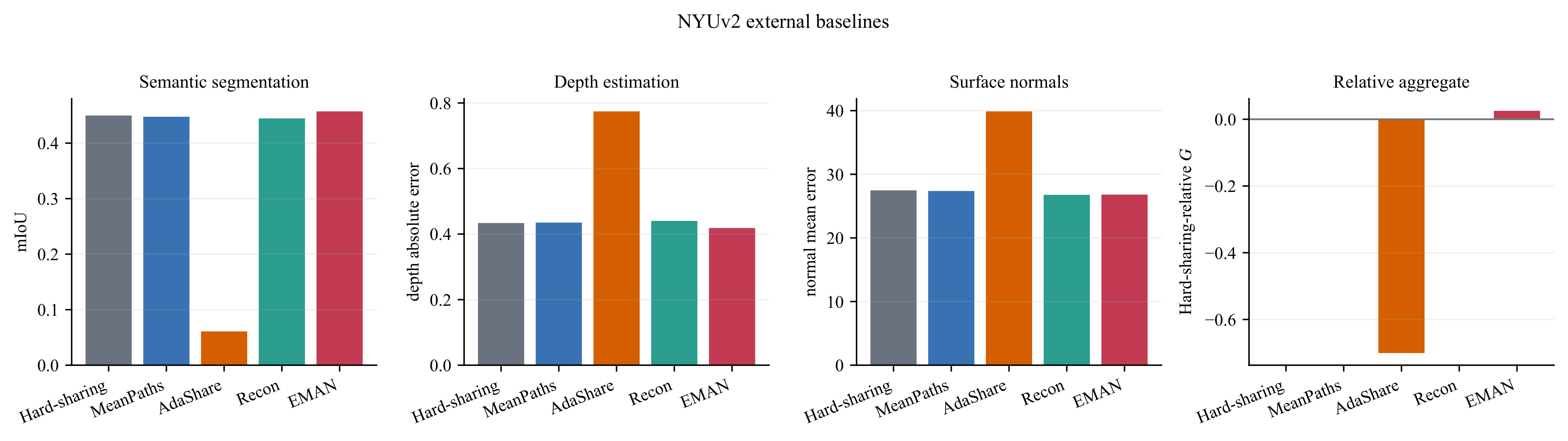}\par
\includegraphics[width=0.90\textwidth]{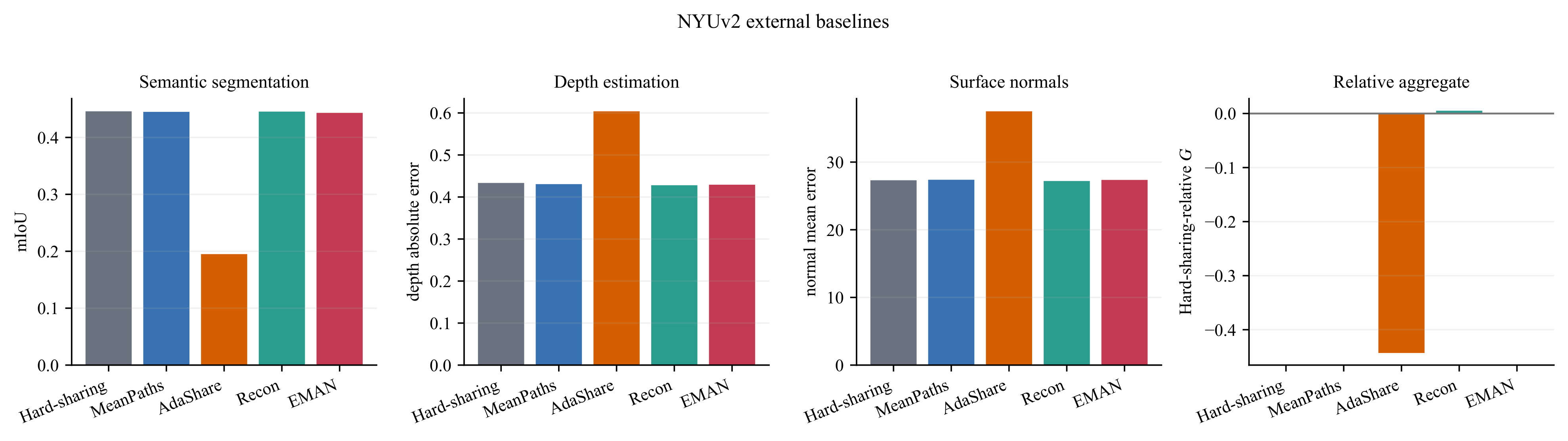}\par
\includegraphics[width=0.90\textwidth]{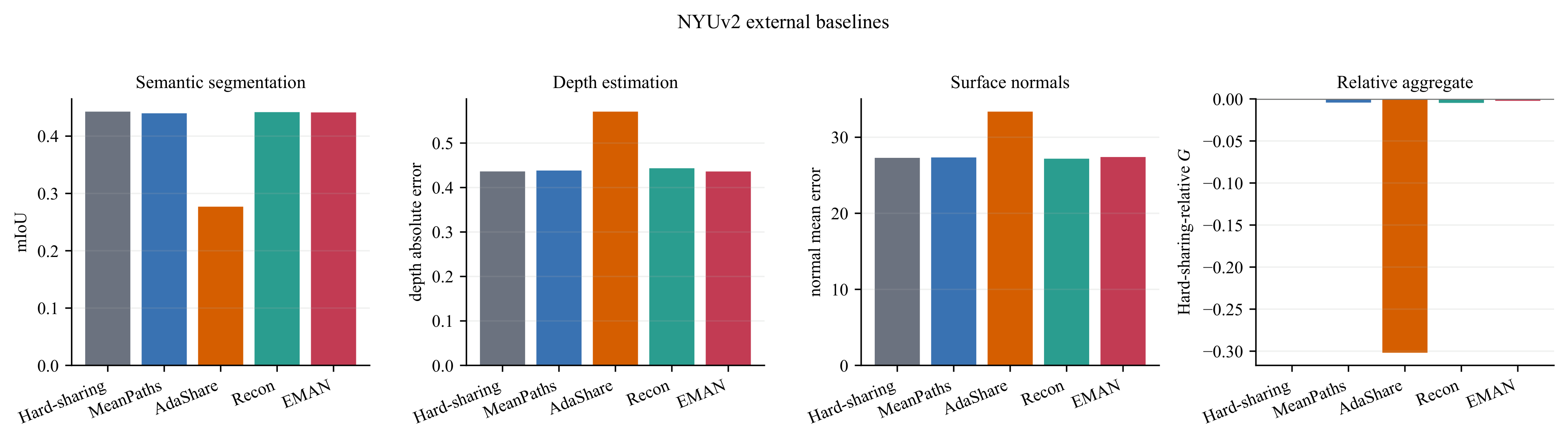}
\caption{Task-level NYUv2 external comparisons for all seeds. The plots show
the positive Seed~1 result and the weaker Seed~2 and Seed~3 results.}
\label{fig:supp_nyuv2_external_seeds}
\end{figure*}

Figure~\ref{fig:supp_nyuv2_external_seeds} visualizes the same seed
dependence at task level. The strong Seed~1 result is not hidden by averaging,
and the weaker Seed~2 and Seed~3 cases are retained at equal scale. Across the
three plots, EMAN remains close to the leading method on each task. The
variation is mainly in the direction and magnitude of the small differences,
not in catastrophic task collapse. AdaShare is the clear unstable baseline in
this matched implementation.

Table~\ref{tab:supp_nyuv2_systems} separates model size, peak memory, and
elapsed time. EMAN uses 31.44M parameters, compared with 26.20M for
Hard-sharing and 28.82M for MeanPaths. Its mean peak memory is below AdaShare
and Recon. Its mean wall-clock matches Hard-sharing and remains close to
MeanPaths, although the larger standard deviation reflects seed-dependent
structural behavior and measurement variability. Recon and AdaShare have
higher runtime in this implementation. These measurements describe the
reported system and do not establish a portable speed ranking across hardware
or software stacks.

\begin{table*}[!t]
\centering

\small
\setlength{\tabcolsep}{4pt}
\begin{tabular}{llccccc}
\toprule
Seed & Method & $G$ & mIoU & Depth Abs. & Normal Mean & Joint Loss \\
\midrule
Seed 1 & Hard-sharing & 0.0000 & 0.4498 & 0.4337 & 27.4766 & 0.9215 \\
 & MeanPaths & -0.0015 & 0.4475 & 0.4350 & 27.3759 & 0.9147 \\
 & AdaShare & -0.7006 & 0.0608 & 0.7743 & 39.8845 & 2.6983 \\
 & Recon & -0.0003 & 0.4445 & 0.4402 & 26.7701 & 0.9511 \\
 & EMAN & \textbf{0.0253} & 0.4572 & 0.4184 & 26.8078 & 0.8992 \\
\midrule
Seed 2 & Hard-sharing & 0.0000 & 0.4459 & 0.4334 & 27.3137 & 0.9588 \\
 & MeanPaths & 0.0005 & 0.4449 & 0.4306 & 27.3891 & 0.9398 \\
 & AdaShare & -0.4431 & 0.1950 & 0.6039 & 37.5043 & 1.4366 \\
 & Recon & \textbf{0.0052} & 0.4454 & 0.4280 & 27.2034 & 0.9291 \\
 & EMAN & 0.0005 & 0.4431 & 0.4293 & 27.3637 & 0.9588 \\
\midrule
Seed 3 & Hard-sharing & \textbf{0.0000} & 0.4427 & 0.4362 & 27.2943 & 0.9606 \\
 & MeanPaths & -0.0043 & 0.4398 & 0.4382 & 27.3466 & 0.9791 \\
 & AdaShare & -0.3018 & 0.2769 & 0.5707 & 33.3711 & 1.2709 \\
 & Recon & -0.0047 & 0.4419 & 0.4433 & 27.1810 & 0.9149 \\
 & EMAN & -0.0023 & 0.4413 & 0.4361 & 27.4025 & 0.9661 \\
\bottomrule
\end{tabular}
\caption{Standard-capacity NYUv2 results for every seed. Boldface marks the
strongest $G$ value within each seed. Since Hard-sharing defines the
reference, its $G$ is zero by construction.}
\label{tab:supp_nyuv2_perseed}
\end{table*}

Table~\ref{tab:supp_nyuv2_perseed} shows why the standard-capacity claim is
stated at the aggregate level. Seed~1 provides a clear positive EMAN result:
it improves segmentation, depth, normal error, and $G$ relative to
its Hard-sharing reference. Seed~2 is near the reference and remains
competitive with the external baselines. Seed~3 is slightly negative in $G$,
although its task values stay close to Hard-sharing. The three-seed mean is
therefore evidence of competitive multi-task behavior, not a claim that EMAN
wins every seed or every metric.

\begin{figure*}[!t]
\centering
\includegraphics[width=0.47\textwidth]{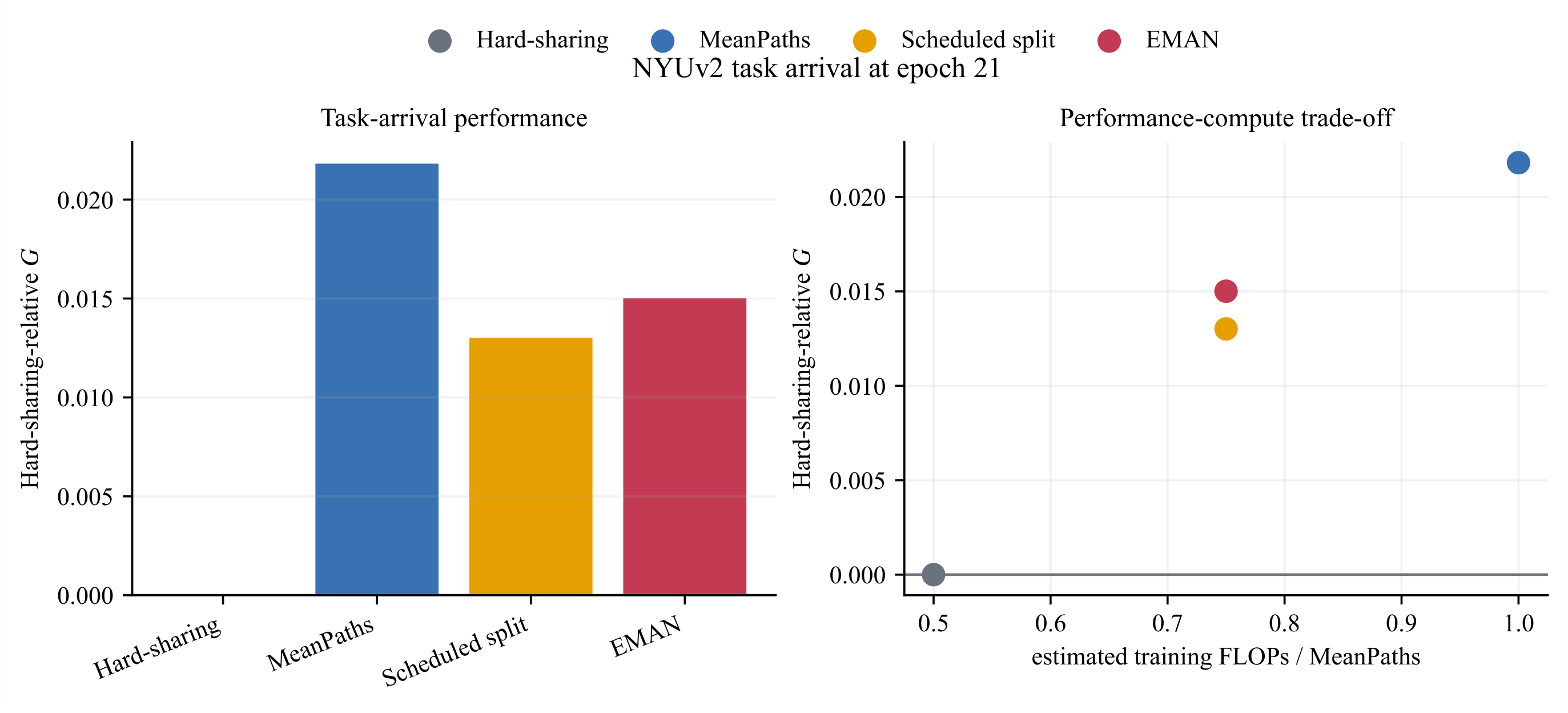}
\par
\includegraphics[width=0.47\textwidth]{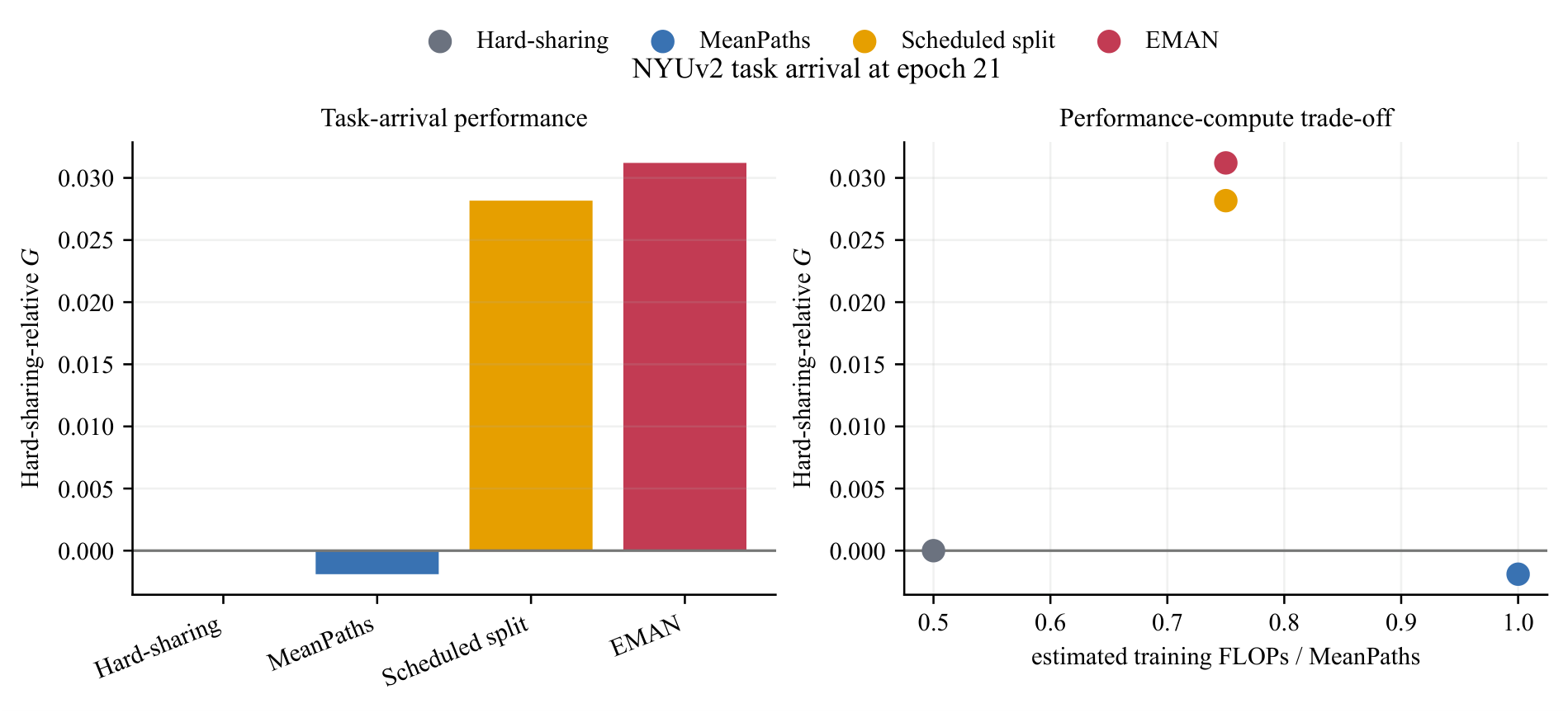}
\hfill
\includegraphics[width=0.47\textwidth]{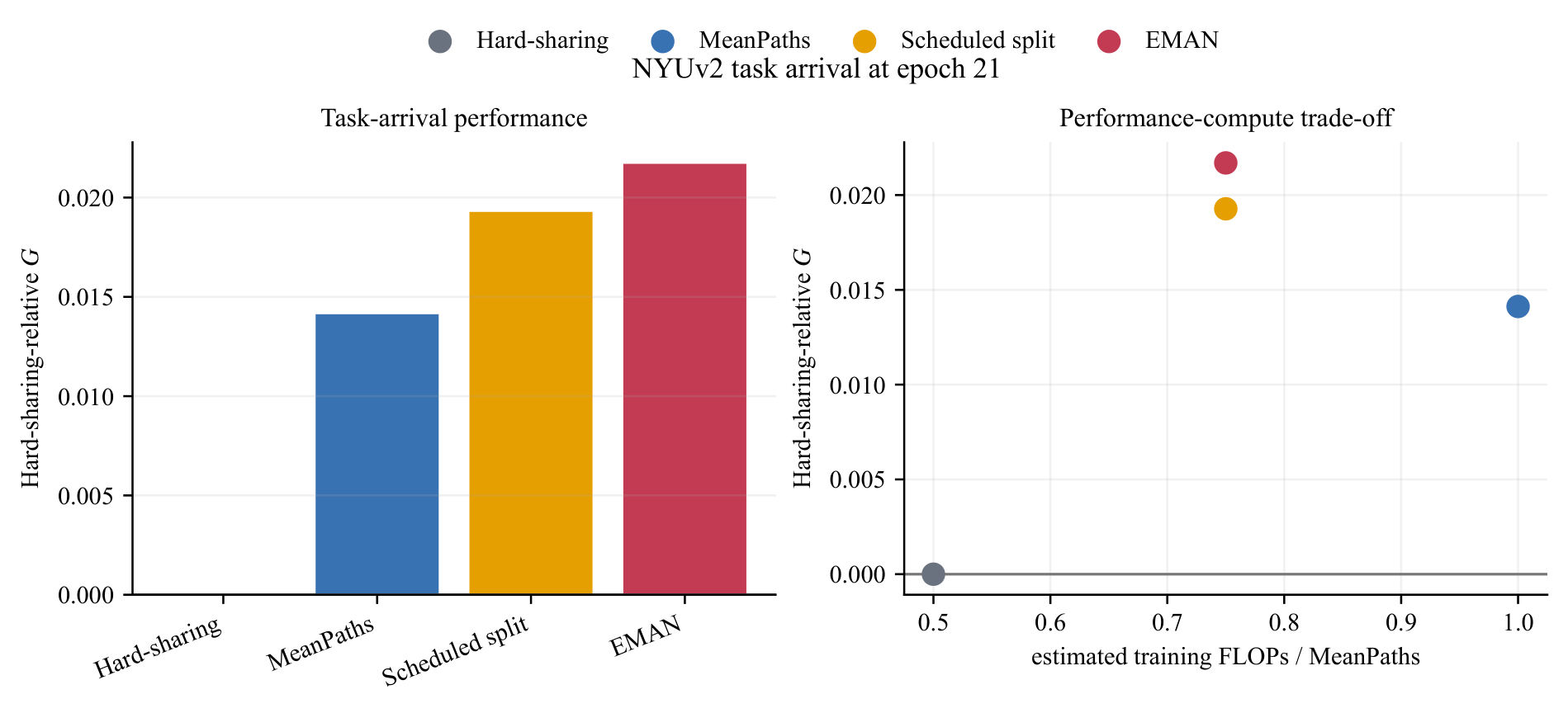}
\caption{NYUv2 task-arrival performance and compute for all seeds. EMAN and
the matched scheduled split use the same estimated path-training budget.}
\label{fig:supp_nyuv2_arrival}
\end{figure*}

\subsection{NYUv2 Task Arrival at Epoch 21}
\label{sec:supp_nyuv2_arrival}

The task-arrival protocol delays part of the demand until epoch 21. EMAN is
not given this epoch. The frozen transition releases at the demand change in
all three runs. Table~\ref{tab:supp_nyuv2_arrival} gives every seed. EMAN
achieves $G=0.0226\pm0.0081$. Scheduled split obtains
$0.0202\pm0.0076$. MeanPaths obtains $0.0114\pm0.0121$.

\begin{table*}[!t]
\centering
\setlength{\abovecaptionskip}{3pt}

\small
\setlength{\tabcolsep}{3.6pt}
\begin{tabular}{llcccccc}
\toprule
Seed & Method & $G$ & mIoU & Depth Abs. & Normal Mean & Path FLOPs/MP & Wall min. \\
\midrule
Seed 1 & Hard-sharing & 0.0000 & 0.1206 & 0.8615 & 43.1966 & 0.50 & 54.2 \\
 & MeanPaths & \textbf{0.0218} & 0.1267 & 0.8578 & 42.7270 & 1.00 & 57.7 \\
 & Scheduled split & 0.0130 & 0.1250 & 0.8581 & 43.2455 & \textbf{0.75} & 57.5 \\
 & EMAN & 0.0150 & 0.1252 & 0.8578 & 43.0776 & \textbf{0.75} & 57.6 \\
\midrule
Seed 2 & Hard-sharing & 0.0000 & 0.1278 & 0.8598 & 43.3667 & 0.50 & 55.5 \\
 & MeanPaths & -0.0019 & 0.1261 & 0.8588 & 43.0829 & 1.00 & 57.6 \\
 & Scheduled split & 0.0282 & 0.1356 & 0.8524 & 42.7119 & \textbf{0.75} & 57.6 \\
 & EMAN & \textbf{0.0312} & 0.1360 & 0.8505 & 42.5749 & \textbf{0.75} & 56.4 \\
\midrule
Seed 3 & Hard-sharing & 0.0000 & 0.1286 & 0.8608 & 43.2018 & 0.50 & 56.8 \\
 & MeanPaths & 0.0141 & 0.1322 & 0.8579 & 42.7357 & 1.00 & 56.7 \\
 & Scheduled split & 0.0193 & 0.1345 & 0.8567 & 42.8944 & \textbf{0.75} & 56.7 \\
 & EMAN & \textbf{0.0217} & 0.1353 & 0.8571 & 42.8449 & \textbf{0.75} & 56.8 \\
\bottomrule
\end{tabular}
\caption{NYUv2 task-arrival results. The estimated path-training FLOPs are
normalized by always-dual MeanPaths. Boldface marks the strongest $G$ values
and the matched delayed path budget.}
\label{tab:supp_nyuv2_arrival}
\end{table*}

Table~\ref{tab:supp_nyuv2_arrival} controls the main compute comparison.
Scheduled split and EMAN materialize at the same epoch and use the same
$0.75$ path-FLOP ratio. Their difference is the release direction: Scheduled
split uses the matched random split, while EMAN uses the certified
antisymmetric direction. EMAN exceeds Scheduled split in all three seeds by
$0.0020$, $0.0030$, and $0.0024$ in $G$. The consistent sign supports the
reported aggregate, but the sample is too small to establish a universal
direction advantage.

Figure~\ref{fig:supp_nyuv2_arrival} shows that the comparison is not produced
by a wall-clock discrepancy. Scheduled split and EMAN have nearly identical
elapsed times, and both remain close to MeanPaths. The saved path evaluations
occur before epoch~21. After release, all three two-path methods execute the
same number of paths. The result therefore isolates delayed allocation rather
than sparse inference.

\subsection{PASCAL-Context as a Cross-Dataset Boundary}
\label{sec:supp_pascal_external}

The static PASCAL-Context protocol activates semantic segmentation,
human-part segmentation, saliency, and surface-normal prediction from the
beginning. It is a cross-dataset scope check rather than a delayed-growth
setting. Table~\ref{tab:supp_pascal_external} reports all seeds. Recon has the
highest mean aggregate, $G=0.0158\pm0.0069$, while EMAN has
$G=-0.0098\pm0.0007$. The aggregate is consistently small and negative, but
the native task metrics are mixed. EMAN obtains the strongest semantic score
and joint loss in Seed~1, the strongest saliency and joint loss in Seed~2, and
near-reference semantic and joint-loss values in Seed~3. The result therefore
narrows the performance scope without implying a uniform task failure.

\begin{table*}[!t]
\centering
\setlength{\abovecaptionskip}{3pt}

\small
\setlength{\tabcolsep}{3.2pt}
\begin{tabular}{llcccccc}
\toprule
Seed & Method & $G$ & Semantic & Parts & Saliency & Normal Mean & Joint Loss \\
\midrule
Seed 1 & Hard-sharing & 0.0000 & 0.7037 & 0.1277 & 0.6999 & 20.3541 & 0.5307 \\
 & MeanPaths & -0.0080 & 0.7042 & 0.1236 & 0.6998 & 20.3702 & 0.5307 \\
 & AdaShare & -0.3186 & 0.5697 & 0.0490 & 0.6011 & 27.0011 & 0.7185 \\
 & Recon & \textbf{0.0118} & 0.7006 & 0.1302 & 0.7006 & \textbf{19.7143} & 0.5659 \\
 & EMAN & -0.0091 & \textbf{0.7044} & 0.1253 & 0.6937 & 20.5573 & \textbf{0.5235} \\
\midrule
Seed 2 & Hard-sharing & 0.0000 & 0.7006 & 0.1238 & 0.6950 & 20.4850 & 0.5413 \\
 & MeanPaths & -0.0001 & 0.7021 & 0.1234 & 0.6954 & 20.4846 & 0.5394 \\
 & AdaShare & -0.4626 & 0.5377 & 0.0507 & 0.2125 & 27.3071 & 1.0061 \\
 & Recon & \textbf{0.0237} & 0.6980 & \textbf{0.1336} & 0.6941 & \textbf{20.0724} & 0.5671 \\
 & EMAN & -0.0100 & 0.6990 & 0.1204 & \textbf{0.6961} & 20.7316 & \textbf{0.5320} \\
\midrule
Seed 3 & Hard-sharing & 0.0000 & 0.7028 & 0.1227 & 0.7004 & 20.4786 & 0.5323 \\
 & MeanPaths & 0.0066 & 0.7000 & 0.1268 & 0.6976 & 20.4531 & \textbf{0.5310} \\
 & AdaShare & -0.3082 & 0.5714 & 0.0799 & 0.4438 & 27.2491 & 0.8009 \\
 & Recon & \textbf{0.0118} & 0.6877 & 0.1249 & \textbf{0.7034} & \textbf{19.5289} & 0.6386 \\
 & EMAN & -0.0105 & 0.7021 & 0.1208 & 0.6918 & 20.7414 & 0.5317 \\
\bottomrule
\end{tabular}
\caption{PASCAL-Context external comparison for every seed. Higher is better
except normal mean and joint loss.}
\label{tab:supp_pascal_external}
\end{table*}

\begin{figure*}[!t]
\centering
\includegraphics[width=0.86\textwidth]{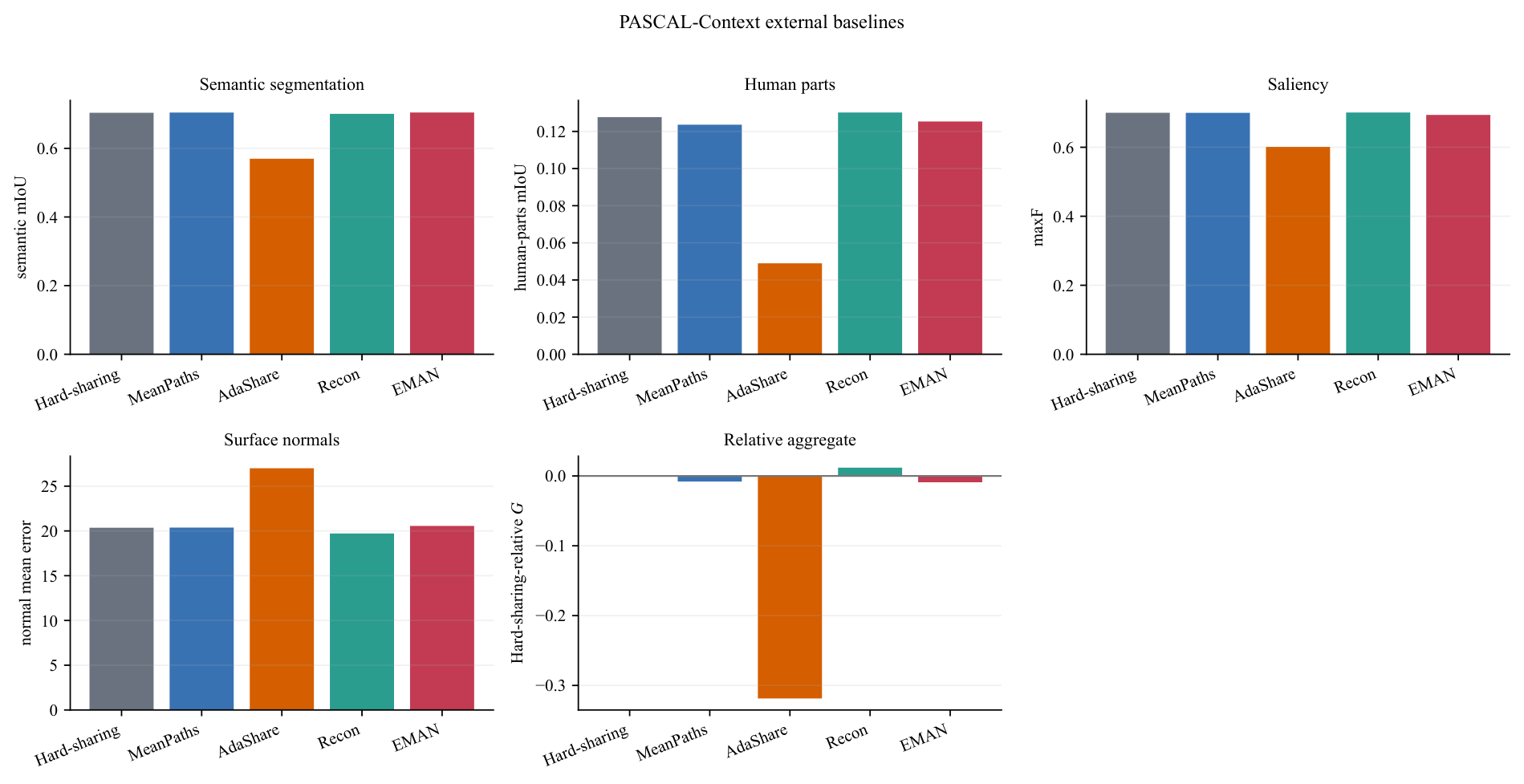}
\par\smallskip
\includegraphics[width=0.86\textwidth]{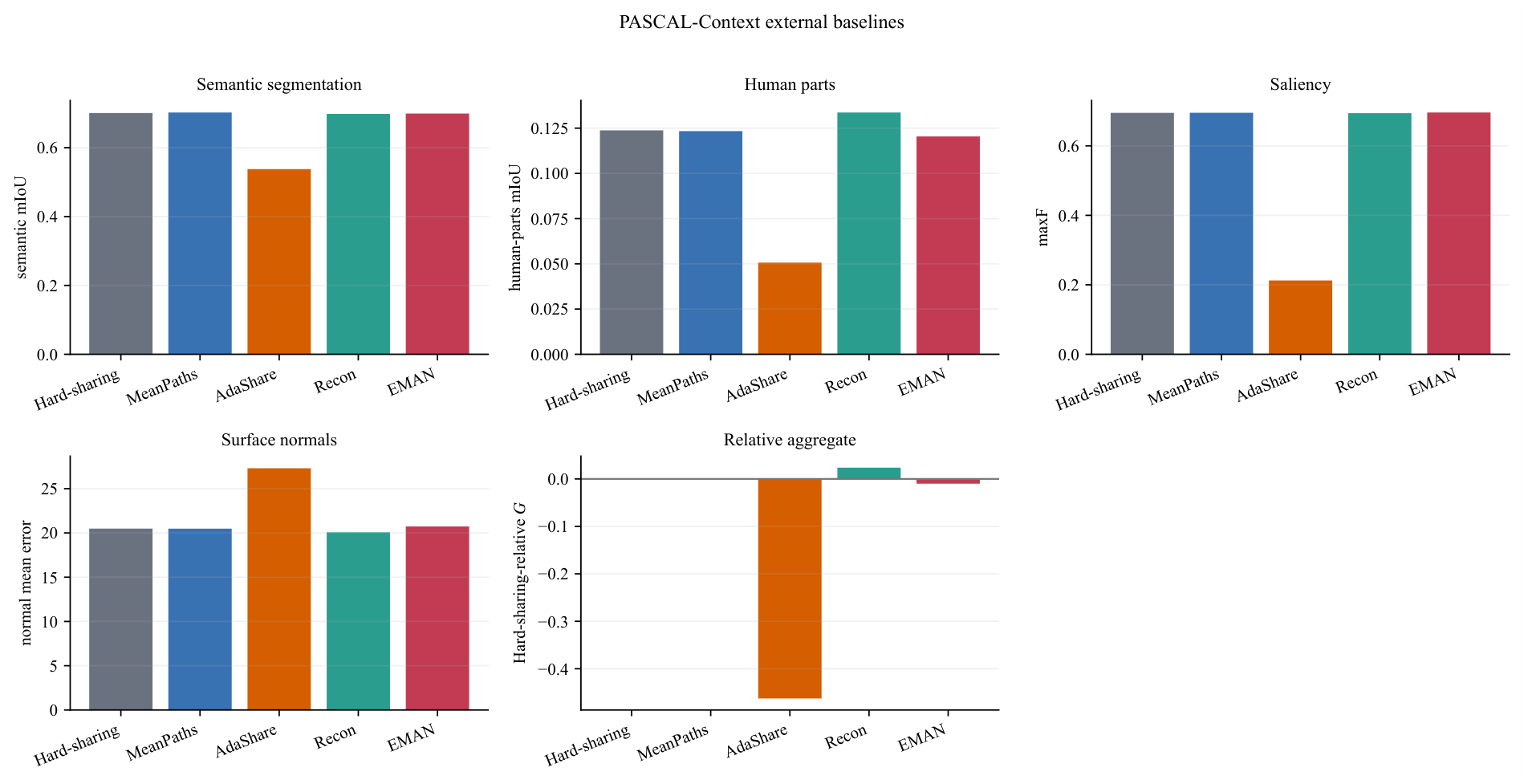}
\caption{PASCAL-Context external comparisons for Seed~1 and Seed~2. The
panels retain the native task metrics and expose the mixed task-level pattern
behind the aggregate.}
\label{fig:supp_pascal_external_12}
\end{figure*}

\begin{figure*}[!t]
\centering
\includegraphics[width=0.86\textwidth]{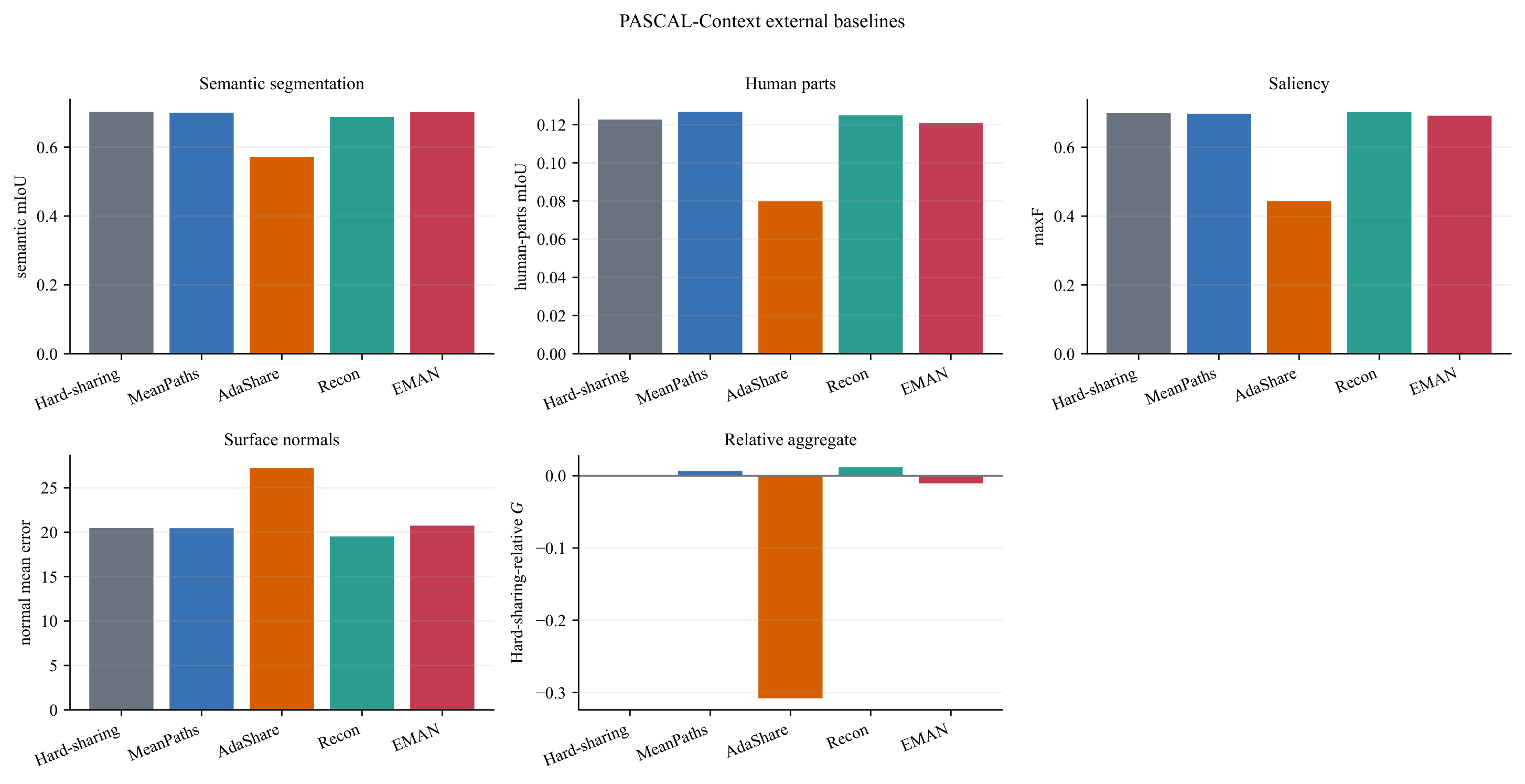}
\caption{PASCAL-Context external comparison for Seed~3. EMAN remains close to
the shared reference on semantic segmentation, while Recon has
the strongest normal result.}
\label{fig:supp_pascal_external_3}
\end{figure*}

Figures~\ref{fig:supp_pascal_external_12} and
\ref{fig:supp_pascal_external_3} make the aggregate decomposition explicit.
The negative $G$ is mainly associated with human-part segmentation and normal
prediction. It coexists with competitive semantic, saliency, and joint-loss
values. Because all four tasks are present from epoch~1, this protocol offers
no pre-demand interval in which delayed materialization can save path
training. It is therefore a demanding test of final multi-task behavior, but
not a direct test of EMAN's timing advantage.

The static PASCAL-Context setting presents all four tasks from the beginning,
and EMAN reaches the dual-path state early. This leaves little opportunity
for delayed path-training savings. We therefore use the static protocol mainly
as a cross-dataset performance boundary. The task-arrival protocols provide
the clearer evidence for release timing under changing demand.

\subsection{PASCAL-Context Task Arrival at Epoch 21}
\label{sec:supp_pascal_arrival21}

This protocol activates the later tasks at epoch 21. EMAN releases at epoch 21
in every seed. MeanPaths is dual from epoch 1. Table
\ref{tab:supp_pascal_arrival21} reports the complete result. EMAN has the
highest $G$ in Seed~3. It remains close to MeanPaths in Seed~1 and Seed~2.

\begin{table*}[!t]
\centering
\setlength{\abovecaptionskip}{3pt}

\small
\setlength{\tabcolsep}{3.4pt}
\begin{tabular}{llcccccc}
\toprule
Seed & Method & $G$ & Semantic & Parts & Saliency & Normal Mean & Path ep. \\
\midrule
Seed 1 & Hard-sharing & 0.0000 & 0.5713 & 0.0616 & 0.4947 & 27.0256 & -- \\
 & MeanPaths & \textbf{0.0009} & 0.5704 & 0.0614 & 0.4956 & 26.8403 & 1 \\
 & EMAN & 0.0008 & 0.5727 & 0.0615 & 0.4996 & 27.2141 & \textbf{21} \\
\midrule
Seed 2 & Hard-sharing & 0.0000 & 0.5699 & 0.0619 & 0.4944 & 27.1443 & -- \\
 & MeanPaths & \textbf{0.0046} & 0.5709 & 0.0624 & 0.4977 & 27.0951 & 1 \\
 & EMAN & 0.0033 & 0.5706 & 0.0622 & 0.4966 & 27.0466 & \textbf{21} \\
\midrule
Seed 3 & Hard-sharing & 0.0000 & 0.5737 & 0.0615 & 0.5010 & 26.9492 & -- \\
 & MeanPaths & 0.0027 & 0.5714 & 0.0617 & 0.5036 & 26.7428 & 1 \\
 & EMAN & \textbf{0.0090} & \textbf{0.5763} & \textbf{0.0624} & \textbf{0.5069} & 26.7782 & \textbf{21} \\
\bottomrule
\end{tabular}
\caption{PASCAL-Context task-arrival results.}
\label{tab:supp_pascal_arrival21}
\end{table*}

Table~\ref{tab:supp_pascal_arrival21} separates timing stability from
final task variation. EMAN releases at epoch~21 in every seed, while MeanPaths
is dual from epoch~1. The three-seed mean $G$ values are $0$ for Hard-sharing,
$0.0027\pm0.0018$ for MeanPaths, and $0.0043\pm0.0042$ for EMAN. Seed~1 and
Seed~2 are close to MeanPaths. Seed~3 gives the strongest aggregate and the
best semantic, parts, and saliency values among the listed methods.

Figure~\ref{fig:supp_pascal_arrival_timing} summarizes the same event timing in all
three runs. The measured MeanPaths times are 14.1, 15.8, and 16.4 minutes. The
EMAN times are 14.2, 15.7, and 16.2 minutes. These values show that delayed
path activation does not introduce a large wall-clock penalty in this
protocol. They are hardware-specific and are not used to claim a general
speedup.

\subsection{Earlier Arrival at Epoch 11}
\label{sec:supp_arrival11}

The earlier-arrival check uses Seed~1. Saliency and normal demand arrive at
epoch 11. The transition rule remains frozen. EMAN releases the path at epoch
11. Table~\ref{tab:supp_arrival11} reports all methods.

\begin{table*}[!t]
\centering
\setlength{\abovecaptionskip}{3pt}

\small
\begin{tabular}{lcccccc}
\toprule
Method & $G$ & Semantic & Parts & Saliency & Normal Mean & Wall min. \\
\midrule
Hard-sharing & 0.0000 & 0.5690 & 0.0616 & 0.4989 & 25.7307 & 16.7 \\
MeanPaths & \textbf{0.0054} & 0.5716 & 0.0618 & 0.5012 & \textbf{25.4987} & 15.4 \\
Scheduled split & 0.0020 & 0.5728 & 0.0617 & 0.5005 & 25.8105 & 15.8 \\
EMAN & 0.0037 & \textbf{0.5740} & \textbf{0.0619} & \textbf{0.5018} & 25.8813 & 18.6 \\
\bottomrule
\end{tabular}
\caption{PASCAL-Context earlier-arrival sensitivity. All values use Seed~1.}
\label{tab:supp_arrival11}
\end{table*}

\begin{figure*}[!t]
\centering
\includegraphics[width=0.47\textwidth]{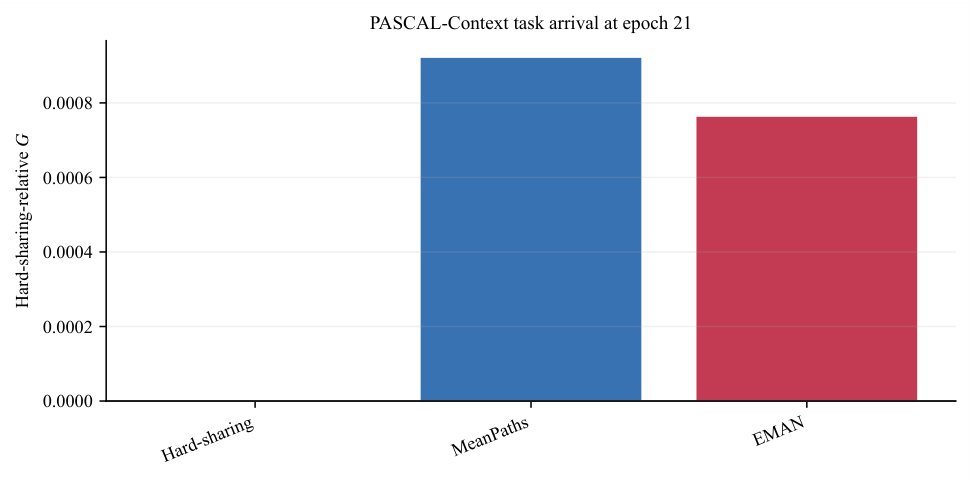}
\hfill
\includegraphics[width=0.47\textwidth]{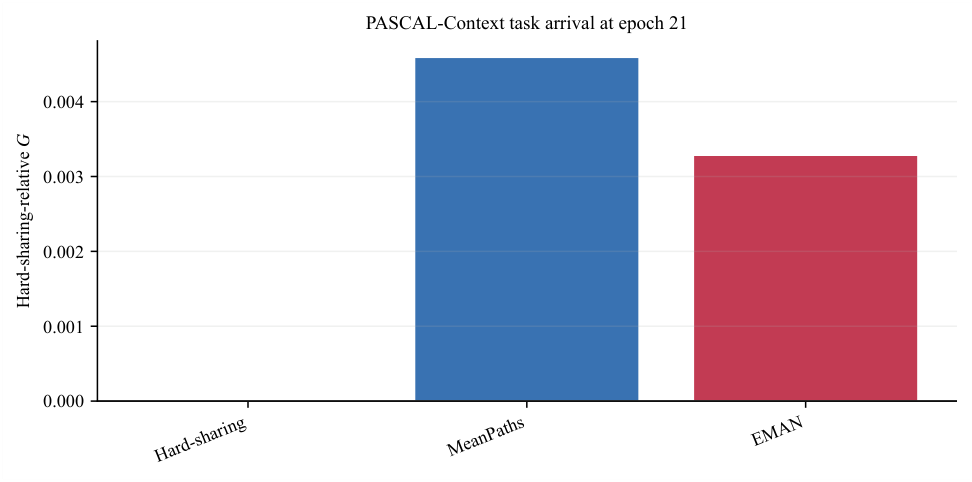}
\par\smallskip
\includegraphics[width=0.47\textwidth]{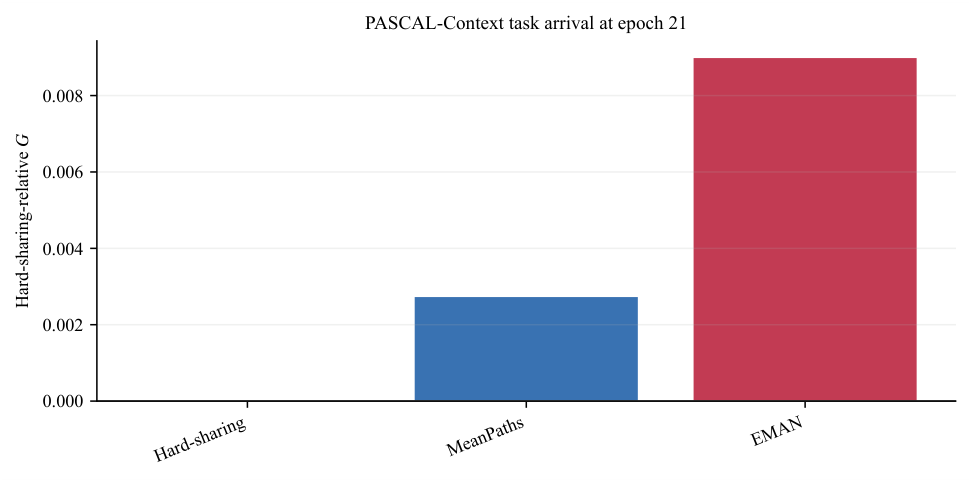}
\hfill
\includegraphics[width=0.47\textwidth]{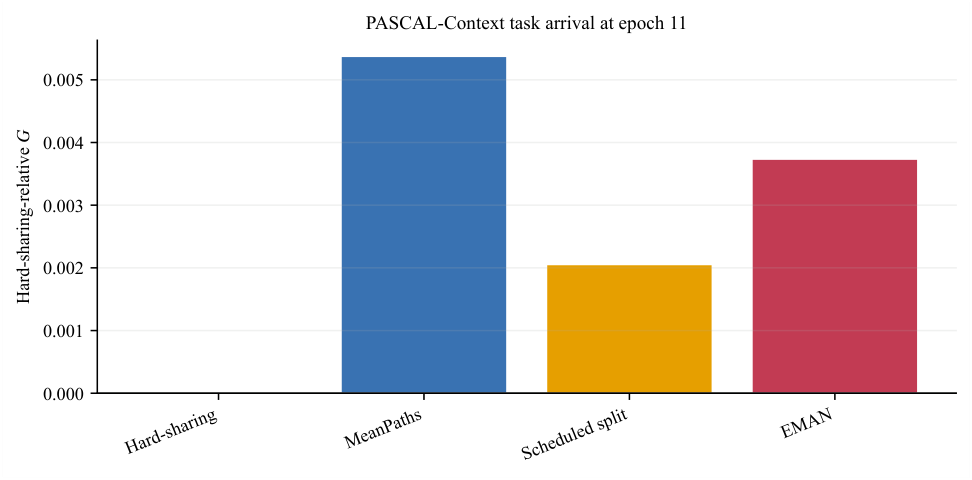}
\caption{PASCAL-Context timing sensitivity. The first three panels show the registered arrival-21 runs; the final panel shows the earlier-arrival check. Tables~\ref{tab:supp_pascal_arrival21} and~\ref{tab:supp_arrival11} distinguish task-arrival epochs from path-release events and report the native metrics.}
\label{fig:supp_pascal_arrival_timing}
\end{figure*}

Table~\ref{tab:supp_arrival11} reports the Seed~1 comparison. EMAN obtains
$G=0.0037$. Scheduled split obtains $0.0020$. MeanPaths obtains $0.0054$.
EMAN has the best semantic, human-part, and saliency scores. MeanPaths has the
strongest normal error and aggregate. The experiment therefore supports timing
adaptation, not superiority to always-dual training.

Figure~\ref{fig:supp_pascal_arrival_timing} is the key timing-sensitivity diagnostic. The later tasks arrive at epoch~11, ten epochs earlier than in the primary protocol, and the path-release event occurs at epoch~11 under the unchanged rule. This correspondence does not imply that the certificate observes or identifies the task-arrival label. EMAN and Scheduled split use an active dual-path ratio of $0.75$
and an estimated path-training ratio of $0.875$ relative to MeanPaths under
this shorter horizon. The movement of the event, not the small performance
difference, is the central result.

\begin{table*}[!t]
\centering

\small
\setlength{\tabcolsep}{3.0pt}
\renewcommand{\arraystretch}{0.94}

\begin{tabular*}{\textwidth}
{@{\extracolsep{\fill}}llcccccc@{}}
\toprule
Seed & Method & $G$ & Semantic & Parts & Saliency
& Normal Mean & Phase/Path ep. \\
\midrule
Seed 1 & Hard-sharing
& 0.0000 & 0.5713 & 0.0615 & 0.5000 & 25.6474 & --/-- \\
& MeanPaths
& 0.0003 & 0.5711 & 0.0615 & 0.4991 & 25.5438 & --/1 \\
& Scheduled split
& -0.0072 & 0.5711 & 0.0615 & 0.5038 & 26.5763 & \textbf{1/11} \\
& Conflict Trigger
& \textbf{0.0073} & 0.5752 & 0.0622 & 0.5034 & 25.5520 & --/-- \\
& EMAN
& -0.0072 & 0.5708 & 0.0614 & 0.5007 & 26.3744 & \textbf{1/11} \\
\midrule
Seed 2 & Hard-sharing
& 0.0000 & 0.5719 & 0.0623 & 0.5000 & 25.6783 & --/-- \\
& MeanPaths
& -0.0013 & 0.5707 & 0.0624 & 0.4965 & 25.5967 & --/1 \\
& Scheduled split
& -0.0069 & 0.5701 & 0.0621 & 0.4978 & 26.0921 & \textbf{1/11} \\
& Conflict Trigger
& \textbf{0.0087} & 0.5774 & 0.0622 & 0.5024 & 25.1159 & --/-- \\
& EMAN
& -0.0067 & 0.5705 & 0.0622 & 0.4988 & 26.1744 & \textbf{1/11} \\
\midrule
Seed 3 & Hard-sharing
& 0.0000 & 0.5739 & 0.0618 & 0.5023 & 25.7341 & --/-- \\
& MeanPaths
& 0.0036 & 0.5735 & 0.0616 & 0.5057 & 25.4228 & --/1 \\
& Scheduled split
& -0.0003 & 0.5748 & 0.0620 & 0.5046 & 26.0294 & \textbf{1/11} \\
& Conflict Trigger
& \textbf{0.0043} & 0.5720 & 0.0622 & 0.5062 & 25.5642 & --/-- \\
& EMAN
& 0.0022 & 0.5745 & 0.0622 & 0.5045 & 25.8008 & \textbf{1/11} \\
\bottomrule
\end{tabular*}
\caption{PASCAL-Context ramp 11--15. A missing event is shown as ``--''.}
\label{tab:supp_ramp}
\end{table*}

\begin{figure*}[!t]
\centering
\includegraphics[width=0.72\textwidth]{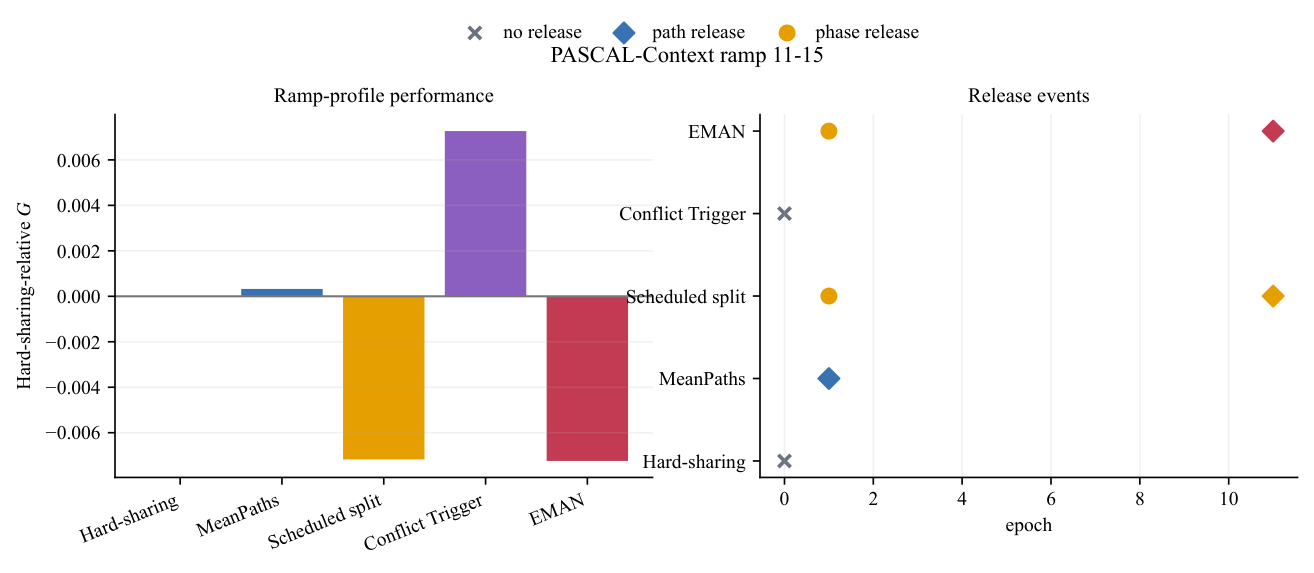}
\par\smallskip
\includegraphics[width=0.72\textwidth]{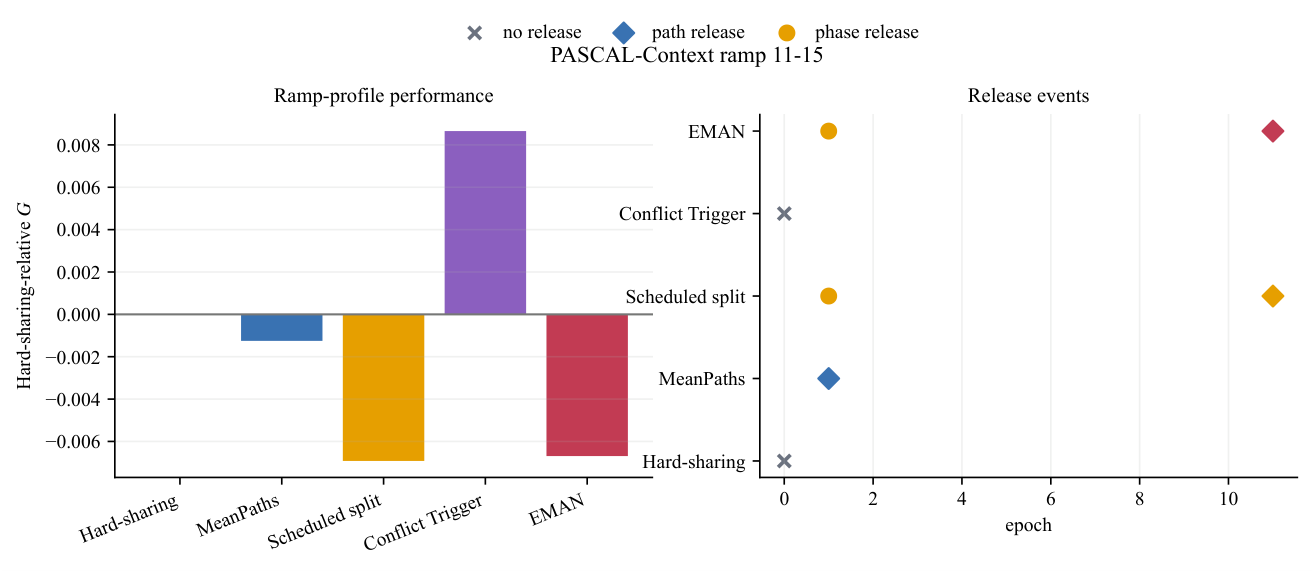}
\par\smallskip
\includegraphics[width=0.72\textwidth]{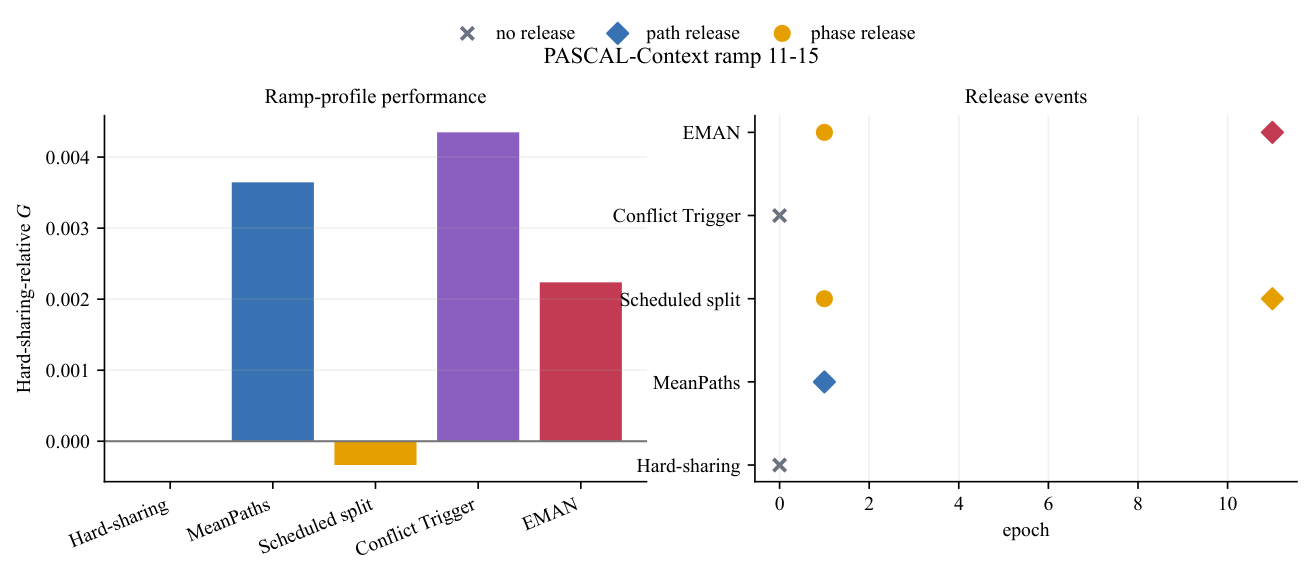}
\caption{Gradual-demand results for Seeds~1--3. The archived evidence does not support EMAN superiority over Conflict Trigger in this protocol.}
\label{fig:supp_ramp}
\end{figure*}

\subsection{Gradual Demand: Ramp from Epochs 11 to 15}
\label{sec:supp_ramp}

The ramp protocol gradually activates the later demand from epochs 11 to 15.
It tests a smoother demand profile than the discrete arrival protocols.
Table~\ref{tab:supp_ramp} retains every seed. EMAN and Scheduled split release
at epoch~11. Conflict Trigger does not release in these archived runs. The
mean $G$ values are $0.0009\pm0.0025$ for MeanPaths,
$-0.0048\pm0.0039$ for Scheduled split, $0.0068\pm0.0022$ for Conflict
Trigger, and $-0.0039\pm0.0053$ for EMAN.

The strongest aggregate is obtained by Conflict Trigger without a structural
release. This observation suggests that the archived ramp does not create a
clear benefit for additional path capacity. EMAN and the matched Scheduled
split have nearly identical means and the same phase/path epochs, so the
result does not isolate a disadvantage of the certified direction. It instead
identifies a scope condition: the current phase-first rule is most clearly
supported under discrete demand onset, while a smooth ramp can favor retaining
the shared state.

Figure~\ref{fig:supp_ramp} confirms the same pattern in every seed. Seed~3 is
the most favorable EMAN run and remains positive, while Seeds~1 and~2 are
negative for both EMAN and Scheduled split. Retaining all three runs prevents
the discrete-arrival result from being generalized to every temporal demand
profile.

\subsection{Replay, Provenance, and Reproducibility Boundaries}
\label{sec:supp_repro}

The accompanying code package provides checksums, per-seed result tables,
release summaries, figure source data, plotting scripts, and replay summary
records. Large natural-image checkpoints are not redistributed.

Archived checkpoint audits reproduced the recorded metrics for the retained
task-arrival run families. The accompanying code package includes the
corresponding replay summaries, but it does not redistribute the large
natural-image checkpoints. Online and offline evaluations agree within the
recorded tolerances. Exact cross-platform bitwise identity is not claimed.

\begin{table}[!t]
\centering

\small
\begin{tabular}{lc}
\toprule
Audited run family & Replay status \\
\midrule
PASCAL arrival-21 EMAN, three seeds & \textbf{PASS} \\
PASCAL matched scheduled split, three seeds & \textbf{PASS} \\
PASCAL arrival-11 sensitivity & \textbf{PASS} \\
\bottomrule
\end{tabular}
\caption{Replay status for the retained task-arrival audits.
PASS denotes agreement within the recorded tolerance.}
\label{tab:supp_replay}
\end{table}

\begin{table*}[!t]
\centering

\small
\begin{tabular}{p{0.16\textwidth}p{0.53\textwidth}p{0.21\textwidth}}
\toprule
Status & Claim & Primary evidence \\
\midrule
\textbf{Supported} & Exact one-path computation is preserved when controlled capacity
is sufficient. & All \textsc{High} seeds; no phase or path release. \\
\textbf{Supported} & Phase freedom alone does not recover the rank-six bottleneck;
independent path capacity does. & PhaseOnly and EMAN in all \textsc{Low}
seeds. \\
\textbf{Supported} & Released paths carry complementary latent information. & Factor
recovery, canonical correlations, principal angles, and permutation control. \\
\textbf{Supported} & The frozen timing rule follows discrete task arrivals at epochs
21 and 11. & NYUv2 arrival, PASCAL arrival, and earlier-arrival sensitivity. \\
\textbf{Qualified support} & Delayed growth reduces estimated path
training evaluations in the registered arrival protocols. & Path-FLOP ratios;
wall-clock remains implementation-dependent. \\
\textbf{Scope condition} & The rank-four controlled bottleneck is not certified. & The
current persistence rule is conservative near the capacity boundary. \\
\textbf{Scope condition} & Static PASCAL-Context has mixed task results and no delayed
compute interval. & Cross-dataset standard protocol. \\
\textbf{Scope condition} & The ramp 11--15 protocol favors a non-releasing conflict
baseline. & Smooth demand onset does not extend the discrete-arrival claim. \\
\bottomrule
\end{tabular}
\caption{Consolidated claim matrix. ``Supported'' denotes direct evidence in
the registered protocols. ``Scope condition'' denotes an observed boundary
that narrows interpretation without invalidating the supported claims.}
\label{tab:supp_claim_matrix}
\end{table*}

Table~\ref{tab:supp_replay} reports the run families represented in the replay
records included with the code package. The audits compare stored metrics and
event records within the recorded tolerances. The manifest and checksum layer
records the packaged assets. The replay summaries record the archived audit
outcomes. The per-seed tables record result inclusion.

Estimated path-training FLOPs count path evaluations. They do not represent a
complete accelerator model. Total training FLOPs include the shared
non-path computation. Wall-clock remains hardware-specific. Inference remains
dual-path after materialization.

\subsection{Consolidated Evidence Boundary}
\label{sec:supp_claim_boundary}

The complete evidence package supports a layered interpretation. The strongest
conclusions concern structural selection under explicit capacity control and
timing adaptation under the registered task-arrival protocols. Natural-image
results describe a trade-off between performance and computation. They remain
task- and seed-dependent.

Table~\ref{tab:supp_claim_matrix} consolidates the positive evidence. It also
records the scope of each protocol. The evidence does not establish task-wise
expert specialization, sparse
inference, a universal wall-clock gain, uniform superiority on every dataset
or seed, a generally optimal release direction, or statistical significance
from three seeds. These statements are not required for the central result.
EMAN is evaluated as a mechanism for evidence-gated capacity growth. It
preserves exact sharing, exposes and certifies a local growth direction, and
materializes independent capacity only when the frozen rule accepts the
transition.

\subsubsection{Evidence hierarchy and interpretation}
The experiments form a hierarchy. The rank-controlled benchmark identifies
whether a capacity deficit exists because the target rank is known in advance.
The certificate traces then identify whether the frozen rule accepts the
transition. The representation diagnostics test whether the released paths
carry complementary information. Finally, the natural-image protocols evaluate
whether the structural decision yields a useful performance--computation
trade-off under realistic task mixtures. A result at one level is not used as
a substitute for evidence at another level.

The strongest mechanistic conclusion comes from the controlled \textsc{Low}
regime. Hard-sharing and PhaseOnly remain near the same rank-limited solution,
while EMAN materializes a second path and recovers the missing factors. The
\textsc{High} regime provides the complementary non-growth result. It shows
that the same mechanism can retain exact sharing when one path is sufficient.
Together, these regimes isolate conditional capacity allocation from generic
multi-path overparameterization.

The natural-image results answer a different question. Standard NYUv2 evaluates final task performance under a frozen static matched adapter. It is not used to assess OEC-triggered materialization. Capacity-Stress and task-arrival protocols provide the corresponding structural-decision evidence. Static PASCAL-Context and the gradual ramp serve as scope checks.
This separation prevents favorable performance from being treated as proof of
a certificate event.

\subsubsection{Reading non-release and weak-result cases}
A non-release run is an output of the structural rule rather than a missing
experiment. Seed~2 under Capacity-Stress does not pass the phase certificate,
so no path-force probe is executed and no second path exists. The empty probe
panels record this control flow. They do not encode zero-valued force or a
failed plotting script. The result demonstrates that the implementation can
retain a shared state when population evidence is not persistent.

The rank-four controlled regime is a weak bottleneck. Its false negative shows
that the frozen persistence rule prioritizes conservative growth near the
capacity boundary. This observation narrows sensitivity claims, but it does
not contradict the verified rank-six recovery. Likewise, static
PASCAL-Context contains several competitive native task values but no delayed
compute interval because all tasks are active from the first epoch. Its mixed
aggregate is therefore most informative as a cross-dataset boundary rather
than as a test of delayed materialization.

The ramp protocol further separates a smooth demand transition from a discrete
task arrival. Conflict Trigger obtains the strongest aggregate without
releasing a path. EMAN and the matched Scheduled split use the same release
epochs and have nearly identical means. The archived result therefore does not
isolate a failure of the certified direction. It indicates that the current
rule has its clearest evidence under discrete demand onset and that smooth
demand may not require additional capacity in the same way.

\subsubsection{Metric-by-metric reading guide}
Phase curvature is an upstream signal. A negative point estimate alone is not
a release event. The uncertainty bound, finite-difference check, safety checks,
and persistence rule determine whether phase is released. Path-force strength
is downstream. It is evaluated only after phase release and the adaptation
window. A missing path-force trace therefore means that the run did not enter
that state. It should not be read as a measured force of zero.

Cross-half and shard cosines describe the stability of the separation axis.
The absolute value removes its arbitrary sign. High consistency does not by
itself imply that the finite perturbation reduces the deployed objective. The
frozen and dynamic descent tests answer that separate question. Dynamic safety
then checks whether calibration, exchange symmetry, and numerical values remain
inside the registered tolerances. Persistence requires the complete conjunction
to survive across probe epochs.

Path distance is a post-release diagnostic. It measures how far the two
parameter blocks move apart after the certified split. A large distance is not
a certificate, and a small distance is not a failure. The value depends on the
release epoch and on the remaining optimization horizon. Representation
complementarity is tested separately through latent-factor recovery, canonical
correlation, principal angles, and permutation controls.

The aggregate $G$ summarizes native task directions relative to Hard-sharing.
It does not replace the task columns. A positive value can coexist with a weak
individual task, and a negative value can coexist with one or more competitive
native metrics. For this reason, every external comparison retains the native
metrics and the per-seed aggregate. System quantities are also kept separate.
Parameter count, peak memory, estimated path evaluations, total FLOPs, and
wall-clock time answer different questions and should not be interchanged.

\FloatBarrier

%% file: references.bib
@article{caruana1997multitask,
  author  = {Caruana, Rich},
  title   = {Multitask Learning},
  journal = {Machine Learning},
  volume  = {28},
  number  = {1},
  pages   = {41--75},
  year    = {1997},
  doi     = {10.1023/A:1007379606734}
}

@inproceedings{misra2016crossstitch,
  author    = {Misra, Ishan and Shrivastava, Abhinav and Gupta, Abhinav and Hebert, Martial},
  title     = {Cross-Stitch Networks for Multi-Task Learning},
  booktitle = {Proceedings of the IEEE Conference on Computer Vision and Pattern Recognition},
  pages     = {3994--4003},
  year      = {2016}
}

@inproceedings{ruder2019sluice,
  author    = {Ruder, Sebastian and Bingel, Joachim and Augenstein, Isabelle and S{\o}gaard, Anders},
  title     = {Latent Multi-Task Architecture Learning},
  booktitle = {Proceedings of the AAAI Conference on Artificial Intelligence},
  volume    = {33},
  pages     = {4822--4829},
  year      = {2019},
  doi       = {10.1609/aaai.v33i01.33014822}
}

@inproceedings{liu2019mtan,
  author    = {Liu, Shikun and Johns, Edward and Davison, Andrew J.},
  title     = {End-to-End Multi-Task Learning with Attention},
  booktitle = {Proceedings of the IEEE/CVF Conference on Computer Vision and Pattern Recognition},
  pages     = {1871--1880},
  year      = {2019}
}

@inproceedings{vandenhende2020mtinet,
  author    = {Vandenhende, Simon and Georgoulis, Stamatios and Van Gool, Luc},
  title     = {{MTI-Net}: Multi-Scale Task Interaction Networks for Multi-Task Learning},
  booktitle = {Computer Vision -- ECCV 2020},
  pages     = {527--543},
  publisher = {Springer},
  year      = {2020},
  doi       = {10.1007/978-3-030-58548-8_31},
  url       = {https://doi.org/10.1007/978-3-030-58548-8_31}
}

@inproceedings{xu2018padnet,
  author    = {Xu, Dan and Ouyang, Wanli and Wang, Xiaogang and Sebe, Nicu},
  title     = {{PAD-Net}: Multi-Tasks Guided Prediction-and-Distillation Network for Simultaneous Depth Estimation and Scene Parsing},
  booktitle = {Proceedings of the IEEE Conference on Computer Vision and Pattern Recognition},
  pages     = {675--684},
  year      = {2018}
}

@inproceedings{sun2020adashare,
  author    = {Sun, Ximeng and Panda, Rameswar and Feris, Rogerio and Saenko, Kate},
  title     = {{AdaShare}: Learning What to Share for Efficient Deep Multi-Task Learning},
  booktitle = {Advances in Neural Information Processing Systems},
  volume    = {33},
  pages     = {2231--2240},
  year      = {2020}
}

@inproceedings{chen2018gradnorm,
  author    = {Chen, Zhao and Badrinarayanan, Vijay and Lee, Chen-Yu and Rabinovich, Andrew},
  title     = {{GradNorm}: Gradient Normalization for Adaptive Loss Balancing in Deep Multitask Networks},
  booktitle = {Proceedings of the 35th International Conference on Machine Learning},
  series    = {Proceedings of Machine Learning Research},
  volume    = {80},
  pages     = {794--803},
  publisher = {PMLR},
  year      = {2018}
}

@inproceedings{yu2020pcgrad,
  author    = {Yu, Tianhe and Kumar, Saurabh and Gupta, Abhishek and Levine, Sergey and Hausman, Karol and Finn, Chelsea},
  title     = {Gradient Surgery for Multi-Task Learning},
  booktitle = {Advances in Neural Information Processing Systems},
  volume    = {33},
  pages     = {5824--5836},
  year      = {2020}
}

@inproceedings{liu2021cagrad,
  author    = {Liu, Bo and Liu, Xingchao and Jin, Xiaojie and Stone, Peter and Liu, Qiang},
  title     = {Conflict-Averse Gradient Descent for Multi-Task Learning},
  booktitle = {Advances in Neural Information Processing Systems},
  volume    = {34},
  pages     = {18878--18890},
  year      = {2021}
}

@inproceedings{navon2022nash,
  author    = {Navon, Aviv and Shamsian, Aviv and Achituve, Idan and Maron, Haggai and Kawaguchi, Kenji and Chechik, Gal and Fetaya, Ethan},
  title     = {Multi-Task Learning as a Bargaining Game},
  booktitle = {Proceedings of the 39th International Conference on Machine Learning},
  series    = {Proceedings of Machine Learning Research},
  volume    = {162},
  pages     = {16428--16446},
  publisher = {PMLR},
  year      = {2022}
}

@inproceedings{javaloy2022rotograd,
  author    = {Javaloy, Adri{\'a}n and Valera, Isabel},
  title     = {{RotoGrad}: Gradient Homogenization in Multitask Learning},
  booktitle = {International Conference on Learning Representations},
  year      = {2022}
}

@inproceedings{shi2023recon,
  author    = {Shi, Guangyuan and Li, Qimai and Zhang, Wenlong and Chen, Jiaxin and Wu, Xiao-Ming},
  title     = {{Recon}: Reducing Conflicting Gradients from the Root for Multi-Task Learning},
  booktitle = {International Conference on Learning Representations},
  year      = {2023}
}

@inproceedings{shazeer2017outrageously,
  author    = {Shazeer, Noam and Mirhoseini, Azalia and Maziarz, Krzysztof and Davis, Andy and Le, Quoc V. and Hinton, Geoffrey E. and Dean, Jeff},
  title     = {Outrageously Large Neural Networks: The Sparsely-Gated Mixture-of-Experts Layer},
  booktitle = {International Conference on Learning Representations},
  year      = {2017}
}

@inproceedings{ma2018mmoe,
  author    = {Ma, Jiaqi and Zhao, Zhe and Yi, Xinyang and Chen, Jilin and Hong, Lichan and Chi, Ed H.},
  title     = {Modeling Task Relationships in Multi-Task Learning with Multi-Gate Mixture-of-Experts},
  booktitle = {Proceedings of the 24th ACM SIGKDD International Conference on Knowledge Discovery and Data Mining},
  pages     = {1930--1939},
  publisher = {ACM},
  year      = {2018},
  doi       = {10.1145/3219819.3220007}
}

@inproceedings{lepikhin2021gshard,
  author    = {Lepikhin, Dmitry and Lee, HyoukJoong and Xu, Yuanzhong and Chen, Dehao and Firat, Orhan and Huang, Yanping and Krikun, Maxim and Shazeer, Noam and Chen, Zhifeng},
  title     = {{GShard}: Scaling Giant Models with Conditional Computation and Automatic Sharding},
  booktitle = {International Conference on Learning Representations},
  year      = {2021}
}

@article{fedus2022switch,
  author  = {Fedus, William and Zoph, Barret and Shazeer, Noam},
  title   = {Switch Transformers: Scaling to Trillion Parameter Models with Simple and Efficient Sparsity},
  journal = {Journal of Machine Learning Research},
  volume  = {23},
  number  = {120},
  pages   = {1--39},
  year    = {2022}
}

@inproceedings{csordas2021modular,
  author    = {Csord{\'a}s, R{\'o}bert and van Steenkiste, Sjoerd and Schmidhuber, J{\"u}rgen},
  title     = {Are Neural Nets Modular? Inspecting Functional Modularity through Differentiable Weight Masks},
  booktitle = {International Conference on Learning Representations},
  year      = {2021}
}

@inproceedings{mittal2022modular,
  author    = {Mittal, Sarthak and Bengio, Yoshua and Lajoie, Guillaume},
  title     = {Is a Modular Architecture Enough?},
  booktitle = {Advances in Neural Information Processing Systems},
  volume    = {35},
  pages     = {28747--28760},
  year      = {2022}
}

@inproceedings{jarvis2023specialization,
  author    = {Jarvis, Devon and Klein, Richard and Rosman, Benjamin and Saxe, Andrew M.},
  title     = {On the Specialization of Neural Modules},
  booktitle = {International Conference on Learning Representations},
  year      = {2023}
}

@inproceedings{zhang2023emergent,
  author    = {Zhang, Zhengyan and Zeng, Zhiyuan and Lin, Yankai and Xiao, Chaojun and Wang, Xiaozhi and Han, Xu and Liu, Zhiyuan and Xie, Ruobing and Sun, Maosong and Zhou, Jie},
  title     = {Emergent Modularity in Pre-Trained Transformers},
  booktitle = {Findings of the Association for Computational Linguistics: ACL 2023},
  pages     = {4066--4083},
  publisher = {Association for Computational Linguistics},
  year      = {2023},
  doi       = {10.18653/v1/2023.findings-acl.250}
}

@inproceedings{schug2024discovering,
  author    = {Schug, Simon and Kobayashi, Seijin and Akram, Yassir and Wolczyk, Maciej and Proca, Alexandra Maria and von Oswald, Johannes and Pascanu, Razvan and Sacramento, Joao and Steger, Angelika},
  title     = {Discovering Modular Solutions that Generalize Compositionally},
  booktitle = {International Conference on Learning Representations},
  year      = {2024}
}

@misc{rusu2016progressive,
  title         = {Progressive Neural Networks},
  author        = {Rusu, Andrei A. and Rabinowitz, Neil C. and Desjardins, Guillaume and Soyer, Hubert and Kirkpatrick, James and Kavukcuoglu, Koray and Pascanu, Razvan and Hadsell, Raia},
  year          = {2016},
  eprint        = {1606.04671},
  archivePrefix = {arXiv}
}

@inproceedings{yoon2018den,
  author    = {Yoon, Jaehong and Yang, Eunho and Lee, Jeongtae and Hwang, Sung Ju},
  title     = {Lifelong Learning with Dynamically Expandable Networks},
  booktitle = {International Conference on Learning Representations},
  year      = {2018}
}

@misc{fernando2017pathnet,
  title         = {{PathNet}: Evolution Channels Gradient Descent in Super Neural Networks},
  author        = {Fernando, Chrisantha and Banarse, Dylan and Blundell, Charles and Zwols, Yori and Ha, David and Rusu, Andrei A. and Pritzel, Alexander and Wierstra, Daan},
  year          = {2017},
  eprint        = {1701.08734},
  archivePrefix = {arXiv}
}

@inproceedings{chen2016net2net,
  author    = {Chen, Tianqi and Goodfellow, Ian and Shlens, Jonathon},
  title     = {{Net2Net}: Accelerating Learning via Knowledge Transfer},
  booktitle = {International Conference on Learning Representations},
  year      = {2016}
}

@book{golubitsky1988singularities,
  author    = {Golubitsky, Martin and Stewart, Ian and Schaeffer, David G.},
  title     = {Singularities and Groups in Bifurcation Theory, Volume II},
  publisher = {Springer},
  address   = {New York},
  year      = {1988},
  doi       = {10.1007/978-1-4612-4574-2}
}

@book{kuznetsov2004elements,
  author    = {Kuznetsov, Yuri A.},
  title     = {Elements of Applied Bifurcation Theory},
  edition   = {3rd},
  publisher = {Springer},
  address   = {New York},
  year      = {2004},
  doi       = {10.1007/978-1-4757-3978-7}
}

@article{pearlmutter1994fast,
  author  = {Pearlmutter, Barak A.},
  title   = {Fast Exact Multiplication by the Hessian},
  journal = {Neural Computation},
  volume  = {6},
  number  = {1},
  pages   = {147--160},
  year    = {1994},
  doi     = {10.1162/neco.1994.6.1.147}
}

@inproceedings{silberman2012indoor,
  author    = {Silberman, Nathan and Hoiem, Derek and Kohli, Pushmeet and Fergus, Rob},
  title     = {Indoor Segmentation and Support Inference from {RGBD} Images},
  booktitle = {Computer Vision -- ECCV 2012},
  pages     = {746--760},
  publisher = {Springer},
  year      = {2012},
  doi       = {10.1007/978-3-642-33715-4_54}
}

@inproceedings{mottaghi2014context,
  author    = {Mottaghi, Roozbeh and Chen, Xianjie and Liu, Xiaobai and Cho, Nam-Gyu and Lee, Seong-Whan and Fidler, Sanja and Urtasun, Raquel and Yuille, Alan},
  title     = {The Role of Context for Object Detection and Semantic Segmentation in the Wild},
  booktitle = {Proceedings of the IEEE Conference on Computer Vision and Pattern Recognition},
  pages     = {891--898},
  year      = {2014}
}

@inproceedings{chai2025spurious,
  author    = {Junyi Chai and Shenyu Lu and Xiaoqian Wang},
  title     = {Identifying and Mitigating Spurious Correlation in
               Multi-Task Learning},
  booktitle = {Proceedings of the IEEE/CVF Conference on Computer Vision
               and Pattern Recognition},
  pages     = {25698--25707},
  year      = {2025}
}

@inproceedings{qin2025consistent,
  author    = {Xiaohan Qin and Xiaoxing Wang and Junchi Yan},
  title     = {Towards Consistent Multi-Task Learning: Unlocking the
               Potential of Task-Specific Parameters},
  booktitle = {Proceedings of the IEEE/CVF Conference on Computer Vision
               and Pattern Recognition},
  pages     = {10067--10076},
  year      = {2025}
}

@inproceedings{baek2025tadformer,
  author    = {Seungmin Baek and Soyul Lee and Hayeon Jo and Hyesong Choi
               and Dongbo Min},
  title     = {{TADFormer}: Task-Adaptive Dynamic TransFormer for Efficient
               Multi-Task Learning},
  booktitle = {Proceedings of the IEEE/CVF Conference on Computer Vision
               and Pattern Recognition},
  pages     = {14858--14868},
  year      = {2025}
}

@inproceedings{mantri2025ditask,
  author    = {Krishna Sri Ipsit Mantri and Carola-Bibiane Sch{\"o}nlieb
               and Bruno Ribeiro and Chaim Baskin and Moshe Eliasof},
  title     = {{DiTASK}: Multi-Task Fine-Tuning with Diffeomorphic
               Transformations},
  booktitle = {Proceedings of the IEEE/CVF Conference on Computer Vision
               and Pattern Recognition},
  pages     = {25218--25229},
  year      = {2025}
}

@inproceedings{wang2026mtl3d,
  author    = {Xiaoye Wang and Chen Tang and Xiangyu Yue and Wei-Hong Li},
  title     = {{3D}-Aware Multi-Task Learning with Cross-View Correlations
               for Dense Scene Understanding},
  booktitle = {Proceedings of the IEEE/CVF Conference on Computer Vision
               and Pattern Recognition},
  pages     = {5793--5803},
  year      = {2026}
}

@inproceedings{park2026mass,
  author    = {Sumin Park and Noseong Park},
  title     = {How Many Experts Are Enough? Towards Optimal Semantic
               Specialization for Mixture-of-Experts},
  booktitle = {Proceedings of the AAAI Conference on Artificial Intelligence},
  volume    = {40},
  pages     = {24792--24800},
  year      = {2026},
  doi       = {10.1609/aaai.v40i29.39665}
}

@inproceedings{zhang2026growondemand,
  author    = {Ying Zhang and Xingyue Guo and Yu Zhao and Xuhui Sui and
               Baohang Zhou and Xinying Qian and Xiaojie Yuan},
  title     = {Grow-on-Demand: Sparse and Adaptive Expert Expansion for
               Continual Instruction Tuning},
  booktitle = {Proceedings of the AAAI Conference on Artificial Intelligence},
  volume    = {40},
  pages     = {28474--28482},
  year      = {2026},
  doi       = {10.1609/aaai.v40i34.40077}
}
